\documentclass{article}

\usepackage{PRIMEarxiv}

\usepackage[utf8]{inputenc} % allow utf-8 input
\usepackage[T1]{fontenc}    % use 8-bit T1 fonts
\usepackage{hyperref}       % hyperlinks
\usepackage{url}            % simple URL typesetting
\usepackage{booktabs}       % professional-quality tables
\usepackage{amsfonts}       % blackboard math symbols
\usepackage{nicefrac}       % compact symbols for 1/2, etc.
\usepackage{microtype}      % microtypography
\usepackage{lipsum}
\usepackage{fancyhdr}       % header
\usepackage{graphicx}       % graphics

\usepackage{multirow}
\usepackage{subfigure}
\usepackage{enumerate}
\usepackage{enumitem}
\setlist{itemsep=1pt}
\usepackage{caption}

\usepackage{amsmath, amsfonts,amssymb}
\usepackage{amsthm}
 \newtheorem{theorem}{Theorem}[section]
 \newtheorem{lemma}[theorem]{{Lemma}}
 \newtheorem{assumption}{Assumption}
  \newtheorem{definition}{Definition}
  
 \newtheorem{remark}{{Remark}}

\usepackage{caption}
\usepackage{subfigure}
\usepackage{lipsum}

\usepackage{stackengine}
\usepackage{booktabs}

\RequirePackage{hyperref}
\RequirePackage{nameref}
\hypersetup{colorlinks, linkcolor=blue, citecolor=blue, urlcolor=magenta, linktocpage, plainpages=false}

\usepackage{algorithm}
\usepackage{algorithmic}

\usepackage{natbib}
\newcommand{\R}{\mathbb{R}}
\newcommand{\E}{\mathbb{E}}
\newcommand{\I}{\mathbb{I}}

\title{On uniform boundedness properties of SGD\\ and its momentum variants}
\author{
  Xiaoyu Wang \\
  KTH - Royal Institute of Technology \\
   Stockholm, Sweden\\
  \texttt{wang10@kth.se} \\
   \And
  Mikael Johansson \\
  KTH - Royal Institute of Technology \\
  Stockholm, Sweden\\
  \texttt{mikaelj@kth.se} \\
}

\begin{document}
\maketitle

\begin{abstract}
A theoretical, and potentially also practical, problem with stochastic gradient descent is that trajectories may escape to infinity.  In this note, we investigate uniform boundedness properties of iterates and function values along the trajectories of the stochastic gradient descent algorithm and its important momentum variant. Under smoothness and $R$-dissipativity of the loss function, we show that broad families of step-sizes, including the widely used step-decay and cosine with (or without) restart step-sizes, result in uniformly bounded iterates and function values. Several important applications that satisfy these assumptions, including phase retrieval problems, Gaussian mixture models and some neural network classifiers, are discussed in detail. We further extend the uniform boundedness of SGD and its momentum variant under the generalized dissipativity for the functions whose tails grow slower than quadratic functions. This includes some interesting applications, for example, Bayesian logistic regression and logistic regression with $\ell_1$ regularization.

\end{abstract}

% keywords can be removed

\section{Introduction}
We consider the stochastic optimization problem of minimizing a (possibly non-convex) function $f$
\begin{align}
\underset{x}{\min}\; f(x) = \E_{\xi \approx \Xi} [f(x; \xi)].
\end{align}
Here, $\xi$ is a random variable drawn from the unknown probability distribution $\Xi$ and $f(x; \xi)$ is the instantaneous loss function over the variable $x$. We assume that we may query the gradient of $f$ by randomly sampling mini-batches of size $b$. Stochastic gradient descent (SGD~\citep{SGD-1951}) and its momentum (or accelerated) variants~\citep{polyak1964some, sutskever2013importance} have been the subject of intense research, both in terms of theoretical developments and applications, particularly to training of deep neural networks and other machine learning models.

The standard SGD algorithm~\citep{SGD-1951} is 
\begin{align}
x_{k+1} = x_k - \eta_k g_k \qquad \mbox{ for } k\in {\mathbb N}
%k=1, 2, \dots
\end{align}
where $\eta_k >0$ is the step-size and $g_k$ is a noisy but unbiased estimator of $\nabla f(x_k)$.  In the noise-less setting (\emph{i.e.,} when $g_k=\nabla f(x_k)$), the function value $f(x_k)$ is decreasing with $k$ as long as $\nabla f$ is $L$-Lipschitz continuous and $\eta_k \leq 1/L$. In fact, neither $f(x_k)$ nor the distance between $x_k$ and the optimal set will expand. However, the behaviour is very different in the presence of noise.  %After running the SGD algorithm with the total number of iteration $T \geq 1$, we can obtain a sequence of $\left\lbrace x_k \right\rbrace_{k=1}^{T}$. 
It may happen that $x_k$ tends to infinity if the correction $\eta_k g_k$ is always large even though the step-size $\eta_k$
tends to zero; or if the variance $\E[\left\|g_k - \nabla f(x_k)\right\|^2]$  increases rapidly with $\left\| x \right\|$. In this note, we are interested in finding simple and verifiable conditions under which $\left\|x_k - x^{\ast}\right\|$ and  $f(x_k) - f^{\ast}$ remain bounded along every possible trajectory of the system. %}{A question appears how the quantities \emph{e.g.}, $\left\|x_k - x^{\ast}\right\|$ and  $f(x_k) - f^{\ast}$ behave when the number of the iteration $k$ goes to infinity?}    %There are a lot of theoretical work that studies the convergence rates or almost sure convergence of the stochastic algorithms~. Unlike these work, in this paper,   
% : do they converge, or are they uniformly bounded or will they go to infinity (diverging)

Of course, $f(x_k)$ will remain bounded if $f$ itself is bounded on $\R^d$; this is the case for some truncated loss functions that have been used in the robust machine-learning literature~\citep{robust-logistic,pmlr-v115-xu20b}. Another trivial case is when the variable $x$ is kept in the bounded region, for example by running a projected SGD. But these are not the cases that we are interested in addressing int his paper. The main focus of this work is on loss functions $f$ that may go to infinity when $\left\| x \right\|$ tends to infinity. The most common loss functions in machine learning, such as the least-squares and the logistic loss, are unbounded on $\R^d$, and it %It is very common choice in machine learning that 
very common to add a weight-decay regularization $\lambda\left\|x\right\|^2$ ($\lambda >0$) to the original loss function to help to avoid overfitting~\citep{weight-decay}. Hence, the resulting objective function is typically unbounded.

{\bf Motivation.} Many theoretical analyses of SGD or its variants simply assume that $\left\|x_k - x^{\ast} \right\|$ or $f(x_k) - f^{\ast}$ are uniformly bounded (in expectation or almost sure)
%is often \added{simply} assumed\deleted{in the literature}
~\citep{ljung1977analysis,hazan2014beyond,xu2019stochastic,davis2020stochastic,Xiaoyu-Step-decay,wang2021bandwidth}, even that $f$ itself is bounded~\citep{Yang-2019,pmlr-v115-xu20b,yuan2019stagewise}. However, in the noise regime, this uniform boundedness assumption may be not always correct. %({\bf give an example}).
One intuitive guess is that such a character may be very relevant to the landscape of the target function $f$ especially the point is far from its global minimizer (or local minimizers). As we know, there is very little work that directly state that the uniform boundedness hypothesis is true or can be guaranteed under the proper conditions.

{\bf Related work.}
A few recent contributions have been able to  establish convergence rates or almost sure convergence (which also implies the stability $\left\| x_k - x^{\ast}\right\| < +\infty$) of SGD~\citep{mertikopoulos2020almost,wojtowytsch2021stochastic,patel2021stochastic} and its momentum (heavy ball) variant~\citep{sebbouh2021almost} if the step-size $\eta_k$ is diminishing and satisfies the Robbins-Monro condition: $\sum_{k=1}^{\infty}\eta_k = + \infty$ and $\sum_{k=1}^{\infty} \eta_k^2 < + \infty$. However, such  step-sizes (e.g., $\eta_k = \eta_1/t^p$ for $p \in (1/2, 1]$) which decay very rapidly are not so attractive in the non-convex setting 
and often performs poorly in practice~\citep{ge2019step,Xiaoyu-Step-decay}. Stepsizes that are more popular in practice, such as step-decay~\citep{ge2019step,Xiaoyu-Step-decay}, exponential decay~\citep{li2020exponential}  and cosine decay step-sizes~\citep{loshchilov2016sgdr} may not satisfy the Robbins-Monro condition. However, these step-sizes are covered by the analysis in this paper.

For constant step-sizes, \citet{Lu-2021} established an asymptotic normality result for the SGD algorithm for a non-convex and
non-smooth objective function satisfying a dissipativity property. The paper~\citep{chen2021stationary} analyzed the asymptotic stationary behavior of SGD for constant step-size on smooth and strongly convex function, while \citet{shi2020learning} revealed the dependence of the linear convergence on the constant step-size for continuous-time formulation of SGD. In the discrete-time case, the authors of~\citep{shi2020learning} also provided an upper bound but do not state how this bound depends on the total number of iteration, which opens up the possibility that the bounds are loose.

{\bf Contributions.}
In this work, we study uniform boundedness properties of SGD and its momentum variant (SGD with momentum) with the help of a regularity assumption which we call ``$R$-dissipativity": $\left\langle x-x^{\ast}, \nabla f(x)\right\rangle \geq \theta_1 \left\| x - x^{\ast} \right\|^2 - \theta_2$  for all $x$ with $\left\| x - x^{\ast} \right\| \geq R$ where $\theta_1 >0$,  $\theta_2 \geq 0$ and $R \geq 0$. This concept is a slight variation of the ``localized dissipativity'' property introduced  
%generalizes the classical dissipativity concept in dynamical %^systems~\citep{hale1988asymptotic}  that was used 
in~\citep{Lu-2021} but we do not concerns the local properties of the function for $x$ with $ \left\| x - x^{\ast}\right\| \leq R$. The term of ``dissipativity" is originated from dynamic systems~\citep{hale1988asymptotic}
and used in a number of papers in optimization, learning and Bayesian analysis~\citep{mattingly2002ergodicity,raginsky2017non,Erdogdu-2018-,gauss-mix}. As we will see in Section~\ref{sec:application}, many important application in machine learning, including shallow neural networks and the sum of a weight-decay regularization and the function whose gradient is bounded result in $R$-dissipative objective functions.
We make the following contributions: 
\begin{itemize}
    \item  Uniform boundedness properties of SGD and its momentum variant (SGD with momentum) are established for $L$-smooth and $R$-dissipative loss functions. 
 \item We consider a broad class of step-sizes that are only required to be smaller than a constant $\mathcal{O}(1/L)$. Our results cover many of the most popular step-size policies, including the classical constant and polynomial decay step-sizes, as well as more recently proposed time-dependent step-sizes such as stage-wise decay~\citep{li2020exponential,Xiaoyu-Step-decay}, cosine with or without restart~\citep{loshchilov2016sgdr}, and bandwidth-based step-sizes~\citep{wang2021convergence-stronglyconvex,wang2021bandwidth}.
\end{itemize}
In addition, we generalize the $R$-dissipativity to cover a class of function of slower asymptotic growth, for example, the $\ell_1$ regularized logistic regression and Bayesian logistic regression. We also make the contributions:
\begin{itemize}
    \item The uniform boundedness properties of SGD can also be guaranteed for a broad of class of step-sizes that is only required to be upper bounded. 
    \item For the important variant momentum, the uniform boundedness can be established for constant step-size, polynomial decaying, exponential decaying step-size and also stage-wise step-sizes which decay like constant, polynomial, and exponential at each stage.
\end{itemize}

%We make the following contributions.

% \begin{itemize}
%     \item The uniform boundedness is often assumed in the literature when the authors may need it to do the analysis. List the related work.
%     \item introduce the main condition Dissipativity we use in the following analysis
%     \item wide applications
%     \item Explain the landscape of the objective function
% \end{itemize}

\section{Preliminaries}
% In this set-up, we consider stochastic optimization problem of minimizing a function $f$ which is possibly non-convex
% \begin{align}
% \min f(x) = \E_{\xi \approx \Xi} f(x; \xi)
% \end{align}
% where $\xi$ is a random variable drawing from the unknown probability distribution $\Xi$ and $f(x; \xi)$ is the instantaneous loss function over the variable $x$. We assume that we may query the gradient of $f$ by randomly sampling mini-batches of size $b$. The standard SGD algorithm~\citep{SGD-1951} is formulated as below ($k\geq 1$)
% \begin{align}
% x_{k+1} = x_k - \eta_k g_k
% \end{align}
% where $\eta_k >0$ is the step-size and 
Throughout the paper, we make the following assumption on the stochastic gradient oracle $\mathcal{O}$ which samples the stochastic gradient as an unbiased estimator of the true gradient given any input variable $x$.  %and the noise level is controlled by the scale of the gradient plus a constant. 
\begin{assumption}\label{assump:gradient}
For any input vector $x$, the stochastic gradient oracle $\mathcal{O}$ returns a vector $g$ such that %(a) $\E[g \mid \mathcal{F}] = \nabla f(x)$;
(a) $\E[g_k] = \nabla f(x_k) $; (b) $\E[\left\| g_k - \nabla f(x_k) \right\|^2] \leq \rho \left\| \nabla f(x_k) \right\|^2 + \sigma^2$ where $\rho \geq 0$.
%\begin{enumerate}[label=(\emph{\alph*})]
%		\item\label{gradient:assump:1} 
%		\item \label{gradient:assump:2}  
%\end{enumerate}
% \begin{itemize}
%      \item[]
%     \item $\E[g_k] = \nabla f(x_k) $
%     \item $\E[\left\| g_k - \nabla f(x_k) \right\|^2] \leq \rho \left\| \nabla f(x_k) \right\|^2 + \sigma^2$ where $\rho \geq 0$. 
% \end{itemize}
\end{assumption}
\begin{definition}({\bf $L$-smooth})\label{assump:smooth}
  A function $f$ is $L$-smooth if
 $\Vert \nabla f(x)-\nabla f(y)\Vert \leq L\Vert x-y\Vert$ for every $x,y \in \mbox{dom}\, (f)$. The smoothness property also implies that $
| f(x) - f(y) - \left\langle \nabla f(y), x - y\right\rangle | \leq \frac{L}{2}\left\| x - y \right\|^2. 
$
\end{definition}

% \begin{definition}({\bf $\tau$-h\"{o}lder continuous gradient})\label{assump:holder}
%   We say that the gradient $\nabla f(x)$ of the objective function $f$ is $\tau$-h\"{o}lder continuous if $\left\| \nabla f(x) - \nabla f(y) \right\| \leq \theta \left\| x - y \right\|^{\tau}$ where $\tau >0 $ and $x, y \in \R^d$. This implies that $
% | f(x) - f(y) - \left\langle \nabla f(y), x - y\right\rangle | \leq \frac{\theta}{2}\left\| x - y \right\|^{\tau +1}. 
% $ 
% \end{definition}
% \added{Clearly, the gradient is $1$-h\"{o}lder continuous means that the objective function is $L$-smooth.}

{\bf Notations.} We use $x^{\ast}$ to denote the minimizer of function $f$ and $f^{\ast} = f(x^{\ast})$. Let $\left\|\cdot \right\| := \left\| \cdot \right\|_2$ without specific mention. We use $[n]$ denote the set of $\left\lbrace 1, 2, \cdots, n \right\rbrace$. The subgradient of the function $f$ on $x$ is denoted by $\partial f(x):= \left\lbrace v: f(y) \geq f(x) + \left\langle v, y-x \right\rangle \text{for any } \, y\right\rbrace$. For simplicity of the notations, we use $N, T, S$ are all integers. 
We use $\mathcal{F}_k$ to denote $\sigma$-algebra generated by all the random information at the current iterate $x_k$ and $x_k \in \mathcal{F}_k$. 
\section{Non-convex Function with its Tail Growing Quadratically}\label{sec:dissipative}

In this section, we establish uniform boundedness of SGD and its momentum variant when the objective function is potentially nonconvex and satisfies the following regularity assumption. 
\begin{definition}({\bf $R$-dissipativity})\label{dissipative:assump}
  A function is $R$-dissipative that there exist $\theta_1 > 0$, $\theta_2 \geq 0$, and $R \geq 0$ such that, for all $x$ with $\left\|x - x^{\ast} \right\| \geq R$, the function $f$ satisfies 
\begin{align}\label{inequ:dissipativity}
\left\langle x - x^{\ast}, \nabla f(x)\right\rangle \geq \theta_1 \left\| x - x^{\ast}\right\|^2 - \theta_2.
\end{align}
 \end{definition}
  When $R=0$, this assumption reduces the standard dissipativity assumption~\citep{raginsky2017non,Lu-2021}(see (A.3) in \citep{raginsky2017non}). Dissipativity can be seen as a relaxation of strong convexity, since every $\mu$-strongly convex function is also dissipative. In particular, it satisfies Definition~\ref{dissipative:assump} with $\theta_1=\mu$, $\theta_2=0$ and $R=0$. %Unlike~\citep{Lu-2021}, we do not concern its local landscape around global minimizer  but restrict the quadratic growth to the tail of the function $f$ when it is far from the global minimizer.  

  The term ``dissipative" comes from the theory of dynamical systems~\citep{hale1988asymptotic}  but the concept has found many uses in the analysis of optimization and learning algorithms~\citep{mattingly2002ergodicity,raginsky2017non,Erdogdu-2018-}, and in Bayesian analysis~\citep{gauss-mix}. 
  $R$-dissipativity is a natural property of many non-convex optimization problems, such as neural network training with weight-decay regularization~\citep{weight-decay}; see \S~\ref{sec:application} for more details.
\begin{remark}\label{rem:dissipative}
For $L$-smooth functions, an alternative way of verifying $R$-dissipativity is to find $\theta_1^{\prime}$ and $\theta_2^{\prime}$ such that 
%One alternative way to check Assumption \ref{dissipative:assump} is that there exist $\theta_1^{'} > 0$ %and $\theta_2^{'} \geq 0$ such that
    \begin{align}\label{dissipative:case2}
      \left\langle x, \nabla f(x)\right\rangle \geq \theta_1^{\prime} \left\| x\right\|^2 - \theta_2^{\prime} 
    \end{align}
    for $\left\|x - x^{\ast} \right\| \geq R$. 
    %If $x^{\ast}=0$, then $\theta_1=\theta^{'}$ and $\theta_2=\theta_2^{'}$. Otherwise, if the function $f$ is also $L$-smooth, then the parameters in Assumption \ref{dissipative:assump} are 
    As shown in Appendix~\ref{supply:sgd}, this implies that $f$ is $R$-dissipative with parameters \begin{align*}
        \theta_1 &= \frac{\theta_1^{\prime}}{2}, \qquad \theta_2 = \left(\theta_1^{\prime} + 2L + \frac{L^2}{2\theta_1^{\prime}}\right)\left\| x^{\ast} \right\|^2 + \theta_2^{\prime}.  
    \end{align*}
    %More details are given at the end of Appendix \ref{supply:sgd}.
\end{remark}
\begin{remark}
 Although the definition of $R$-dissipativity in Definition~\ref{dissipative:assump} considers differentiable functions, smoothness is not an essential property of our analysis. If $f$is non-differentiable, the gradient $\nabla f(x)$ in Definition~\ref{dissipative:assump} can be replaced by the sub-gradient $\partial f(x)$ and the subsequent analysis for SGD and SGD with momentum still holds true. 
\end{remark}

\subsection{Applications}
\label{sec:application}
We now show that several important applications  in machine learning,  robust statistic learning, and Bayesian models are $R$-dissipative. Some of these are also dissipative, while others are localized dissipative, but it is only the $R$-dissipative concept that captures all examples.

% \iffalse
% {\color{red}\bf 
% \begin{itemize}
%     \item it would be nice to list, for each application, if they are also dissipative and localized dissipative, or not.
   
%   {\color{blue}
%   \item It seems a little bit strange to say that neural networks are R-dissipative, since it should depend on the loss function that we use to penalize prediction errors. In the appendix, this seems all good, but the way the text is written makes this a little hard to see.
%   \item I am not sure if we need to restate the R-dissipative inequality, like we do after the $\ell_2$-regularization; isn't it enough to state $\theta_1$, $\theta_2$ and $R$?
%   \item if all applications are described very briefly, the I believe that we could also typeset ``regularized MLE with heavy-tailed linear regression'', and ``Blake-Zisserman'' in boldface, as separate examples. 
%   }
%     \item Gaussian mixture model are localized dissipative and R-dissipative
%     \item shallow neural networks and phase retrieval does not satisfies is not localized dissipative but they are dissipative and R-dissipative 
%     \item it seems that all the applications are dissipative or localized dissipative, but no exception that does not belong to the two cases.
% \end{itemize}
% }
% \fi

{\bf Least squares problems} where 
%One common example 
%in the context of linear regression 
%that satisfies Assumption \ref{dissipative:assump} is the least-squares problems \deleted{(MSE)} where \deleted{the function is given by} 
$f(x) = \frac{1}{2}\left\| Px - b \right\|^2$ and $P$ has full rank are $R$-dissipative with $\theta_1 = \lambda_{\min}(P^{T}P)$ (the smallest eigenvalues of $P^{T}P$), $\theta_2 = 0$, and  $R=0$. They are therefore also dissipative. 
        
{\bf $\ell_2$-regularized problems}, where $f + \lambda \left\|x\right\|^2$ with $\lambda>0$ and $f$ is a possibly non-convex whose gradient norm is bounded by $G$ also satisfy Assumption~\ref{dissipative:assump}.
%A canonical example that can be verified Assumption \ref{dissipative:assump} is the sum of a non-convex function $f$ and a weight decay regularization~\citep{weight-decay}: $f(\cdot)$ + $\left\|\cdot\right\|^2$ where the gradient of $f$ is bounded by $G$.  
For example, most of the activation functions in machine learning including the sigmoid, hyperbolic tangent, rectified linear unit (ReLU) and variants such as Leaky ReLU and ELU can be applied. In this case, we have $\theta_1 = \lambda/2$, $\theta_2 = G/(2\lambda)$, and $R=0$.
% \iffalse
% , i.e.
%          \begin{align*}
%         \left\langle x - x^{\ast}, \nabla f(x) \right\rangle & = \lambda\left\|x - x^{\ast}\right\|^2 +  \left\langle x- x^{\ast}, \nabla f(x) \right\rangle  \notag \\
%         & \geq \lambda\left\|x - x^{\ast}\right\|^2 - \left\|\nabla f(x) \right\|\left\|x- x^{\ast}\right\| \geq \frac{\lambda}{2}\left\|x - x^{\ast}\right\|^2 - \frac{G}{2\lambda}.
%         \end{align*}
% \fi

{\bf Shallow neural networks} A subset of neural network training objectives with weight-decay regularization can also be shown to satisfy Assumption \ref{dissipative:assump} (or the alternative condition (\ref{dissipative:case2})). For example, consider the training a fully connected 2-layer neural network (with $m$ hidden nodes) for binary classification with a weight decay ($\lambda\left\| \cdot \right\|^2/2$). Let $\sigma_1$ and $\sigma_2(\cdot)$ denote the inner and output layer activation functions, respectively and use cross-entropy to measure the output with its true label. %Given $n$ pairs input data $\left\lbrace a_j, b_j\right\rbrace_{j=1}^n$, we revise the input data as $a_j = [a_j, 1] \in \R^d$ then
       % \begin{align}
    %        F(x) = \frac{1}{n}\sum_{j=1}^n CE(\sigma_2(X_2\sigma_1(X_1a_j))) +  \frac{\lambda}{2}\left\|X \right\|^2
     %   \end{align}
      %  where $X = [X_1, X_2] \in \R^{dm + sm}$, $X_1 \in \R^{m \times d}$ and $X_2 \in \R^{s\times m}$. 
       %For simplicity, we consider the binary dataset %$b_j \in \left\lbrace -1, +1\right\rbrace$ %and $s=1$,
       %and
       If we use $\sigma_2(y) = 1/(1+\exp(-y))$ and $\sigma_1(x) = \max(0, x)$ (the ReLU loss function), then the resulting training objective satisfies
       %its gradient $\partial \sigma_1(x) = \chi_{x \geq 0}$, we have
 \begin{align}
     \left\langle x, \nabla f(x)\right\rangle \geq \lambda\left\| x\right\|^2  - 2. 
 \end{align}
 On the other hand, if $\sigma_1(x) = 1/(1+\exp(-x))$, we find
 \begin{align}
     \left\langle x, \nabla f(x)\right\rangle & 
     \geq \frac{\lambda}{2}\left\| x\right\|^2 - \left(1+ \frac{m}{2\lambda} \right).
 \end{align}
Similarly, a $S$-layer fully connected neural network with ReLU as the activation function of the inner-layers satisfies
\begin{align}    
\left\langle x, \nabla f(x)\right\rangle  \geq \lambda\left\|x \right\|^2 - S.
\end{align}
These three examples are all dissipative and $R$-dissipative, but not localized dissipative; see Appendix~\ref{supple:application} for details. %\textcolor{red}{L-Layers, fully connected neural networks.   }

{\bf Phase retrieval} problems~\citep{tan2019online}, where $f(x) = 1/(2n)\sum_{i=1}^n (|\left\langle a_i, x\right\rangle|-b_i)^2$ also satisfy  this condition with
%In statistic learning, there are many non-convex problems such as %phase retrieval  that satisfies this condition where $f(x) = %1/(2n)\sum_{i=1}^n (|\left\langle a_i, x\right\rangle|-b_i)^2$ with 
$\theta_1 = \lambda_{\min}(P^{T}P)/(2n)$ ($P=[a_1^{T}, \cdots, a_n^{T}]^T \in \R^{n\times d}$), $\theta_2 = \sum_{i=1}^n b_i^2/(2n)$, and $R=0$.
% \iffalse
%         \begin{align*}
%          \frac{1}{n}\sum_{i=1}^n\left\langle x,  \left(\left\langle a_i, x\right\rangle - \mbox{sign}(\left\langle a_i, x\right\rangle)b_i \right) a_i \right\rangle & = \frac{1}{n}\sum_{i=1}^n \left\langle a_i, x\right\rangle^2 -  b_i\left\langle a_i,x\right\rangle \mbox{sign}(\left\langle a_i,x\right\rangle ) \notag \\
%          & \geq \frac{1}{n}\sum_{i=1}^n \left\langle a_i, x\right\rangle^2 -  \frac{1}{2}\left( b_i^2 +  \left\langle a_i,x\right\rangle^2\right):= \frac{1}{2n}\left( x^{T}A^{T}Ax - \sum_{i=1}^{n}b_i^2 \right).
%         \end{align*}
%  \fi
 
 {\bf Regularized MLE} %Besides, in \citep{Lu-2021}, the authors show 
 Another important example that satisfies the $R$-dissipativity assumption is 
 the regularized MLE for heavy tailed linear regression~\citep{Lu-2021} arising in robust statistics. Here,  $f(x)=\frac{1}{2n}\sum_{i=1}^n\log(1+(b_i-\left\langle a_i, x\right\rangle)^2) + \frac{\lambda}{2}\left\| x \right\|^2$ with $\theta_1 = \lambda$, $\theta_2=\frac{1}{n}\sum_{i=1}^n |b_i|$ and $R=0$. 
 
 {\bf Blake \& Zisserman MLE} As shown in~\citep{Lu-2021}, the objective function of regularized Blake-Zisserman maximum-likelihood estimation, where  $f(x) = -\frac{1}{2n}\sum_{i=1}^n\log(\nu +\exp(-(b_i-\left\langle a_i, x\right\rangle)^2)) + \frac{\lambda}{2}\left\| x \right\|^2$ for some $\nu>0$  
 satisfies Assumption \ref{dissipative:assump}  with $\theta_1=\lambda$, $\theta_2= 1/(2n)\sum_{i=1}^n |b_i|/\sqrt{\nu(\nu+1)}$, and $R=0$. 
 
 {\bf Gaussian mixture model} As a final example, we note that over-specified Bayesian Gaussian mixture models, which have been widely used to study datasets
with heterogeneity \citep{gauss-mix}, can be verified to satisfy Assumption \ref{dissipative:assump} with $\theta_1=4c_1$, $\theta_2=4c_1$, and $R = \sqrt{2}$ where $c_1>0$ is a universal constant (see Section~4.3 in \citep{gauss-mix}). Gaussian mixture models are $R$-dissipative and localized dissipative, but not dissipative.
 %Informative non-response model ($R > 0$) with $\theta_1=$, $\theta_2=$, and $R=$. 
% \item \textcolor{red}{Questions: could I find some example in neural networks that for the point $x$ is far from the optimal solution, then this condition is satisfied.}
\\
  
\begin{remark}
If the original objective function does not satisfy $R$-dissipativity, %Assumption~\ref{dissipative:assump}, 
it can always be rendered $R$-dissipative 
%is not $R$-dissipative, In a more general setting, if the objective %function is not consistent with the functions we mentioned above, one %simple way to satisfy Assumption \ref{dissipative:assump} is to add a
by adding a smooth regularizer $\max\left\lbrace \left\|x \right\|^2 - R^2, 0 \right\rbrace$ with $R >0$. By properly choosing $R$, we do not have to change the structure of the function around global optimizer %minimizer (or local minimizer), 
but still force the iterations in a bounded region. A similar technique has been used in matrix completion problems in~\citep{sun-matrix-completion}. 

\end{remark}

%\textcolor{red}{Notes: In the above example, $R=0$. If we choose $R>0$, then the parameter $\theta_1$ and $\theta_2$ may be different.}

\subsection{Uniform Boundedness Properties of SGD}

In the remaining pages of this note, we use the concept of $R$-dissipativity to establish boundedness properties of SGD and SGD with momentum. Our first result, proven in Appendix~\ref{supply:sgd} is the following:
% \iffalse
% In the following of this work, we first consider the theoretical behavior of the SGD algorithm when the objective function satisfies Assumption~\ref{dissipative:assump} for $\left\| x - x^{\ast} \right\| \geq R$ where $R >0$. It implies that Assumption~\ref{dissipative:assump}  does not need to hold everywhere only for the point that is far from the global minimzer. The uniform boundedness property of SGD is addressed below and the proof is given in Appendix~\ref{supply:sgd}.
% \fi
\begin{theorem}\label{thm:sgd:general}
Let $f$ be $L$-smooth and $R$-dissipative. If the stochastic gradient oracle that satisfies Assumption~\ref{assump:gradient}, then the iterates of SGD with step-size $\eta_k\leq \frac{\theta_1}{(\rho+1)L^2}$ satisfy 
\begin{align}
\E[\left\| x_k - x^{\ast} \right\|^2] \leq \max\left\lbrace \left\|x_1 - x^{\ast} \right\|^2, \frac{2(\sigma^2+L^2 r^2)\theta_1^2}{(1+\rho)^2L^4} + 2r^2\right\rbrace. \label{eqn:sgd_bound}
\end{align}
where $r^2 = \max \left\lbrace R^2, \frac{2\theta_2}{\theta_1} + \frac{\sigma^2}{(1+\rho)L^2} \right\rbrace$.
\end{theorem}

Theorem \ref{thm:sgd:general} shows that the iterations of SGD can be guaranteed to be in a bounded region under mild conditions. The uniform bound is directly related to the initial state, the noise $\sigma^2$, and the properties of the function itself. 

Although the result is valid for all values of $R$ (and therefore holds for both dissipative and localized dissipative functions), the right-hand side of (\ref{eqn:sgd_bound}) is not tight when $R=0$. We analyze dissipative functions in Appendix~\ref{supple:Rzero} and derive tighter results, recovering the well-known bounds for $\mu$-strongly convex functions as a special case.
%, nor the step-size. Especially, if Assumption %\ref{dissipative:assump}  holds for all $x \in \R^d$ which means %that $R=0$, we can get tighter bounds. More results are provided in %Appendix~\ref{supple:Rzero}.
%
%Another observation is that when $R=0$, Assumption %\ref{dissipative:assump} actually is not tight when $x$ is close to %$x^{\ast}$(the right side of (\ref{inequ:dissipativity}) does not %approach zero), therefore, we may need another condition to derive %the further convergence rates for the algorithms. However, it is not %the main purpose of this paper and we will not go into details. }

\begin{remark} Theorem~\ref{thm:sgd:general} guarantees uniform boundedness of the SGD iterates for any  (possibly non-monotonic) step-size upper bounded by $\frac{\theta_1}{(\rho+1)L^2}$. This includes common step-size policies such as constant, polynomial decay~\citep{Moulines-Bach2011}
%step-sizes $\eta_k = \eta_1/t^p$ ($p \in %(0,1]$)~\citep{Moulines-Bach2011}, 
%cosine step-sizes~\citep{loshchilov2016sgdr}, 
step-decay~\citep{Xiaoyu-Step-decay}, and exponential decay~\citep{li2020exponential}, as well as more recently proposed non-monotonic step-sizes such as the bandwidth-based~\citep{wang2021convergence-stronglyconvex,wang2021bandwidth} and cosine~\citep{loshchilov2016sgdr}  step-sizes.

\end{remark}

% % \iffalse
% \begin{remark}%({\bf Upper bound of step-size $\eta_k$})
% %The upper bound on the step-size given in %Theorem~\ref{thm:sgd:general} can seem restrictive, especially if %$\theta_1\ll L$. 
% \textcolor{red}{One may ask that the upper bound for $\eta_k$ in Theorem \ref{thm:sgd:general} is a little restrictive, especially $\theta_1 \ll L$. However, a trick below shows that this bound can be relaxed to the standard case. As shown in Theorem \ref{thm:sgd:general}, the expectation of all the iterations are upper bounded by $\bar{R}^2 := \max\left\lbrace \left\|x_1 - x^{\ast} \right\|^2, \frac{2(\sigma^2+L^2R^2)\theta_1^2}{(1+\rho)^2L^4} + 2R^2\right\rbrace $. If we fix $\hat{\theta}_1 = L$, then $\hat{\theta}_2 = (L-\theta_1)\left\| x - x^{\ast}\right\|^2 + \theta_2 \leq (L-\theta_1)\bar{R}^2 + \theta_2$. In this case, the restriction for step-size will reduce to be $\eta_k \leq \frac{1}{(\rho+1)L}$, which is commonly used in the literature~\citep{Xiaoyu-Step-decay, li2020exponential}.} %In this case, accordingly we assume that Condition (\ref{dissipative:assump} ) holds with $\bar{R}^2 \geq \frac{\theta_2^{'}}{\theta_1^{'}} + \frac{\sigma^2}{(1+\rho)L^2}$
% \end{remark}

The last comment we make for SGD is to relax the $L$-smoothness and try to include more classes of examples.
\begin{remark}({\bf Relaxation of $L$-smooth})
The $L$-smooth assumption in Theorem \ref{thm:sgd:general} can be replaced by $\left\| \nabla f(x) \right\| \leq L\left( 1 + \left\| x\right\|\right)$ as~\citep{Lu-2021} (or there exists a vector $g \in \partial f(x)$ such that $\left\| g\right\| \leq L( 1 + \left\| x\right\|)$).   As a consequence, this also ensures that uniform boundedness of the SGD algorithm. This condition allows some functions that are non-smooth and is an extension of the $L$-smoothness. 
%  which implies that the gradient of the function $f$ is at most linear growth for all $x \in \R^d$. 
\end{remark}

% \fi
\subsection{Uniform Boundedness of SGD with Momentum}\label{sec:mom}
We now extend the results to also cover the following momentum variant of SGD
\begin{subequations}
\begin{align}
    v_{k+1} & = \beta v_k + (1-\beta)g_k \label{equ:mom:v}\\
    x_{k+1} & = x_k - \eta_k v_{k+1}  \label{equ:mom:x}
\end{align}
\end{subequations}
where $\beta \in (0,1)$ and $x_0=x_1$. 
 We analyze the three common step-size policies: constant step-size, decaying step-sizes, and bandwidth step-sizes, respectively. Our first result is the following:

\begin{theorem}({\bf Constant step-size})\label{thm:mom:const}
Let $f$ be $L$-smooth and $R$-dissipative with a stochastic oracle satisfying Assumption~\ref{assump:gradient}. Consider SGD with momentum defined by (\ref{equ:mom:v}) and (\ref{equ:mom:x}) with $\beta \in (0,1)$, for any constant step-size 
\begin{align*}
\eta_k =  \eta \leq  \min\left\lbrace \frac{(1-\beta^2)}{\beta L\left(1-\beta + (1-\beta)^{-1} \right)},\,\, \frac{\theta_1}{2((1-\beta)^2+1)(\rho+1)L^2}\right\rbrace,
\end{align*} 
The, the quantities 
$\E[\left\| x_k - x^{\ast}\right\|^2]$ and $\E[f(x_k) - f^{\ast}]$ are uniformly bounded for all $k \in [1, T+1]$.
\end{theorem}
Theorem \ref{thm:mom:const} establishes  uniform boundedness properties of SGD with momentum for any constant step-size under mild conditions. The most challenging part in establishing Theorem~ \ref{thm:mom:const} is how to construct a Lyapunov function that is decreasing when $\Vert x_k-x^{\star}\Vert_2^2\geq R^2$. The details of the proofs are given in Appendix \ref{supple:mom}.

For the time dependent step-size, the analysis is more complicated than the constant step-size (see Theorem \ref{thm:mom:const}). But our next theorem shows that the uniform boundedness 
can also be guaranteed.
% Procedures:
% \begin{itemize}
%     \item First we need to show that although $\left\|x_{k} - x^{\ast} \right\|^2 \geq R^2$, $\left\|x_k - x^{\ast} \right\|^2$ will not go far from $R^2$. 
%     \item Next we show in this case (i.e., $\left\|x_k - x^{\ast} \right\|^2 \geq R^2$ but $\left\| x_{k-1} - x^{\ast}\right\|^2 \leq R^2$), $W_k$ is upper bounded
%     \item Then  once $\left\|x_k - x^{\ast} \right\|^2 \geq R^2$, then $W_k$ is decreasing, so we can conclude that $W_k$ is also uniformly bounded. 
    
% \end{itemize}
\begin{theorem}({\bf Decaying step-sizes})\label{thm:mom:decay}
Suppose that the objective function is $L$-smooth and $R$-dissipative,  and that the stochastic gradient oracle satisfies Assumption \ref{assump:gradient}. Consider SGD with momentum defined by (\ref{equ:mom:v}) and (\ref{equ:mom:x}) with a time-varying step-size $\eta_k \leq \min \left\lbrace \frac{\theta_1}{2(\rho+1)(1+(1-\beta)^2)L^2}, \frac{1-\beta^2}{\beta L(1-\beta + (1-\beta)^{-1}))} \right\rbrace$. Then both $\E[\left\| x_k - x^{\ast}\right\|^2]$ and $\E[f(x_k) - f^{\ast}]$ are uniformly bounded for all $k \in [1, T+1]$, under the following step-size policies:
\begin{itemize}
    \item[(1)] polynomial decaying $\eta_k = \eta_1/k^r$ for $r \in (0,1]$ and $\eta_1 \geq \frac{2}{\theta_1}$, for any $k \geq 3$
    \item[(2)] linearly decaying $\eta_k = A - Bk$ where $\eta_1 = \eta_{\max}$ and $\eta_{T} =  c/\sqrt{T}$ for $c  \geq \left(\frac{2\eta_{\max}^2}{\theta_1} \right)^{1/2}$
    \item[(3)] cosine decaying $\eta_k = A + B\cos\left(\frac{k\pi}{T}\right)$ where $\eta_1 = \eta_{\max}$ and $\eta_{T}= c/\sqrt{T}$ with $c \geq \max\left\lbrace \left(\frac{\eta_{\max}^2\pi^2}{2\theta_1}\right)^{1/2}, \frac{\eta_{\max}\pi^2}{4T^{3/2}} \right\rbrace$ 
    \item[(4)] exponentially decaying $\eta_k = \eta_1/\alpha^{k-1}$ where $\alpha = (\nu/T)^{-1/T} > 1$, $\nu \geq 1$, and $\eta_1 \geq \frac{2\ln(T/\nu) }{\theta_1 \nu}$. %Suppose that $ \frac{T}{\ln^2T} \geq \frac{4}{\eta_1\theta_1}$
    \end{itemize}

\end{theorem}

In the above theorem, we analyze four important families of decaying step-sizes that have excellent practical performance, see \emph{e.g.}~\citep{loshchilov2016sgdr} for the cosine step-size and~\cite{li2020exponential} for the exponential step-size. Our results demonstrate that the uniform boundedness of SGD with momentum can also be verified under these step-size policies.

Finally, we show that the uniform boundedness properties of SGD with momentum can be guaranteed also by stage-wise bandwidth step-sizes~\citep{wang2021convergence-stronglyconvex,wang2021bandwidth}, a class of step-size policies that allow for non-monotonic behaviour (bot within and between stages). The bandwidth framework not only covers the most popular stage-wise step-size (sometimes referred to as the step-decay step-size)~\citep{ge2019step, yuan2019stagewise, Xiaoyu-Step-decay}, but it also includes the popular cosine with restart step-size~\citep{loshchilov2016sgdr}. 

\begin{theorem}({\bf Stage-wise bandwidth-based step-size})\label{thm:mom:band}
 Given the total number of iterations $T$,  we consider the bandwidth step-size with the form $ \eta_{\min}^t \leq \eta_k \leq \eta_{\max}^t$ and its upper and lower bounds $\eta_{\min}^t$ and $\eta_{\max}^t$ are decreasing with $t$ where $t\in [N]$, $k \in [\sum_{i=1}^{t-1}S_i+1, \sum_{i=1}^{t}S_i]$, and $\sum_{t=1}^{N}S_t = T$.  We assume that the bandwidth $s = \eta_{\max}^t/\eta_{\min}^t$ is bounded and the step-size is decreasing at each stage $t$. Under the same setting as Theorem \ref{thm:mom:decay}, we consider the constant and decaying modes discussed in Theorems \ref{thm:mom:const} and \ref{thm:mom:decay}, then the quantities %$\E[W_{k}]$,
 $\E[\left\| x_k - x^{\ast}\right\|^2]$ and $\E[f(x_k) - f^{\ast}]$ for all $k \in [1, T+1]$ are uniformly bounded. 
 
\end{theorem}

Theorem \ref{thm:mom:band} gives us a lot of freedom to choose the lower and upper bounds $\eta_{\min}^t, \eta_{\max}^t $ and stage length $S_t$. For example $\eta_{\min}^t$ and $\eta_{\max}^t $ can be selected as $\mathcal{O}\left(1/t^p\right)$ with $p\in (0,1]$ or $\mathcal{O}\left(1/\alpha^t\right)$ with $\alpha > 1$. Although we restrict $s=\eta_{\max}^t / \eta_{\min}^t$ is bounded, the lower bound $\eta_{\min}^t$ and upper bound $\eta_{\max}^t$ may be of different orders. The stage length $S_t$ can be a constant ($S_t = S \geq 1$) or time dependent with $t$.  

\section{Non-convex Function with its Tails Growing Slower than Quadratic}\label{sec:general:sgd}

In this part, we extend the $R$-dissipativity of Section \ref{sec:dissipative} to include the non-convex function with its tail growing slower than quadratic, for example, logistic regression with $\ell_1$ regularizer.  When the point is far from the global minimizer $x^{\ast}$ (i.e. $\left\| x - x^{\ast}\right\| \geq R$), we consider the landscape of the objective function $f$ has the following properties:
 \begin{definition}\label{assump:dissipative:general}(Generalized $R$-dissipativity)
 Suppose that there exist constants $\theta_1 >0$, $\theta_2 \geq 0$, and $R >0$ such that for all $\left\| x -x^{\ast} \right\| \geq R$, the objective function satisfies that
 \begin{align}
     \left\langle \nabla f(x), x - x^{\ast} \right\rangle \geq \theta_1 \left\|x -x^{\ast} \right\|^p - \theta_2
 \end{align}
 
%  (b) $\left\| \nabla f(x) \right\|^2 \leq \theta_3(1+\left\|x-x^{\ast} \right\|^{2\tau})$
%     \begin{align}\label{dissipative:slow:quad}
%  \left\langle \nabla f(x), x - x^{\ast} \right\rangle \geq \theta_1 \left\|x -x^{\ast} \right\|^p - \theta_2, \quad \left\| \nabla f(x) \right\|^2 \leq \theta_3(1+\left\|x-x^{\ast} \right\|^{2\tau})
% \end{align}
with $p \in [0, 2)$.
\end{definition}
Definition \ref{assump:dissipative:general} generalized the $R$-dissipativity property in Section \ref{sec:dissipative} which extends to the problems whose tails grow slower than quadratic (e.g., superlinear or sublinear $\left\|x\right\|^{p}$ for $p \in [0,2)$). As we can see $p=2$ can be included Definition \ref{assump:dissipative:general}. However, to distinguish with $R$-dissipativity in Section \ref{sec:dissipative}, here we do not include $p=2$. 
% Assumption \ref{assump:dissipative:general}(b) shows that $\left\|\nabla f(x)\right\|$ goes slower than $\left\|x \right\|^{\tau}$. This generalized the $L$-smooth properties and is milder than $\tau$-h\"{o}lder continuous gradient ($\left\| \nabla f(x)-\nabla f(y)\right\| \leq L\left\|x-y\right\|^{\tau}$).
\begin{definition}($\tau$-growth gradient)\label{assump:grad:growth} The gradient of the objective function $f$ is $\tau$-growth  if
 $\left\| \nabla f(x) \right\|^2 \leq \theta_3(1+\left\|x - x^{\ast} \right\|^{2\tau})$ where $\tau \geq 0$ and $x \in \R^d$.
\end{definition}
The $\tau$-growth gradient condition is regarded as the extension of $L$-smoothness in Section \ref{sec:dissipative} and is coupled with generalized $R$-dissipativity in the following analysis. The generalized $R$-dissipativity where $p \in [1,2]$ together with $\tau$-growth gradient ($\tau \leq p/2$) has been introduced to develop the convergence rates for Langevin Monte Carlo algorithm~\citep[Assumption 2]{pmlr-v134-erdogdu21a}. In our definition, we allow $p \in (0,2)$. This includes the potential function may behave like $c \ln(x)$. 

Some potential functions for example $f(x) = c \log(1+\frac{1}{2}\left\|x\right\|^2)$ for $\left\| x \right\| \geq R$, which is sub-linear growth in $\left\| x \right\|$ satisfies generalized $R$-dissipativity with $\theta_1=2c, \theta_2=4c/3$, $R=1$, and $p=0$,  and its gradient is $\tau$-growth with $\tau=0$ and $\theta_3 = 2c^2/3$ for $\left\| x \right\| \geq 1$. Also, for any function $f$ whose tails grow like $c x^{\alpha}$ for $\alpha \in (0, 2]$ for some $c>0$, we have that the generalized $R$-dissipativity holds with $\theta_1=c\alpha, \theta_2 =0$, and $p=\alpha$, and the gradient is $\tau$-growth with $\tau = \alpha-1 < \alpha/2$. 
\begin{remark}(How to check generalized $R$-dissipativity)\label{rem:general:dissipative}
 Suppose that there exist two positive constants $\theta_1^{'}$ and $\theta_2^{'}$ such that
\begin{align}
\left\langle \nabla f(x), x \right\rangle \geq \theta_1^{'} \left\|x \right\|^p - \theta_2^{'}
\end{align}
If $x^{\ast} \neq 0$, let $R \geq \max \left\lbrace 2\left\|x^{\ast}\right\|,1 \right\rbrace$ and assume that gradient is $\tau$-growth with $\tau \leq p/2$, there exist two constant $\theta_1, \theta_2 \geq 0$ such that
\begin{align}
\left\langle \nabla f(x), x - x^{\ast}\right\rangle \geq  \theta_1 \left\| x -x^{\ast} \right\|^p - \theta_2
\end{align}
where  $\theta_1 = \frac{\theta_1^{'}}{2^{p+1}}$ and $\theta_2 = \theta_2^{'}  + \frac{\left\| x^{\ast} \right\|^{\alpha_2}}{s^{\alpha_2}\alpha_2}$ with $\alpha_2 = \frac{p}{p-\tau}$ and $s=\left(\frac{\theta_1^{'}p}{\tau2^{p+1}}\right)^{\tau/p}/\sqrt{2\theta_3}$. 
\end{remark}

\subsection{Applications of Generalized $R$-dissipativity}

Next we present some interesting applications in machine learning regime that satisfy the generalized $R$-dissipativity. 

{\bf Bayesian logistic regression~\citep{gauss-mix}}: $f(x) = \frac{1}{n}\sum_{i=1}^nb_i\log\left(1+\exp(- a_i^T{x})\right) + (1-b_i)\log(1+\exp(a_i^{T}x)) $ where $a_i$ is generated from the Gaussian distribution $N(0, \I_d)$ and $b_i|a_i$ is drawn according to the conditional probability from logistic regression models. This function satisfies generalized $R$-dissipativity with  $\theta_1 = c_1$, $\theta_2 = 0$, $R=1$, and $p=1$ for the universal constant $c_1 >0$. %Its gradient is $\tau$-growth with $\theta_3 = \frac{1}{n}\sum_{i=1}^n \left\|a_i\right\|^2$ and $\tau= 0$. 

{\bf Logistic regression with a $\ell_1$ regularizer} where $f(x) = \frac{1}{n}\sum_{i=1}^n\log(1 + \exp(-b_i a_i^{T}x)) + \lambda \left\| x \right\|_1$ satisfies generalized $R$-dissipativity with $\theta_1 = \lambda, \theta_2 = \frac{1}{2}$, $R=0$, and $p=1$. The gradient of $f$ is $\tau$-growth with $\theta = \frac{2}{n}(\sum_{i=1}^n \left\| a_i \right\|^2 + \lambda^2d)$ and $\tau = 0$.

% {\bf $\ell_1$-regularized problems:} where $f + \lambda \left\|x \right\|_1$ with $\lambda >0$ and $f$ is possibly non-convex and $\nabla f(x)$ is bounded by $G$. The function satisfies Assumption \ref{assump:dissipative:general} with $\theta_1^{'} = \lambda$, $\theta_2=$, $\theta_3$. $p=1$ and $\tau=0$. 

% \begin{align}
% \left\langle \nabla f(x) + \lambda sign(x) , x\right\rangle \geq \lambda \left\|x\right\|_1 - \left\|\nabla f(x) \right\|\left\|x\right\|
% \end{align}

{\bf Shallow neural networks with  $\ell_1$ regularizer} A subset of deep neural networks with $\ell_1$ regularization can be proved to satisfy generalized $R$-dissipativity. For instance, we consider a fully connected $S$-layer ($S-1$ hidden) neural network to train a binary classification dataset with $\lambda \left\| \right\|_1$. We use $\sigma_i(x) = \max(0,x)$ for $i=1,\cdots, S-1$ be the activation function of the inner layers, $\sigma_S(y) = 1/(1+\exp(-y))$ be the activation function of the output layer, and use cross-entropy to measure the output with its true label. This results the objective function satisfies that
\begin{align}
\left\langle  \nabla f(x), x\right\rangle \geq \lambda \left\|x \right\| - S
\end{align}
which implies that the generalized $R$-dissipativity holds with $\theta_1 = \lambda$, $\theta_2 = S$,$R=0$ and $p=1$.

\subsection{Uniform boundedness of SGD Under Generalized Dissipativity}\label{sec:sgd:general}
In the remaining of this section, we first show how the uniformly boundedness of SGD can be guaranteed under the generalized $R$-dissipativity.
\begin{theorem}\label{thm:general:sgd}
%Let the objective function be $L$-smooth.
Suppose that the objective function $f$ is generalized $R$-dissipativity with $p \in [0,2)$ and the gradient $\tau$-growth with $\tau \leq p/2 $ for $\left\|x -x^{\ast} \right\| \geq R$.   Consider the SGD algorithm with the stochastic gradient oracle satisfies Assumption \ref{assump:gradient}. For any step-size $\eta_k \leq \eta_{\max}$ where $\eta_{\max} $ is a positive constant, we have
\begin{align}
\E[\left\| x_{k} - x^{\ast} \right\|^2 ] \leq \min \left\lbrace \left\|x_1-x^{\ast} \right\|^2,   2(1 + \eta_{\max}^2(\rho+1)\theta_3)r^2 + 2 \eta_{\max}^2 \left(\sigma^2 +  (\rho+1)\theta_3\right)\right\rbrace
\end{align}
with $
 r^2 = \max \left\lbrace R^2, \left(8\eta_{\max}(\rho+1)\frac{\theta_3}{\theta_1}\right)^{1/(p-2\tau)}, \left(\frac{2\theta_2}{\theta_1} + \frac{\left(\sigma^2 + (\rho+1)\theta_3  \right)\eta_{\max}}{\theta_1}\right)^{2/p} \right\rbrace.
$
\end{theorem}
The theorem shows that for the functions whose tails grow slower than quadratic, the uniform boundedness property of SGD can also be guaranteed for any bounded step-sizes. In Definitions  \ref{assump:dissipative:general} and \ref{assump:grad:growth}, we assume that the objective function is differentiable. However, the above results also hold if the gradient in  Definitions  \ref{assump:dissipative:general} and \ref{assump:grad:growth} is replaced by a vector $g$ which belongs to its subgradient $\partial f(x)$ ($g \in \partial f(x)$).

\subsection{Uniform Boundedness of SGD with Momentum Under Generalized Dissipativity}
\label{sec:general:mom}
We then provide the analysis to cover the momentum variant of SGD under generalized $R$-dissipativity. 
\begin{theorem}(Constant step-size)\label{mom:general:const}
Let the objective function $f$ be $L$-smooth. We further assume that $f$  is generalized $R$-dissipativity with $ p\in (0,2)$ and its gradient is $\tau$-growth with $\tau \leq p/2$ for $\left\|x -x^{\ast} \right\| \geq R \geq 1$. Consider the SGD with momentum method defined in (\ref{equ:mom:v}) and (\ref{equ:mom:x}) with $\beta \in (0,1)$ and stochastic gradient oracle satisfies Assumption \ref{assump:gradient}, then for any constant step-size that 
\begin{align}
   \eta_k \leq \frac{1-\beta^2}{\beta L(1-\beta + (1-\beta)^{-1})} 
\end{align}
the quantities $\E[\left\|x_k -x^{\ast} \right\|^2]$, $\E[f(x_k) - f^{\ast}]$ are uniformly bounded for all $k \geq 1$.
\end{theorem}
Theorem \ref{mom:general:const} establishes the uniform boundedness of the SGD algorithm with momentum under the constant step-size. Compared to the analysis of SGD in Section \ref{sec:sgd:general}, here we also require that the objective function $f$ is $L$-smooth in $\R^d$. This is not contradictory with $\tau$-growth gradient for $\left\| x -x^{\ast} \right\| \geq R \geq 1$ because if the gradient is $\tau$-growth, of course, it also implies that $\left\| \nabla f(x) \right\|^2 \leq \theta_3(1+\left\|x - x^{\ast} \right\|^{2\tau}) \leq \theta_3(1+\left\|x - x^{\ast} \right\|^{2})$ for $\left\| x -x^{\ast} \right\| \geq R \geq 1$.

We then analyze the two popular time-varying step-sizes: polynomial decaying and exponential decaying step-size. The results below show the uniformly boundedness of SGD with momentum is also guaranteed.
\begin{theorem}(Time-decaying step-size)\label{mom:general:decaying}
Let the objective function $f$ be $L$-smooth. We further assume that $f$ satisfies Assumption \ref{assump:dissipative:general} and stochastic gradient oracle satisfies Assumption \ref{assump:gradient}. Consider the SGD with momentum method defined in (\ref{equ:mom:v}) and (\ref{equ:mom:x}) with $\beta \in (0,1)$, then $\E[\left\|x_k -x^{\ast} \right\|^2]$ and $\E[f(x_k) - f^{\ast}]$ are uniformly bounded for all $k \geq 1$, under the following step-sizes:
\begin{itemize}
    \item Polynomial decaying step-size $\eta_k = \eta_1/k^p$ for $p \in (0,1]$ where $\eta_k \leq  \frac{1-\beta^2}{2\beta L(1-\beta + (1-\beta)^{-1})} $;
    \item Exponential decaying step-size $\eta_k = \eta_1/\alpha^k$ with $\alpha = (\nu/T)^{-1/T}$ where $\eta_k \leq \frac{1-\beta^2}{\beta L(1-\beta + 2(1-\beta)^{-1})}$. 
\end{itemize}
\end{theorem}
Finally, based on the results of Theorems \ref{mom:general:const} and \ref{mom:general:decaying}, the uniform boundedness property also holds for three special stage-wise bandwidth step-sizes. The proof of the below theorem is very similar to that of Theorem \ref{thm:mom:band}, therefore we omit it in this note. 
\begin{theorem}(bandwidth-based stage-wise step-size)\label{mom:general:bandwidth}
Given the total number of iteration $T$, we consider the bandwidth step-size with the form $\eta_{\min}^t \leq \eta_k \leq \eta_{\max}^t $ and the upper and lower bounds are decreasing with $t \in [N]$ where $k \in [\sum_{i=1}^{t-1}S_i+1, \sum_{i=1}^t S_i]$ and $\sum_{i=1}^N S_i = T$. We assume the bandwidth $s = \eta_{\max}^t/\eta_{\min}^t$ is bounded and step-size is decreasing at each stage $t$. Under the same setting as Theorem \ref{mom:general:decaying}, we consider the constant and decaying modes discussed in Theorem \ref{mom:general:decaying}, then $\E[\left\|x_k -x^{\ast} \right\|^2]$ and $\E[f(x_k) - f^{\ast}]$ are uniformly bounded for all $k \geq 1$.
\end{theorem}

% \begin{remark}(Tail growing faster than quadratic)
% Give some discussion or counter example to show that it will diverge. 

% \end{remark}

\section{Conclusion}
We have provided uniform boundedness guarantees for SGD and SGD with momentum under the $L$-smoothness and $R$-dissipativity conditions. Our results allow for broad families of step-sizes that are only upper bounded by a constant $\mathcal{O}(1/L)$. We have demonstrated that, $\E[\left\|x_k - x^{\ast} \right\|^2]$ and $\E[f(x_k) - f^{\ast}]$  stay bounded along every possible trajectory, even under step-sizes that do not satisfy the Robbins-Monro conditions (square summable but not summable). The uniform bounds that we have derived depend on loss function properties, noise parameters $\sigma^2$ and $\rho$ of the stochastic oracle, and the initial state of the algorithm. 

$R$-dissipativity captures non-convex functions that grow quadratically when we move (possibly far) away from the global minimizer. As shown in Section~\ref{sec:application}, several interesting applications in machine learning satisfy this regularity condition. Furthermore, 
uniform boundedness properties can be guaranteed also for loss functions with slower asymptotic growth, such as Bayesian logistic regression. For the SGD algorithm, we have proved that, $\E[\left\|x_k - x^{\ast} \right\|^2]$ and $\E[f(x_k) - f^{\ast}]$ stay bounded under a broad class of step-size that is only required to be upper bounded. The uniformly boundedness properties of SGD with momentum can be guaranteed for several popular step-sizes including constant step-size, polynomial decaying and exponential decaying step-size, and also the stage-wise bandwidth step-sizes which locally decay like constant, polynomial and exponential functions.

%It would be interesting to know whether uniform boundedness properties can be guaranteed also for loss functions with slower asymptotic growth, such as Bayesian logistic regression.  We leave this question for future work. 

% \section*{Acknowledgments}
% This was was supported in part by......

%Bibliography
%\bibliographystyle{unsrt}  
\bibliography{references}

\begin{thebibliography}{34}
\providecommand{\natexlab}[1]{#1}
\providecommand{\url}[1]{\texttt{#1}}
\expandafter\ifx\csname urlstyle\endcsname\relax
  \providecommand{\doi}[1]{doi: #1}\else
  \providecommand{\doi}{doi: \begingroup \urlstyle{rm}\Url}\fi

\bibitem[Chen et~al.(2021)Chen, Mou, and Maguluri]{chen2021stationary}
Z.~Chen, S.~Mou, and S.~T. Maguluri.
\newblock Stationary behavior of constant stepsize sgd type algorithms: An
  asymptotic characterization.
\newblock \emph{arXiv preprint arXiv:2111.06328}, 2021.

\bibitem[Davis et~al.(2020)Davis, Drusvyatskiy, Kakade, and
  Lee]{davis2020stochastic}
D.~Davis, D.~Drusvyatskiy, S.~Kakade, and J.~D. Lee.
\newblock Stochastic subgradient method converges on tame functions.
\newblock \emph{Foundations of computational mathematics}, 20\penalty0
  (1):\penalty0 119--154, 2020.

\bibitem[Erdogdu and Hosseinzadeh(2021)]{pmlr-v134-erdogdu21a}
M.~A. Erdogdu and R.~Hosseinzadeh.
\newblock On the convergence of langevin monte carlo: The interplay between
  tail growth and smoothness.
\newblock In M.~Belkin and S.~Kpotufe, editors, \emph{Proceedings of Thirty
  Fourth Conference on Learning Theory}, volume 134 of \emph{Proceedings of
  Machine Learning Research}, pages 1776--1822. PMLR, 15--19 Aug 2021.

\bibitem[Erdogdu et~al.(2018)Erdogdu, Mackey, and Shamir]{Erdogdu-2018-}
M.~A. Erdogdu, L.~Mackey, and O.~Shamir.
\newblock Global non-convex optimization with discretized diffusions.
\newblock In S.~Bengio, H.~Wallach, H.~Larochelle, K.~Grauman, N.~Cesa-Bianchi,
  and R.~Garnett, editors, \emph{Advances in Neural Information Processing
  Systems}, volume~31, 2018.

\bibitem[Ge et~al.(2019)Ge, Kakade, Kidambi, and Netrapalli]{ge2019step}
R.~Ge, S.~M. Kakade, R.~Kidambi, and P.~Netrapalli.
\newblock The step decay schedule: A near optimal, geometrically decaying
  learning rate procedure for least squares.
\newblock In \emph{Advances in Neural Information Processing Systems}, pages
  14977--14988, 2019.

\bibitem[Hale(1988)]{hale1988asymptotic}
J.~K. Hale.
\newblock Asymptotic behavior of dissipative systems.
\newblock In \emph{AMS}, 1988.

\bibitem[Hazan and Kale(2014)]{hazan2014beyond}
E.~Hazan and S.~Kale.
\newblock Beyond the regret minimization barrier: optimal algorithms for
  stochastic strongly-convex optimization.
\newblock \emph{Journal of Machine Learning Research}, 15\penalty0
  (1):\penalty0 2489--2512, 2014.

\bibitem[Krogh and Hertz(1992)]{weight-decay}
A.~Krogh and J.~A. Hertz.
\newblock A simple weight decay can improve generalization.
\newblock In \emph{Advances in neural information processing systems}, pages
  950--957, 1992.

\bibitem[Li et~al.(2021)Li, Zhuang, and Orabona]{li2020exponential}
X.~Li, Z.~Zhuang, and F.~Orabona.
\newblock A second look at exponential and cosine step sizes: Simplicity,
  adaptivity, and performance.
\newblock In \emph{International Conference on Machine Learning}, pages
  6553--6564. PMLR, 2021.

\bibitem[Ljung(1977)]{ljung1977analysis}
L.~Ljung.
\newblock Analysis of recursive stochastic algorithms.
\newblock \emph{IEEE transactions on automatic control}, 22\penalty0
  (4):\penalty0 551--575, 1977.

\bibitem[Loshchilov and Hutter(2017)]{loshchilov2016sgdr}
I.~Loshchilov and F.~Hutter.
\newblock {SGDR}: Stochastic gradient descent with warm restarts.
\newblock In \emph{5th International Conference on Learning Representations,
  {ICLR} 2017, Toulon, France}, 2017.

\bibitem[Mattingly et~al.(2002)Mattingly, Stuart, and
  Higham]{mattingly2002ergodicity}
J.~C. Mattingly, A.~M. Stuart, and D.~J. Higham.
\newblock Ergodicity for sdes and approximations: locally lipschitz vector
  fields and degenerate noise.
\newblock \emph{Stochastic processes and their applications}, 101\penalty0
  (2):\penalty0 185--232, 2002.

\bibitem[Mertikopoulos et~al.(2020)Mertikopoulos, Hallak, Kavis, and
  Cevher]{mertikopoulos2020almost}
P.~Mertikopoulos, N.~Hallak, A.~Kavis, and V.~Cevher.
\newblock On the almost sure convergence of stochastic gradient descent in
  non-convex problems.
\newblock \emph{Advances in Neural Information Processing Systems}, 33, 2020.

\bibitem[Mou et~al.(2019)Mou, Ho, Wainwright, Bartlett, and Jordan]{gauss-mix}
W.~Mou, N.~Ho, M.~J. Wainwright, P.~Bartlett, and M.~I. Jordan.
\newblock A diffusion process perspective on posterior contraction rates for
  parameters.
\newblock \emph{arXiv preprint arXiv:1909.00966}, 2019.

\bibitem[Moulines and Bach(2011)]{Moulines-Bach2011}
E.~Moulines and F.~R. Bach.
\newblock Non-asymptotic analysis of stochastic approximation algorithms for
  machine learning.
\newblock In \emph{Advances in Neural Information Processing Systems}, pages
  451--459, 2011.

\bibitem[Park and Liu(2011)]{robust-logistic}
S.~Y. Park and Y.~Liu.
\newblock Robust penalized logistic regression with truncated loss functions.
\newblock \emph{Canadian Journal of Statistics}, 39\penalty0 (2):\penalty0
  300--323, 2011.

\bibitem[Patel and Zhang(2021)]{patel2021stochastic}
V.~Patel and S.~Zhang.
\newblock Stochastic gradient descent on nonconvex functions with general noise
  models.
\newblock \emph{arXiv preprint arXiv:2104.00423}, 2021.

\bibitem[Polyak(1964)]{polyak1964some}
B.~T. Polyak.
\newblock Some methods of speeding up the convergence of iteration methods.
\newblock \emph{Ussr computational mathematics and mathematical physics},
  4\penalty0 (5):\penalty0 1--17, 1964.

\bibitem[Raginsky et~al.(2017)Raginsky, Rakhlin, and
  Telgarsky]{raginsky2017non}
M.~Raginsky, A.~Rakhlin, and M.~Telgarsky.
\newblock Non-convex learning via stochastic gradient langevin dynamics: a
  nonasymptotic analysis.
\newblock In \emph{Conference on Learning Theory}, pages 1674--1703. PMLR,
  2017.

\bibitem[Robbins and Monro(1951)]{SGD-1951}
H.~Robbins and S.~Monro.
\newblock A stochastic approximation method.
\newblock \emph{The Annals of Mathematical Statistics}, pages 400--407, 1951.

\bibitem[Sebbouh et~al.(2021)Sebbouh, Gower, and Defazio]{sebbouh2021almost}
O.~Sebbouh, R.~M. Gower, and A.~Defazio.
\newblock Almost sure convergence rates for stochastic gradient descent and
  stochastic heavy ball.
\newblock In \emph{Conference on Learning Theory}, pages 3935--3971. PMLR,
  2021.

\bibitem[Shi et~al.(2020)Shi, Su, and Jordan]{shi2020learning}
B.~Shi, W.~J. Su, and M.~I. Jordan.
\newblock On learning rates and schr$\backslash$" odinger operators.
\newblock \emph{arXiv preprint arXiv:2004.06977}, 2020.

\bibitem[Sun and Luo(2016)]{sun-matrix-completion}
R.~Sun and Z.-Q. Luo.
\newblock Guaranteed matrix completion via non-convex factorization.
\newblock \emph{IEEE Transactions on Information Theory}, 62\penalty0
  (11):\penalty0 6535--6579, 2016.

\bibitem[Sutskever et~al.(2013)Sutskever, Martens, Dahl, and
  Hinton]{sutskever2013importance}
I.~Sutskever, J.~Martens, G.~Dahl, and G.~Hinton.
\newblock On the importance of initialization and momentum in deep learning.
\newblock In \emph{International Conference on Machine Learning}, pages
  1139--1147. PMLR, 2013.

\bibitem[Tan and Vershynin(2019)]{tan2019online}
Y.~S. Tan and R.~Vershynin.
\newblock Online stochastic gradient descent with arbitrary initialization
  solves non-smooth, non-convex phase retrieval.
\newblock \emph{arXiv preprint arXiv:1910.12837}, 2019.

\bibitem[Wang and Johansson(2021)]{wang2021bandwidth}
X.~Wang and M.~Johansson.
\newblock Bandwidth-based step-sizes for non-convex stochastic optimization.
\newblock \emph{arXiv preprint arXiv:2106.02888}, 2021.

\bibitem[Wang and Yuan(2021)]{wang2021convergence-stronglyconvex}
X.~Wang and Y.-x. Yuan.
\newblock On the convergence of stochastic gradient descent with
  bandwidth-based step size.
\newblock \emph{arXiv preprint arXiv:2102.09031}, 2021.

\bibitem[Wang et~al.(2021)Wang, Magn{\'u}sson, and
  Johansson]{Xiaoyu-Step-decay}
X.~Wang, S.~Magn{\'u}sson, and M.~Johansson.
\newblock On the convergence of step decay step-size for stochastic
  optimization.
\newblock In \emph{Advances in Neural Information Processing Systems}, 2021.

\bibitem[Wojtowytsch(2021)]{wojtowytsch2021stochastic}
S.~Wojtowytsch.
\newblock Stochastic gradient descent with noise of machine learning type. part
  i: Discrete time analysis.
\newblock \emph{arXiv preprint arXiv:2105.01650}, 2021.

\bibitem[Xu et~al.(2019{\natexlab{a}})Xu, Qi, Lin, Jin, and
  Yang]{xu2019stochastic}
Y.~Xu, Q.~Qi, Q.~Lin, R.~Jin, and T.~Yang.
\newblock Stochastic optimization for dc functions and non-smooth non-convex
  regularizers with non-asymptotic convergence.
\newblock In \emph{International Conference on Machine Learning}, pages
  6942--6951. PMLR, 2019{\natexlab{a}}.

\bibitem[Xu et~al.(2019{\natexlab{b}})Xu, Yuan, Yang, Jin, and Yang]{Yang-2019}
Y.~Xu, Z.~Yuan, S.~Yang, R.~Jin, and T.~Yang.
\newblock On the convergence of (stochastic) gradient descent with
  extrapolation for non-convex minimization.
\newblock In \emph{Proceedings of the Twenty-Eighth International Joint
  Conference on Artificial Intelligence (IJCAI-19)}, 2019{\natexlab{b}}.

\bibitem[Xu et~al.(2020)Xu, Zhu, Yang, Zhang, Jin, and Yang]{pmlr-v115-xu20b}
Y.~Xu, S.~Zhu, S.~Yang, C.~Zhang, R.~Jin, and T.~Yang.
\newblock Learning with non-convex truncated losses by sgd.
\newblock In \emph{Proceedings of The 35th Uncertainty in Artificial
  Intelligence Conference}, volume 115, pages 701--711. PMLR, 22--25 Jul 2020.

\bibitem[Yu et~al.(2021)Yu, Balasubramanian, Volgushev, and Erdogdu]{Lu-2021}
L.~Yu, K.~Balasubramanian, S.~Volgushev, and M.~A. Erdogdu.
\newblock An analysis of constant step size sgd in the non-convex regime:
  Asymptotic normality and bias.
\newblock In \emph{Advances in Neural Information Processing Systems}, 2021.

\bibitem[Yuan et~al.(2019)Yuan, Yan, Jin, and Yang]{yuan2019stagewise}
Z.~Yuan, Y.~Yan, R.~Jin, and T.~Yang.
\newblock Stagewise training accelerates convergence of testing error over sgd.
\newblock In \emph{Proceedings of the 33rd International Conference on Neural
  Information Processing Systems}, pages 2608--2618, 2019.

\end{thebibliography}

\appendix
\section{Supplementary Theoretical Results for SGD}
\label{supply:sgd}
\begin{lemma}\label{lem:sgd:iter}
We consider the SGD algorithm and assume that the noise of the stochastic gradient satisfies Assumption \ref{assump:gradient}. Then
\begin{align*}
 \E[\left\|x_{k+1} - x^{\ast} \right\|^2 \mid \mathcal{F}_k] \leq \left\| x_k - x^{\ast}\right\|^2 -  2\eta_k \left\langle \nabla f(x_k), x_k - x^{\ast}\right\rangle + \eta_k^2\left( (1+\rho)\left\| \nabla f(x_k)\right\|^2 + \sigma^2 \right)
\end{align*}
\begin{proof}
Applying the iteration of the SGD algorithm, we have
\begin{align}\label{inequ:sgd:lem:1}
 \left\|x_{k+1} - x^{\ast} \right\|^2 & = \left\|x_k  - \eta_k g_k - x^{\ast} \right\|^2
  = \left\| x_k - x^{\ast}\right\|^2 - 2\eta_k \left\langle g_k, x_k - x^{\ast}\right\rangle + \eta_k^2 \left\|g_k\right\|^2.
\end{align} 
By Assumption \ref{assump:gradient}, we can estimate the expectation of $\left\|g_k\right\|^2$ as below
\begin{align}\label{inequ:grad}
    \E[\left\|g_k \right\|^2 \mid \mathcal{F}_k]  = \E[\left\|g_k - \nabla f(x_k) + \nabla f(x_k) \right\|^2 \mid \mathcal{F}_k] & = \E[\left\|g_k - \nabla f(x_k) \right\|^2 \mid \mathcal{F}_k] + \left\|\nabla f(x_k) \right\|^2  \notag \\
    & \leq \sigma^2 + (\rho+1)\left\|\nabla f(x_k) \right\|^2.
\end{align}
Taking conditional expectation with respect to $\mathcal{F}_k$ on the both sides of (\ref{inequ:sgd:lem:1}) and applying (\ref{inequ:grad}) gives
\begin{align}
 \E\left[\left\|x_{k+1} - x^{\ast} \right\|^2 \mid \mathcal{F}_k\right] & = \left\| x_k - x^{\ast}\right\|^2 -  2\eta_k \left\langle \nabla f(x_k), x_k - x^{\ast}\right\rangle + \eta_k^2\left(\E\left[ \left\|g_k - \nabla f(x_k) + \nabla f(x_k)\right\|^2 \right]\right) \notag \\
 & =   \left\| x_k - x^{\ast}\right\|^2 -  2\eta_k \left\langle \nabla f(x_k), x_k - x^{\ast}\right\rangle + \eta_k^2\left( \E[\left\|g_k - \nabla f(x_k)\right\|^2] + \left\| \nabla f(x_k)\right\|^2 \right) \notag \\
 & \leq \left\| x_k - x^{\ast}\right\|^2 -  2\eta_k \left\langle \nabla f(x_k), x_k - x^{\ast}\right\rangle + \eta_k^2\left(  (1+\rho)\left\| \nabla f(x_k)\right\|^2 + \sigma^2\right). \notag 
\end{align}
\end{proof}
\end{lemma}
\begin{proof}({\bf of Theorem \ref{thm:sgd:general}})
Let $r^2 = \max \left\lbrace R^2, \frac{2\theta_2}{\theta_1} + \frac{\sigma^2}{(1+\rho)L^2} \right\rbrace$. In this theorem, we consider two different cases based on the distance of the initial point $x_1$ with the global minimizer $x^{\ast}$.  First, we consider the initial point $\left\| x_1 - x^{\ast} \right\| \leq r$. Suppose that at time $k$, $\E[\left\| x_k - x^{\ast}\right\|^2 \mid \mathcal{F}_{k-1}] > r^2$ but $\E[\left\| x_{k-1} - x^{\ast}\right\|^2 \mid \mathcal{F}_{k-2}] \leq r^2$. We can see that 
\begin{align}
\E[\left\| x_{k} - x^{\ast} \right\|^2 \mid \mathcal{F}_{k-1}] & = \E[\left\| x_k - x_{k-1} + x_{k-1} - x^{\ast} \right\|^2 \mid \mathcal{F}_{k-1}] \leq 2\E[\left\| x_k - x_{k-1} \right\|^2 \mid \mathcal{F}_{k-1}] + 2\left\|x_{k-1} - x^{\ast} \right\|^2 \notag  \\
& \leq 2 \eta_{k-2}^2 \E[\left\| g_{k-1} \right\|^2\mid \mathcal{F}_{k-1}] + 2\left\|x_{k-1} - x^{\ast} \right\|^2  \leq 2 \eta_{k-1}^2\left(\sigma^2 + L^2\left\|x_{k-1}-x^{\ast} \right\|^2 \right) + 2\left\|x_{k-1} - x^{\ast} \right\|^2 \notag \\ 
& = 2 \eta_{k-1}^2 \sigma^2 + 2(\eta_{k-1}^2L^2+1)\left\|x_{k-1}-x^{\ast} \right\|^2 \notag \\
& < 2 \eta_{k-1}^2 \sigma^2 + 2(\eta_{k-1}^2L^2+1)r^2 \leq \frac{2(\sigma^2+L^2r^2)\theta_1^2}{(1+\rho)^2L^4} + 2r^2.
\end{align}
This means that 
when $\E[\left\| x_k - x^{\ast} \right\|^2 ] > r^2 $, then $R$-dissipativity holds. %Suppose that the iteration $s$ is the first iteration after $k$ such that $\left\| x_{k+s} - x^{\ast} \right\| < R$. It means that for all $ t \in [k, k+s)$, $\left\| x_t - x^{\ast} \right\| > R$.  Then for any iteration $t \in [k, k+s)$,
By Lemma \ref{lem:sgd:iter}, we have
\begin{align}\label{inequ:iter}
\E[\left\| x_{k+1} - x^{\ast} \right\|^2 \mid \mathcal{F}_k] & = \left\| x_k - x^{\ast}\right\|^2 -  2\eta_k \left\langle \nabla f(x_k), x_k - x^{\ast}\right\rangle + \eta_k^2\left(  (1+\rho)\left\| \nabla f(x_k)\right\|^2 + \sigma^2\right) \notag \\
& \leq \left\| x_k - x^{\ast}\right\|^2 + (- 2\theta_1\eta_k +  (1+\rho)L^2\eta_k^2)\left\| x_k - x^{\ast}\right\|^2  + 2\theta_2 \eta_k + \eta_k^2\sigma^2.
\end{align}
By $\eta_k \leq \frac{\theta_1}{(\rho+1)L^2}$, we have $ - 2\theta_1\eta_k +   (1+\rho)L^2\eta_k^2 \leq - \theta_1 \eta_k$. Then
\begin{align}\label{inequ:iter:2}
\E[\left\| x_{k+1} - x^{\ast} \right\|^2\mid \mathcal{F}_k] & \leq \left\| x_k - x^{\ast}\right\|^2 - \theta_1\eta_k \left\| x_k - x^{\ast}\right\|^2  + 2\theta_2 \eta_k + \eta_k^2\sigma^2.
\end{align}
Once $\E[\left\| x_k - x^{\ast}\right\|^2] > r^2$, 
for any kinds of step-sizes as long as $\eta_k \leq \frac{\theta_1}{(\rho+1)L^2}$, if $r$ is large enough for example
\begin{align*}
    r^2 = \max \left\lbrace R^2,  \frac{2\theta_2}{\theta_1} + \frac{\sigma^2}{(1+\rho)L^2}  \right\rbrace
\end{align*}
to make sure that 
\begin{align*}
- \theta_1\eta_k \left\| x_k - x^{\ast}\right\|^2 + 2\theta_2 \eta_k + \eta_k^2\sigma^2 < 0
\end{align*}
Then 
\begin{align*}
\E[\left\| x_{k+1} - x^{\ast} \right\|^2 \mid \mathcal{F}_k] <  \left\| x_k - x^{\ast}\right\|^2. 
\end{align*}
That is to say, once there exist an iterate $k$ such that $\E[\left\| x_k - x^{\ast} \right\|^2] > r^2$, for any step-size $\eta_k \leq \frac{\theta_1}{(\rho+1)L^2}$, %if Assumption (\ref{dissipative:assump}) holds with $R^2 \geq \frac{\theta_2}{\theta_1} + \frac{\sigma^2}{(1+\rho)L^2}$,
then the follow-up $\E[\left\| x_{k+1} - x^{\ast} \right\|^2]$ is decreasing. 

If the initial point $x_1$ is far from the optimal point $x^{\ast}$, that is $\left\| x_1 - x^{\ast} \right\|^2 > r^2$ but $\left\| x_1 - x^{\ast} \right\|$ is bounded, by applying the above statements, we can see that the follow-up $\E[\left\| x_{k+1} - x^{\ast} \right\|^2]$ is decreasing until $\E[\left\| x_{k+1} - x^{\ast} \right\|^2 ] \leq r^2$. 
% \begin{align}
% \theta_2 \eta_t + \eta_t^2\sigma^2 \leq \frac{\theta_1\theta_2}{(\rho+1)L^2} + \frac{\theta_1^2\sigma^2}{(\rho+1)^2L^4} < R. 
% \end{align}
% With loss of generalization, we let $R = \frac{2\theta_1\theta_2}{(\rho+1)L^2} + \frac{2\theta_1^2\sigma^2}{(\rho+1)^2L^4}$. Applying the recursion (\ref{inequ:iter:2}) $s$ times starting from $t=k$, we have
% \begin{align}\label{inequ:iter:2}
% \E[\left\| x_{t+s} - x^{\ast} \right\|^2] & \leq \Pi_{i=1}^s(1- \theta_1\eta_{k+i})\left\| x_{k} - x^{\ast}\right\|^2  + \sum_{l=1}^{s}\left(\theta_2 \eta_{k+l} + \eta_{k+l}^2\sigma^2 \right)\Pi_{i \geq l}^{s}(1- \theta_1\eta_{k+i}).
% \end{align}

\end{proof}

\begin{proof}({\bf of Remark \ref{rem:dissipative}})
In this proof, we show that $R$-dissipativity condition can be derived from
\begin{align}
   \left\langle x, \nabla f(x)\right\rangle \geq \theta_1^{'} \left\|x\right\|^2 - \theta_2^{'}  
\end{align}
for all $\left\| x - x^{\ast} \right\| \geq R$. 
When $x^{\ast} = 0$, the two conditions are the same. Suppose that $\left\| x^{\ast} \right\| > 0$, then 
\begin{align}
 \left\langle x, \nabla f(x)\right\rangle \geq \theta_1^{'} \left\|x\right\|^2 - \theta_2^{'}   \geq \theta_1^{'} \left(\left\|x - x^{\ast}\right\|^2 + \left\|x^{\ast} \right\|^2 - 2 \left\|x^{\ast} \right\|\left\|x-x^{\ast} \right\| \right) - \theta_2^{'}
\end{align}
and
\begin{align}
 \left\langle \nabla f(x), x^{\ast}\right\rangle  \mathop{\leq}^{(a)} \left\| \nabla f(x) \right\| \left\|x^{\ast} \right\| \leq L\left\|x - x^{\ast}\right\|\left\|x^{\ast} \right\|
\end{align}
where $(a)$ follows from the fact that $\left\| \nabla f(x) \right\| \leq L\left\|x-x^{\ast} \right\|$ (due to $L$-smoothness).
Incorporating the above results gives that
\begin{align}
\left\langle x - x^{\ast}, \nabla f(x)\right\rangle & \geq \theta_1^{'}\left\|x - x^{\ast}\right\|^2 - \left(2\theta_1^{'} + L \right)\left\|x^{\ast} \right\|\left\|x - x^{\ast}\right\| - \theta_2^{'}  + \theta_1^{'} \left\|x^{\ast} \right\|^2\notag \\
& \geq \frac{\theta_1^{'}}{2}\left\|x - x^{\ast}\right\|^2 - \left(\frac{(2\theta_1^{'}+L)^2}{2\theta_1^{'}} - \theta_1^{'} \right)\left\| x^{\ast} \right\|^2- \theta_2^{'}. 
\end{align}
Thus the $R$-dissipativity condition %\ref{dissipative:assump}
holds with $\theta_1 = \frac{\theta_1^{'}}{2}$ and $\theta_2 = \left(\theta_1^{'} + 2L + \frac{L^2}{2\theta_1^{'}}\right)\left\| x^{\ast} \right\|^2 + \theta_2^{'}$.

\end{proof}

\section{Supplementary Material for SGD with Momentum}\label{supple:mom}

In this part, we provide supplementary proofs for theorems in Section \ref{sec:mom}. Before given the main proofs, we first show some extra and useful lemmas. 
\begin{lemma}\label{lem:mom:1}
Suppose that the objective function $f$ is $L$-smooth and the stochastic gradient satisfies Assumption \ref{assump:gradient}. Consider SGD with momentum with the momentum parameter $\beta \in (0,1)$, for any step-size $\eta_k > 0$, we have
\begin{align*}
& \E[\left\| x_{k+1} - x_k \right\|^2 \mid \mathcal{F}_k ]  \notag \\
& = \frac{\eta_k^2\beta^2}{\eta_{k-1}^2} \left\| x_k - x_{k-1} \right\|^2 + \eta_k^2(1-\beta)^2\E[\left\| g_k \right\|^2\mid \mathcal{F}_k ] + \frac{\eta_k^2\beta(1-\beta)}{\eta_{k-1}}\left(f(x_{k-1}) - f(x_k) + \frac{L}{2}\left\| x_k - x_{k-1}\right\|^2\right).
\end{align*}
\end{lemma}
\begin{proof}({\bf of Lemma \ref{lem:mom:1}})
Applying the update recursion of the SGD with momentum algorithm in (\ref{equ:mom:v}) and (\ref{equ:mom:x}) gives
\begin{align}\label{mom:inequ:xx}
 \E[\left\| x_{k+1} - x_k \right\|^2 \mid \mathcal{F}_k ]  &= \E[\left\|\eta_k v_{k+1} \right\|^2 \mid \mathcal{F}_k ] = \eta_k^2\, \E[\left\| \beta v_k + (1-\beta) g_k \right\|^2 \mid \mathcal{F}_k ] \notag \\
& = \eta_k^2\E\left[\left\| \beta \eta_{k-1}^{-1}(x_{k-1} -x_{k}) + (1-\beta) g_k \right\|^2 \mid \mathcal{F}_k \right] \notag \\
& = \frac{\eta_k^2\beta^2 }{\eta_{k-1}^2}\left\| x_k - x_{k-1} \right\|^2 + \eta_k^2(1-\beta)^2\E[\left\| g_k \right\|^2\mid \mathcal{F}_k] + \frac{2\eta_k^2\beta(1-\beta)}{\eta_{k-1}}\left\langle  x_{k-1}- x_k, \nabla f(x_k) \right\rangle
\end{align}
where $\E[g_k \mid \mathcal{F}_k] = \nabla f(x_k)$.
Using the $L$-smooth assumption, for any $x, y \in \R^d$, we have
\begin{align*}
\left\| \nabla f(x) - \nabla f(y) \right\| \leq L \left\| x - y\right\|.
\end{align*}
Let $x = x_{k-1}$ and $y=x_k$, the $L$-smoothness also implies that
\begin{align}\label{L-smooth:inequ}
f(x_{k-1} ) \geq f(x_{k}) + \left\langle  x_{k-1}- x_k, \nabla f(x_k) \right\rangle - \frac{L}{2}\left\| x_k - x_{k-1}\right\|^2.
\end{align}
Applying the above inequality into (\ref{mom:inequ:xx}), we have
\begin{align*}
& \E[\left\| x_{k+1} - x_k \right\|^2 \mid \mathcal{F}_k ] \notag \\
& \leq \frac{\eta_k^2}{\eta_{k-1}^2}\beta^2 \left\| x_k - x_{k-1} \right\|^2 + \eta_k^2(1-\beta)^2\E[\left\| g_k \right\|^2 \mid \mathcal{F}_k] + \frac{2\eta_k^2\beta(1-\beta)}{\eta_{k-1}}\left(f(x_{k-1}) - f(x_k) + \frac{L}{2}\left\| x_k - x_{k-1}\right\|^2\right).
\end{align*}
Then the proof is complete. 
\end{proof}
\begin{lemma}\label{lem:mom:2}
We define $\tilde{x}_{k+1} := \frac{x_{k+1} - \beta x_k}{1-\beta}$. Suppose that Assumption \ref{assump:gradient} and (\ref{dissipative:assump} ) hold at current iterations $x_k$ and $\tau_k =\eta_k/\eta_{k-1} \in (0,1]$, then
\begin{itemize}
    \item[(i)] If $\tau_k = 1$, \emph{i.e.}, $\eta_k = \eta$ for all $k \geq 1$, we have
    \begin{align*}
      \E[\left\| \tilde{x}_{k+1} - x^{\ast} \right\|^2 \mid \mathcal{F}_k] & \leq \left\|\tilde{x}_k - x^{\ast} \right\|^2 - 2\theta_1 \eta_k \left\| x_k - x^{\ast} \right\|^2  + 2\theta_2 \eta_k  + \eta_k^2 \E[\left\|g_k\right\|^2 \mid \mathcal{F}_k] \notag \\
      & \quad + \frac{2\beta\eta_k}{1-\beta}\left(f(x_{k-1}) - f(x_k) + \frac{L}{2}\left\| x_k - x_{k-1} \right\|^2\right).
    \end{align*}
    \item[(ii)] else if $\tau_k \in (0,1)$, we have
    \begin{align*}
     \E[\left\| \tilde{x}_{k+1} - x^{\ast} \right\|^2 \mid \mathcal{F}_k] & \leq \tau_k\left\|\tilde{x}_k - x^{\ast} \right\|^2 - \left(2\eta_k\theta_1 -\left(1- \tau_k\right) \right)\left\| x_k - x^{\ast} \right\|^2   + 2\eta_k \theta_2 + \eta_k^2 \E[\left\|g_k\right\|^2 \mid \mathcal{F}_k]\notag \\
& \quad + \frac{2\beta  L}{1-\beta} \cdot \eta_k\tau_k \left\|x_{k-1}-x_{k}\right\|^2 + \frac{2\beta \tau_k \eta_k}{1-\beta}\left( f(x_{k-1}) - f(x_k) \right).
     \end{align*}
\end{itemize}
\end{lemma}
\begin{proof}({\bf of Lemma \ref{lem:mom:2}})
Recalling the definition of $\tilde{x}_{k+1}$ and applying the recursion of SGD with momentum, we have
\begin{align}\label{inequ:tildex}
\left\| \tilde{x}_{k+1} - x^{\ast} \right\|^2 
   & =   \left\|(1-\beta)^{-1} \left(x_{k+1} - \beta x_k \right) -  x^{\ast} \right\|^2 = \left\| (1-\beta)^{-1}\left(x_{k} - \eta_k (\beta v_k + (1-\beta)g_k) - \beta x_k\right) - x^{\ast} \right\|^2 \notag \\
   & = \left\| x_{k} - \frac{\eta_k}{1-\beta} \left(\beta \left(\frac{ x_{k-1} - x_{k}}{\eta_{k-1}}\right) + (1-\beta)g_k \right) - x^{\ast} \right\|^2 \notag \\
   & = \left\|\frac{\eta_k}{(1-\beta)\eta_{k-1}}\left(x_k - \beta x_{k-1}\right) + \left(1- \frac{\eta_k}{\eta_{k-1}}\right)x_k - \eta_k g_k - x^{\ast} \right\|^2 \notag \\
   & = \left\|\frac{\eta_k}{\eta_{k-1}}\left(\tilde{x}_k - x^{\ast}\right) + \left(1- \frac{\eta_k}{\eta_{k-1}}\right) (x_k - x^{\ast}) - \eta_k g_k\right\|^2 \notag \\
   & = \tau_k \left\|\tilde{x}_k - x^{\ast} \right\|^2 + \left(1- \tau_k\right) \left\|x_k - x^{\ast} \right\|^2 - 2\tau_k\eta_k \left\langle \tilde{x}_k - x^{\ast}, g_k \right\rangle - 2(1-\tau_k)\eta_k \left\langle {x}_k - x^{\ast}, g_k \right\rangle  \notag \\
   & \quad + \eta_k^2 \left\|g_k\right\|^2 % + 2\tau_k(1-\tau_k)\left\langle \tilde{x}_k-x^{\ast}, x_k - x^{\ast} \right\rangle.
  % & \mathop{\leq}^{(a)}  \tau_k\left\|\tilde{x}_k - x^{\ast} \right\|^2 + \left(1- \tau_k\right) \left\|x_k - x^{\ast} \right\|^2 - 2\tau_k\eta_k \left\langle \tilde{x}_k - x^{\ast}, g_k \right\rangle - 2(1-\tau_k)\eta_k \left\langle {x}_k - x^{\ast}, g_k \right\rangle + \eta_k^2 \left\|g_k\right\|^2
\end{align}
%where $\tau_k = \eta_k/\eta_{k-1}$ and the last inequality $(a)$ follows from Cauchy-Schwartz inequality that $\left\langle \tilde{x}_k-x^{\ast}, x_k - x^{\ast} \right\rangle \leq \frac{1}{2}\left(\left\|\tilde{x}_k - x^{\ast} \right\|^2 + \left\|x_{k}-x^{\ast} \right\|^2\right)$.
Taking conditional expectation on both sides, we then estimate the first inner product:
\begin{align}\label{inequ:xtilde:xg}
-\E[\left\langle \tilde{x}_k - x^{\ast}, g_k \right\rangle] & = -\left\langle \tilde{x}_k - x^{\ast}, \nabla f(x_k) \right\rangle = -\left\langle (1-\beta)^{-1}\left(x_k - \beta x_{k-1}\right) - x^{\ast}, \nabla f(x_k) \right\rangle \notag \\
& = -\frac{\left\langle x_k- x^{\ast}, \nabla f(x_k) \right\rangle}{1-\beta} + \frac{\beta}{1-\beta}\left\langle x_{k-1} - x_k + x_k-x^{\ast}, \nabla f(x_k) \right\rangle \notag \\
& \mathop{\leq}^{(a)} -\left(\theta_1\left\|x_k- x^{\ast}\right\|^2 - \theta_2 \right) + \frac{\beta}{1-\beta}\left(f(x_{k-1}) - f(x_k) + \frac{L}{2}\left\| x_k - x_{k-1} \right\|^2\right) % \notag \\
%& \mathop{\leq}^{(b)} -\left(\theta_1\left\|x_k- x^{\ast}\right\|^2 - \theta_2 \right) + \frac{\beta}{1-\beta}\left(\frac{\omega}{2}\left\|x_{k-1}-x_{k}\right\|^2 + \frac{1}{2\omega}L^2\left\| x_k - x^{\ast} \right\|^2\right)
\end{align}
where $(a)$ follows from the $R$-dissipativity condition and $L$-smoothness which implies that (\ref{L-smooth:inequ}). %the Cauchy-Schwartz inequality that $x^{T}y \leq \frac{\omega}{2}\left\|x \right\|^2 + \frac{1}{2\omega}\left\| y \right\|^2$ for any $\omega >0$ and $(b)$ applies the $L$-smooth assumption which implies that $\left\| \nabla f(x) \right\| \leq L\left\| x - x^{\ast} \right\|$. 

Next, we consider two situations: (1) if $\tau_k = 1$, that is $\eta_k = \eta_{k-1}$. Incorporating the inequality (\ref{inequ:xtilde:xg}) into (\ref{inequ:tildex}), we have
\begin{align}
& \E[ \left\| \tilde{x}_{k+1} - x^{\ast} \right\|^2 \mid \mathcal{F}_k] \notag \\
& \leq  \left\|\tilde{x}_k - x^{\ast} \right\|^2  -2\theta_1\eta_k \left\|x_k- x^{\ast}\right\|^2  + 2 \theta_2 \eta_k + \eta_k^2 \E[\left\|g_k\right\|^2 \mid \mathcal{F}_k] + \frac{2\beta\eta_k}{1-\beta}\left(f(x_{k-1}) - f(x_k) + \frac{L}{2}\left\| x_k - x_{k-1} \right\|^2\right) 
 \notag \\
%& = \left\|\tilde{x}_k - x^{\ast} \right\|^2 - 2\theta_1 \eta_k \left\| x_k - x^{\ast} \right\|^2   + \frac{2\beta\eta_k}{1-\beta}\left(f(x_{k-1}) - f(x_k) + \frac{L}{2}\left\| x_k - x_{k-1} \right\|^2\right)  + \eta_k^2 \E[\left\|g_k\right\|^2 \mid \mathcal{F}_k].
\end{align}

(2) If $\tau_k \in (0,1)$, we then estimate 
\begin{align}\label{inequ:xx:g}
- \E[\left\langle {x}_k - x^{\ast}, g_k \right\rangle \mid \mathcal{F}_k] = - \left\langle {x}_k - x^{\ast}, \nabla f(x_k) \right\rangle \leq -\theta_1\left\|x_k-x^{\ast} \right\|^2 + \theta_2. 
\end{align}
% Next we turn to estimate $\left\langle \tilde{x}_k-x^{\ast}, x_k - x^{\ast} \right\rangle$ as follows:
% \begin{align}\label{inequ:xxast}
% \left\langle \tilde{x}_k-x^{\ast}, x_k - x^{\ast} \right\rangle  & = \left\langle \tilde{x}_k-x^{\ast}, \tilde{x}_k - \frac{\beta}{1-\beta}\left(x_{k}-x_{k-1}\right) - x^{\ast} \right\rangle \notag \\
% & \leq \left\|\tilde{x}_k - x^{\ast}\right\|^2 + \frac{\beta}{1-\beta}\left(\frac{1}{2\omega_1}\left\|\tilde{x}_k - x^{\ast} \right\|^2 + \frac{\omega_1}{2}\left\| x_{k} - x_{k-1}\right\|^2\right)
% \end{align}
% for any $\omega_1 > 0$.
% Incorporating the above inequality into (\ref{inequ:tildex}) and in order to make the scalar in front of  $\left\|\tilde{x}_k - x^{\ast} \right\|^2$ is not larger than 1, we choose $\omega_1 = \frac{\beta \tau_k}{(1-\beta)(1-\tau_k)}$ such that 
% \begin{align}
%  \tau_k^2 + 2\tau_k(1-\tau_k)\left( 1+ \frac{\beta}{1-\beta}\frac{1}{2\omega_1}\right) \leq 1. 
% \end{align}

Finally, incorporating the above results (\ref{inequ:xtilde:xg}) and (\ref{inequ:xx:g}) into (\ref{inequ:tildex}), we can achieve that
\begin{align*}
& \E[ \left\| \tilde{x}_{k+1} - x^{\ast} \right\|^2 \mid \mathcal{F}_k] \notag \\
& \leq  \tau_k\left\|\tilde{x}_k - x^{\ast} \right\|^2 + \left(1- \tau_k\right)\left\|x_k - x^{\ast} \right\|^2  -2\tau_k\eta_k\left(\theta_1\left\|x_k- x^{\ast}\right\|^2 - \theta_2 \right) + \eta_k^2 \E[\left\|g_k\right\|^2 \mid \mathcal{F}_k] \notag\\
& \quad + \frac{2\beta\tau_k\eta_k}{1-\beta}\left(f(x_{k-1}) - f(x_k) + \frac{L}{2}\left\| x_k - x_{k-1} \right\|^2\right) 
+ 2(1-\tau_k)\eta_k \left(-\theta_1\left\|x_k-x^{\ast} \right\|^2 + \theta_2 \right) \notag \\
& = \tau_k\left\|\tilde{x}_k - x^{\ast} \right\|^2 - \left(2\eta_k\theta_1 -\left(1- \tau_k\right) \right)\left\| x_k - x^{\ast} \right\|^2   + 2\eta_k \theta_2 + \eta_k^2 \E[\left\|g_k\right\|^2 \mid \mathcal{F}_k]\notag \\
& \quad + \frac{2\beta  L}{1-\beta} \cdot \eta_k\tau_k \left\|x_{k-1}-x_{k}\right\|^2 + \frac{2\beta \tau_k \eta_k}{1-\beta}\left( f(x_{k-1}) - f(x_k) \right).
\end{align*}
We now complete the proof.
\end{proof}

\begin{proof}(of Theorem \ref{thm:mom:const})
In this case, we consider $\eta_k = \eta$ is a constant step-size. From Lemma \ref{lem:mom:1}, we have
\begin{align}\label{inequ:xx}
    & \E[\left\| x_{k+1} - x_k \right\|^2 \mid \mathcal{F}_k ] \notag \\
    & \leq \beta^2 \left\| x_k - x_{k-1} \right\|^2 + \eta^2(1-\beta)^2\E[\left\| g_k \right\|^2] + 2\eta\beta(1-\beta)\left(f(x_{k-1}) - f(x_k) + \frac{L}{2}\left\| x_k - x_{k-1}\right\|^2\right) \notag \\
    & = \left(\beta^2 + \eta\beta(1-\beta)L\right)\left\| x_k - x_{k-1} \right\|^2 + \eta^2(1-\beta)^2\E[\left\| g_k \right\|^2] + 2\eta\beta(1-\beta)\left(f(x_{k-1}) - f(x_k)\right).
\end{align}
Then we turn to estimate $\left\| \tilde{x}_{k+1} - x^{\ast} \right\|^2$. For the constant step-size  $\eta_k = \eta$, we have $\tau_k = \eta_k/\eta_{k-1} = 1$. By Lemma \ref{lem:mom:2}(i), we can achieve that
\begin{align*}%\label{inequ:tildex}
\E[\left\| \tilde{x}_{k+1} - x^{\ast} \right\|^2 \mid \mathcal{F}_k] & \leq  \left\|\tilde{x}_k - x^{\ast} \right\|^2 - 2\theta_1 \eta \left\| x_k - x^{\ast} \right\|^2  + 2\theta_2 \eta_k + \eta^2 \E[\left\|g_k\right\|^2 \mid \mathcal{F}_k] \notag \\
& \quad +  \frac{2\beta\eta}{1-\beta}\left(f(x_{k-1}) - f(x_k) + \frac{L}{2}\left\| x_k - x_{k-1} \right\|^2\right).
\end{align*}
Next we define a function $W_k$:
\begin{align*}
    W_{k+1} = \left\| \tilde{x}_{k+1} - x^{\ast} \right\|^2 +   \left\| x_{k+1} - x_k \right\|^2  + 2 \eta \beta\left( (1-\beta) + (1-\beta)^{-1} \right)\left(f(x_{k})  -f^{\ast}\right).
\end{align*}
 Then applying the results derived from Lemmas \ref{lem:mom:1} and \ref{lem:mom:2}, we have
\begin{align}\label{inequ:W}
& \E[W_{k+1}\mid \mathcal{F}_k] = \E[\left\| \tilde{x}_{k+1} - x^{\ast} \right\|^2 \mid \mathcal{F}_k] +   \E[\left\| x_{k+1} - x_k \right\|^2 \mid \mathcal{F}_k ] + 2 \eta\beta(1-\beta)\left(f(x_{k})  -f^{\ast}\right) \notag \\
& \leq \left\|\tilde{x}_k - x^{\ast} \right\|^2 + \left(\beta^2 + \eta\beta(1-\beta)L + \eta\frac{\beta L}{(1-\beta)}\right)\left\|x_{k-1}-x_{k}\right\|^2 - 2\theta_1\eta\left\|x_k - x^{\ast} \right\|^2 + 2\eta \theta_2 \notag \\
& \quad + \eta^2\left((1-\beta)^2 + 1 \right)\left( \sigma^2 + (\rho+1)L^2\left\| x_k - x^{\ast}\right\|^2 \right) + 2 \eta \beta\left( (1-\beta) + (1-\beta)^{-1} \right) \left(f(x_{k-1})  -f^{\ast}\right) .
\end{align}
where the above inequality uses the fact that \begin{align*}
   \E[\left\|g_k\right\|^2 \mid \mathcal{F}_k] & = \E[\left\| g_k -\nabla f(x_k) + \nabla f(x_k) \right\|^2 \mid \mathcal{F}_k] = \E[\left\| g_k -\nabla f(x_k) \right\|^2 \mid \mathcal{F}_k] + \left\|\nabla f(x_k) \right\|^2  \notag \\
   & \leq (\rho+1)\left\|\nabla f(x_k)\right\|^2 + \sigma^2 \leq (\rho+1)L^2\left\|x_k - x^{\ast}\right\|^2 + \sigma^2.
\end{align*}
%Let $\omega = \frac{\beta L^2}{(1-\beta)\theta_1}$, then $2\theta_1 - \frac{\beta L^2}{\omega(1-\beta)} =\theta_1$.
Suppose that
\begin{align}
    \eta \leq \frac{(1-\beta^2)}{\beta L\left(1-\beta + (1-\beta)^{-1} \right)},
\end{align}
we have
\begin{align*}
\beta^2 + \eta\beta(1-\beta)L + \eta\frac{\beta L}{(1-\beta)}  \leq 1.
\end{align*}
Then (\ref{inequ:W}) can be estimated as
\begin{align*}
\E[W_{k+1} \mid \mathcal{F}_k] & \leq  W_k - 2\eta\theta_1 \left\|x_k - x^{\ast} \right\|^2 + \eta^2\left((1-\beta)^2 + 1 \right)\left( \sigma^2 + (\rho+1)L^2\left\| x_k - x^{\ast}\right\|^2 \right) + 2\eta \theta_2 \notag \\
& = W_k - \eta\left(2\theta_1 - \eta ((1-\beta)^2 + 1 )(\rho+1)L^2\right)\left\|x_k - x^{\ast} \right\|^2  + 2 \eta \theta_2 + \eta^2\left((1-\beta)^2 + 1 \right)\sigma^2 \notag.
\end{align*}
Furthermore, assume that $\eta \leq \frac{\theta_1}{((1-\beta)^2+1)(\rho+1)L^2}$, then $2\theta_1 - \eta ((1-\beta)^2 + 1 )(\rho+1)L^2 \geq \theta_1$, we have
\begin{align*}
\E[W_{k+1} \mid \mathcal{F}_k] & \leq  W_k - \eta\theta_1 \left\|x_k - x^{\ast} \right\|^2  + 2 \eta \theta_2 + \eta^2\left((1-\beta)^2 + 1 \right)\sigma^2.
\end{align*}
If we further let 
\begin{align*}
r^2 = \max \left\lbrace R^2, \frac{2\theta_2}{\theta_1} + \frac{2\eta \left((1-\beta)^2 + 1 \right)\sigma^2}{\theta_1} \right\rbrace,
\end{align*}
once $\left\|x_{k} - x^{\ast} \right\|^2 \geq r^2$, we have
\begin{align*}
- \eta\theta_1\left\|x_{k} - x^{\ast} \right\|^2  + 2 \eta \theta_2 + \eta^2\left((1-\beta)^2 + 1 \right)\sigma^2 \leq 0.
\end{align*}
Then $\E[W_{k+1} \mid \mathcal{F}_k ] \leq W_{k}$. 

If $\left\|x_1 - x^{\ast} \right\|^2 \leq r^2$, 
let $x_{k^{'}}$ be the first iteration that makes $\left\|x_{k^{'}} - x^{\ast} \right\|^2 \geq r^2$ and $\left\| x_{k^{'}-2} - x^{\ast}\right\|^2, \left\| x_{k^{'}-1} - x^{\ast}\right\|^2 \leq r^2$. First, we show that $\left\|x_{k^{'}} - x^{\ast} \right\|^2$ will not be far larger than $r^2$.

If $\left\| x_{k^{'}-1} - x^{\ast}\right\|^2 \leq r^2$ and $\left\| x_{k^{'}-2} - x^{\ast}\right\|^2 \leq r^2$, then applying the recursion of SGD with momentum, we have
\begin{align}
\E[\left\| x_{k^{'}} - x^{\ast} \right\|^2] & = \E[\left\|x_{k^{'}-1} - \eta v_{k^{'}} - x^{\ast}\right\|^2] = \E\left[\left\|x_{k^{'}-1} - \eta \left(\beta \eta^{-1}\left(x_{k^{'}-2}-x_{k^{'}-1}\right) + (1-\beta) g_{k^{'}-1} \right) - x^{\ast}\right\|^2\right] \notag \\
& = \E[\left\|(1+\beta)\left(x_{k^{'}-1} -x^{\ast} \right) - \beta\left(x_{k^{'}-2} - x^{\ast}\right) - \eta(1-\beta)g_{k^{'}-1} \right\|^2 ]\notag \\
& \mathop{\leq}^{(a)} 3(1+\beta)^2\left\|x_{k^{'}-1} -x^{\ast}\right\|^2 + 3\beta^2\left\|x_{k^{'}-2} -x^{\ast}\right\|^2 + 3\eta^2(1-\beta)^2\E[\left\|g_{k^{'}-1}\right\|^2] \notag \\
& \mathop{\leq}^{(b)} 3(1+\beta)^2\left\|x_{k^{'}-1} -x^{\ast}\right\|^2 + 3\beta^2\left\|x_{k^{'}-2} -x^{\ast}\right\|^2 + 3\eta^2(1-\beta)^2(\sigma^2 + (\rho+1)L\left\| x_{k^{'}-1} - x^{\ast} \right\|^2)\notag \\
& \leq 3\left( (1+\beta)^2 + \beta^2\right)r^2 + 3 \eta^2(1-\beta)^2(\sigma^2 + (\rho+1)L^2r^2):= \Delta_{'}^2
\end{align}
where $(a)$ follows from the fact that $\left(\frac{x+y+z}{3}\right)^2 \leq 3(x^2 + y^2 +z^2)$ and $(b)$ applies inequality (\ref{inequ:grad}) that $\E[\left\| g_k \right\|^2 \mid \mathcal{F}_k] \leq \sigma^2 + (\rho+1)\left\|\nabla f(x_k)\right\|^2$.
In the case that $\left\|x_{k^{'}} - x^{\ast} \right\|^2 \geq r^2$ and $\left\| x_{k^{'}-1} - x^{\ast}\right\|^2 \leq r^2$), we can estimate the Lyapunov function $W_k$ at $k^{'}$-th iteration 
\begin{align*}
\E[W_{k^{'}}] & = \E[\left\| \tilde{x}_{k^{'}} - x^{\ast} \right\|^2 +   \left\| x_{k^{'}} - x_{k^{'}-1} \right\|^2  + 2 \eta \beta\left( (1-\beta) + (1-\beta)^{-1} \right)\left(f(x_{k^{'}-1})  -f^{\ast}\right)] \notag \\
& \leq \E\left[\left\| (1-\beta)^{-1}\left(x_{k^{'}} - \beta x_{k^{'}-1} \right) -x^{\ast}\right\|^2 + \left\| x_{k^{'}}-x^{\ast} - (x_{k^{'}-1}-x^{\ast}) \right\|^2\right] \notag \\
& \quad + \eta \beta\left( (1-\beta) + (1-\beta)^{-1} \right)L\E[\left\| x_{k^{'}-1} - x^{\ast} \right\|^2 \mid \mathcal{F}_k] \notag \\
& \leq \left(\frac{1}{(1-\beta)^2} + \frac{\beta^2}{(1-\beta)^2}\right) \left(\E[\left\|x_{k^{'}} - x^{\ast} \right\|^2 ] + \E[\left\|x_{k^{'}-1} - x^{\ast} \right\|^2 ]\right) + 2 \left(\E[\left\|x_{k^{'}} - x^{\ast} \right\|^2 ] + \E[\left\|x_{k^{'}-1} - x^{\ast} \right\|^2 ]\right) \notag \\
& \quad + \eta \beta\left( (1-\beta) + (1-\beta)^{-1} \right) L\E[\left\| x_{k^{'}-1} - x^{\ast} \right\|^2] \notag \\
& \leq \left(\frac{1+\beta^2}{(1-\beta)^2} + 2  \right)\left(\Delta_{'}^2 + r^2 \right) + \eta \beta\left( (1-\beta) + (1-\beta)^{-1} \right) L r^2.
\end{align*}
As we discussed before, once $\left\|x_{k} - x^{\ast} \right\|^2 \geq r^2$, then $\E[W_{k}]$ is decreasing. Thus we can conclude that $\E[W_k]$ is uniformly bounded by
\begin{align*}
 \E[W_k] \leq \left(\frac{1+\beta^2}{(1-\beta)^2} + 2  \right)\left(\Delta_{'}^2 + r^2 \right) + \eta \beta\left( (1-\beta) + (1-\beta)^{-1} \right) Lr^2.
\end{align*}

If the initial point $\left\|x_1 - x^{\ast} \right\|^2 >  r^2$ but is finite. Then we can see that $\E[W_k]$ is decreasing until $\left\|x_k - x^{\ast} \right\|^2 \leq r^2$. Thus in this case, we also can conclude that $\E[W_k]$ is uniformly bounded. Due to that the three quantities  $\E[\left\|\tilde{x}_k - x^{\ast} \right\|], \E[\left\| x_{k} - x_{k-1} \right\|^2],  \eta \E[f(x_{k-1}) - f^{\ast}] \geq 0$, so we can conclude that the three quantities are uniformly bounded. By the $L$-smoothness, we can estimate the function value $\E[f(x_k) - f^{\ast}]$ as
\begin{align*}
    \E[f(x_k) - f^{\ast}] & \leq \frac{L}{2}\E[\left\| x_k - x^{\ast}\right\|^2] = \frac{L}{2}\E\left[\left\| \tilde{x}_k - \frac{\beta}{1-\beta}(x_k-x_{k-1}) - x^{\ast}\right\|^2 \right] \notag \\
    & \leq L\E[\left\|\tilde{x}_k  - x^{\ast}\right\|^2] + \frac{L\beta^2}{(1-\beta)^2}\E[\left\|x_k-x_{k-1}\right\|^2] \leq L\left(1+ \frac{\beta^2}{(1-\beta)^2} \right)\E[W_k]
\end{align*}
is uniformly bounded. Now we complete the proof. 
\end{proof}

\begin{proof}(Proofs of Theorem \ref{thm:mom:decay})
In this case, we consider the step-size $\eta_k$ is strictly decaying which implies that $\tau_k = \eta_k/\eta_{k-1} < 1$. %$\eta_k = \eta_1/k$, we have $\tau_k = (k-1)/k$ and $1-\tau_k = 1/k$. 
By applying Lemma \ref{lem:mom:1}, we have
\begin{align}\label{inequ:decay:1}
 \E[\left\| x_{k+1} - x_k \right\|^2 \mid \mathcal{F}_k ] 
& \leq \tau_k^2\beta^2 \left\| x_k - x_{k-1} \right\|^2 + \eta_k^2(1-\beta)^2\E[\left\| g_k \right\|^2 \mid \mathcal{F}_k] \notag \\
& \quad + 2\beta(1-\beta)\tau_k\eta_k\left(f(x_{k-1}) - f(x_k) + \frac{L}{2}\left\| x_k - x_{k-1}\right\|^2\right).
\end{align}
By Lemma \ref{lem:mom:2}(ii), we can achieve that
\begin{align}\label{inequ:decay:2}
     \E[\left\| \tilde{x}_{k+1} - x^{\ast} \right\|^2 \mid \mathcal{F}_k] & \leq  \tau_k\left\|\tilde{x}_k - x^{\ast} \right\|^2 - \left(2\eta_k\theta_1 -\left(1- \tau_k\right) \right)\left\| x_k - x^{\ast} \right\|^2   + 2\eta_k \theta_2 + \eta_k^2 \E[\left\|g_k\right\|^2 \mid \mathcal{F}_k]\notag \\
& \quad + \frac{2\beta L }{1-\beta}\cdot\eta_k\tau_k \left\|x_{k-1}-x_{k}\right\|^2 + \frac{2\beta \tau_k \eta_k}{1-\beta}\left( f(x_{k-1}) - f(x_k) \right).
\end{align}
We now define a Lyapunov function $W_k$:
\begin{align*}
W_{k+1} = \left\| \tilde{x}_{k+1} - x^{\ast} \right\|^2 + \left\| x_{k+1} - x_k \right\|^2  + 2\gamma_{\beta} \eta_k \tau_k \left(f(x_{k})  -f^{\ast}\right).
\end{align*}
where  $\gamma_{\beta} = \beta(1-\beta + (1-\beta)^{-1})$.
Then incorporating the above inequalities (\ref{inequ:decay:1}) and (\ref{inequ:decay:2}), we have
\begin{align}\label{inequ:main:mom}
\E[W_{k+1} \mid \mathcal{F}_k] 
& \leq \left\|\tilde{x}_k - x^{\ast} \right\|^2 +  \left(\tau_k^2\beta^2+ \beta(1-\beta)L\tau_k\eta_k + \frac{\beta \tau_k }{1-\beta}\left( 2L \eta_k \right)\right)\left\| x_k - x_{k-1} \right\|^2  \notag \\ 
& \quad + 2\gamma_{\beta}\eta_{k-1}\tau_{k-1}\left(f(x_{k-1})  -f^{\ast}\right)
 +2 \gamma_{\beta} \left(\eta_k\tau_k - \eta_{k-1}\tau_{k-1} \right)(f(x_{k-1}) - f^{\ast}) \notag \\
& \quad   - \left(2\theta_1\eta_k -(1-\tau_k) - (\rho+1)\left((1-\beta)^2 +1\right)L^2\eta_k^2\right)\left\| x_k - x^{\ast} \right\|^2 \notag \\
& \quad  + 2\theta_2\eta_k + \left( (1-\beta)^2+ 1\right)\sigma^2\eta_k^2.
\end{align}
If the step-size $\eta_k$ is decreasing, then $\tau_k = \eta_k/\eta_{k-1} \leq 1$. We assume 
    \begin{align}\label{inequ:eta0}
        \eta_k \leq \frac{1-\tau_k^2\beta^2}{L\tau_k\gamma_{\beta}}
    \end{align}
  to ensure that
    \begin{align*}
    \tau_k^2\beta^2+ \beta(1-\beta)\tau_k\eta_k L + \frac{\beta \tau_k }{1-\beta} \cdot 2L\eta_k  \leq 1
    \end{align*}
The right side of (\ref{inequ:eta0}) is decreasing with $\tau_k \in (0,1]$, so we set 
\begin{align*}
 \eta_k  \leq \frac{1-\beta^2}{L\gamma_{\beta}}.
\end{align*}
Furthermore, we choose 
\begin{align*}
 \eta_k \leq  \frac{\theta_1}{2(\rho+1)((1-\beta)^2+1)L^2} ,
\end{align*}
such that 
\begin{align*}
(\rho+1)\left((1-\beta)^2 +1\right)L^2\eta_k^2 \leq \frac{\theta_1\eta_k}{2}.
\end{align*}
That is to say, if
\begin{align*}
\eta_k \leq  \min \left\lbrace \frac{\theta_1}{2(\rho+1)((1-\beta)^2+1)L^2}, \frac{1-\beta^2}{L\beta(1-\beta +(1-\beta)^{-1})} \right\rbrace
\end{align*}
then (\ref{inequ:main:mom}) can be re-written as 
\begin{align*}
\E[W_{k+1} \mid \mathcal{F}_k] 
& \leq W_k- \left(\frac{3\theta_1\eta_k}{2} -(1-\tau_k)\right) \left\| x_k - x^{\ast} \right\|^2 + 2\theta_2\eta_k + \left( (1-\beta)^2+ 1\right)\sigma^2\eta_k^2 \notag \\
& \quad   + 2\gamma_{\beta}\left(\eta_k\tau_k - \eta_{k-1}\tau_{k-1} \right)(f(x_{k-1}) - f^{\ast}).
\end{align*}
Based on different decaying modes, we can achieve the following results.
\begin{itemize}
 \item Polynomial decaying step-size: $\eta_k = \eta_1/k^r$ for $r\in (0,1]$. In this case $\tau_k = \eta_k/\eta_{k-1} = (k-1)^r/k^r$, then
        \begin{align*}
        \eta_k\tau_k - \eta_{k-1}\tau_{k-1} & = \eta_1\left(\frac{(k-1)^r}{k^{2r}} - \frac{(k-2)^r}{(k-1)^{2r}} \right) 
         = \eta_1\frac{(k-1)^{3r} - (k-2)^rk^{2r}}{k^{2r}(k-1)^{2r}}   \notag \\
         & = \eta_1\frac{(k-1)^{3r} - \left((k-1)^{3}-(k-1) + (k-1)^2 -1\right)^r}{k^{2r}(k-1)^{2r}}  < 0
    \end{align*}
    for all $ k \geq 3$. For any $r \in (0,1]$ and $ k \geq 1$, we have
    \begin{align*}
    k^r \leq (k-1)^r + 1
    \end{align*}
    then \begin{align*}
    1 - \tau_k = 1-\frac{(k-1)^r}{k^r}  \leq 1-\frac{k^r-1}{k^r} = \frac{1}{k^{r}}.
    \end{align*}
%     Then 
%     \begin{align}
%     \E[W_{k+1} \mid \mathcal{F}_k] & \leq  W_k -\left(\gamma_1\left(\theta_1 \frac{\eta_1}{k^r} - \frac{1}{k^{2r}}\right) - (\rho+1)\left((1-\beta)^2 +\gamma_1\right)L^2\frac{\eta_1^2}{k^{2r}}\right)\left\| x_k - x^{\ast} \right\|^2 \notag \\
%     & \quad + 
%   2\gamma_1\theta_2\frac{\eta_1}{k^r} + \frac{\eta_1^2}{k^{2r}}\left( (1-\beta)^2+\gamma_1\right)\sigma^2 
%     \end{align}
    For $\eta_1 \geq \frac{2}{\theta_1}$, we have $1-\tau_k \leq \frac{\theta_1}{2}\eta_k$, then%, \right\rbrace \left(\frac{2\left(\gamma_1+ (\rho+1)\left( (1-\beta)^2+\gamma_1\right)L^2\eta_1^2\right)}{\gamma_1\theta_1 \eta_1} \right)^{1/r}
% If $\eta_k \leq  \frac{\gamma_1\theta_1}{4(\rho+1)((1-\beta)^2 + \gamma_1)L^2}$
%     such that 
%     \begin{align}
%   (\rho+1)\left((1-\beta)^2 +\gamma_1\right)L^2\eta_k^2\leq \frac{\gamma_1\theta_1\eta_k}{4},
%     \end{align}
%     then
      \begin{align*}
    \E[W_{k+1} \mid \mathcal{F}_k] & \leq  W_k -  \theta_1\eta_k\left\| x_k - x^{\ast} \right\|^2  + 
   2\theta_2\eta_k + \left( (1-\beta)^2+1\right)\sigma^2 \eta_k^2.
    \end{align*}
%   In this case, if 
%     \begin{align}
%      R^2 \geq \frac{4\theta_2}{\theta_1} + \frac{2\eta_k\left(\gamma_1+(1-\beta)^2 \right)\sigma^2}{\gamma_1\theta_1},
%     \end{align}
%     we can see that once $\E[\left\| x_k - x^{\ast} \right\|^2] \geq R^2$, then $\E[W_{k+1}]$ is decreasing.  
    \item Linear decay: $\eta_k = A - B k$ where $\eta_1 = \eta_{\max}$ and $\eta_{T} = \eta_{\min} = c/\sqrt{T}$, then we have
    \begin{align*}
    A = \frac{T\eta_{\max} - \eta_{\min}}{T-1},  & \quad B =  \frac{\eta_{\max} - \eta_{\min}}{T-1}. 
    \end{align*}
%  \begin{align}
%     \eta_k\tau_k - \eta_{k-1}\tau_{k-1}= \eta_{k-1}\left( \left(\frac{\eta_k}{\eta_{k-1}}\right)^2 - \frac{\eta_{k-1}}{\eta_{k-2}}\right)
%     \end{align}
%     The sign of $\eta_k\tau_k - \eta_{k-1}\tau_{k-1}$ is depended by $ \left(\frac{\eta_k}{\eta_{k-1}}\right)^2 - \frac{\eta_{k-1}}{\eta_{k-2}}$. 
Next we turn to estimate the sign of $\eta_k\tau_k - \eta_{k-1}\tau_{k-1}$.
    \begin{align*}
    \eta_k\tau_k - \eta_{k-1}\tau_{k-1}  = \frac{\eta_k^2\eta_{k-2} - \eta_{k-1}^3}{\eta_{k-1}\eta_{k-2}} 
    & = \frac{(A - Bk)^2(A-B(k-2)) - (A-B(k-1))^3}{(A - B(k-1))(A - B(k-2))} \notag \\
    & =  \frac{(A - B(k-1+1))^2(A - B(k-1-1)) - (A - B(k-1))^3}{(A - B(k-1))(A - B(k-2))} \notag \\
    & = \frac{- B\left[(A-B(k-1))^2 + B(A - B(k-1))) - B^2 \right]}{(A - B(k-1))(A - B(k-2))} \notag \\
    & = \frac{- B\left[(A - Bk + B)^2 + B(A - Bk) \right]}{(A - B(k-1))(A - B(k-2))}.
    \end{align*}
    We know that $B > 0$ and $A  - Bk >0$, then $(A -Bk + B)^2 + B(A - Bk)  >0$, thus we have $\eta_k\tau_k - \eta_{k-1}\tau_{k-1} <0$.  Next we estimate $1-\tau_k$:
    \begin{align*}
     1- \tau_k = 1- \frac{\eta_k}{\eta_{k-1}}  = 1 - \frac{A - Bk}{A - B(k-1)} = \frac{B}{A - B(k-1)}.
    \end{align*}
   We know that $\eta_{\min} = c/\sqrt{T}$, let $c  \geq \left(2\frac{\eta_{\max}}{\theta_1}\right)^{1/2}$, we have%$ T  \geq \left(\frac{4\eta_{\max}^2}{\theta_1\eta_1^3}\right)^2$, we know that $\eta_{\min}=\eta_1/\sqrt{T}$, then
    \begin{align*}
        1-\tau_k \leq \frac{B}{\eta_{k-1}} \leq \frac{(\eta_{\max} - \eta_{\min})}{(T-1) \eta_{\min}}\leq \frac{\theta_1}{2}\eta_{\min} \leq  \frac{\theta_1}{2}\eta_k.
    \end{align*}
Finally, we can achieve that
  \begin{align*}
   \E[W_{k+1} \mid \mathcal{F}_k] & \leq  W_k - \theta_1 \eta_k\left\| x_k - x^{\ast} \right\|^2  + 
   2\theta_2 \eta_k  + \eta_k^2\left( (1-\beta)^2+1\right)\sigma^2.
   \end{align*}

    \item Cosine decay step-size:
 $\eta_k = A + B\cos(k\pi/T) $ where $A=\frac{\eta_{\min} + \eta_{\max}}{2}$ and $B= \frac{\eta_{\max} - \eta_{\min}}{2}$. We first estimate $\tau_k\eta_k - \eta_{k-1}\eta_{k-1}$:
 \begin{align*}
     \eta_k\tau_k - \eta_{k-1}\tau_{k-1} & = \frac{\eta_k^2}{\eta_{k-1}} - \frac{\eta_{k-1}^2}{\eta_{k-2}}   =  \eta_{k-1}\left( \left(\frac{\eta_k}{\eta_{k-1}}\right)^2 - \frac{\eta_{k-1}}{\eta_{k-2}}\right)
    %  & = \frac{(A + B\cos(i\pi/T))^2(A+B\cos((i-2)\pi/T)) - (A+B\cos((i-1)\pi/T))^3}{(A + B\cos((i-1)\pi/T))(A + B\cos((i-2)\pi/T))} \notag \\
    %  & =  \frac{(A + B\cos(i\pi/T))^2(A+B\cos((i-2)\pi/T)) - (A+B\cos((i-1)\pi/T))^3}{(A + B\cos((i-1)\pi/T))(A + B\cos((i-2)\pi/T))}
    \end{align*}
    In order to estimate the sign of $\eta_k\tau_k - \eta_{k-1}\tau_{k-1}$, we try to estimate $\left(\frac{\eta_k}{\eta_{k-1}}\right)^2 - \frac{\eta_{k-1}}{\eta_{k-2}}$. Then
    \begin{align}\label{inequ:psi:k}
   \psi_k := \frac{\left(1-\frac{\eta_{k-1}}{\eta_{k-2}}\right)}{1- \left(\frac{\eta_k}{\eta_{k-1}}\right)^2} =  \frac{ \frac{\eta_{k-1}}{\eta_{k-2}}}{\frac{\eta_k}{\eta_{k-1}}+1}  \cdot \frac{\eta_{k-2} - \eta_{k-1}}{\eta_{k-1} - \eta_k}. 
    \end{align}
 For $k \in [1, T/2]$, by the graph of the step-size, we know that $\eta_{k}/\eta_{k-1} \leq \eta_{k-1}/\eta_{k-2} \leq 1$ and $\eta_{k-2}-\eta_{k-1} \leq \eta_{k-1} - \eta_k$, we have  $\psi_k < 1$. If $k \in [T/2, T)$, we can see that $\eta_k/\eta_{k-1} \geq \eta_{k-1}/\eta_{k-2}$, then
 \begin{align*}
     \psi_k \leq \frac{1}{2} \frac{\eta_{k-2} - \eta_{k-1}}{\eta_{k-1} - \eta_k} &  = \frac{\sin(\frac{(2k-3)\pi}{2T})}{2\sin(\frac{(2k-1)\pi}{2T})} = \frac{1}{2}\left( \cos(\frac{\pi}{T}) - \frac{\cos(\frac{(2k-1)\pi}{2T})\sin(\frac{\pi}{T})}{\sin(\frac{(2k-1)\pi}{2T})}\right) \notag \\
     & \leq \frac{1}{2}\left( \cos(\frac{\pi}{T}) + \frac{\cos(\frac{3\pi}{2T})\sin(\frac{\pi}{T})}{\sin(\frac{3\pi}{2T})}\right) = \frac{1}{2}\cos(\frac{\pi}{T})\left(  1 + \frac{\tan(\frac{\pi}{2T})}{\tan(\frac{3\pi}{2T})}\right) < 1.
 \end{align*}
 That is for $k \in [1, T)$, we have $\psi_k < 1$. Thus $\left(\frac{\eta_k}{\eta_{k-1}}\right)^2 - \frac{\eta_{k-1}}{\eta_{k-2}} < 0$, then we have
 $\eta_k\tau_k - \eta_{k-1}\tau_{k-1} < 0$.  Next we turn to estimate $1-\tau_k$:
     \begin{align*}
    1- \tau_k  & =1- \frac{\eta_k}{\eta_{k-1}}  =1- \frac{A + B\cos(k\pi/T)}{A + B \cos((k-1)\pi/T)} = \frac{B\left(\cos((k-1)\pi/T) - \cos(k\pi/T)\right) }{A + B\cos((k-1)\pi/T)} \notag \\
    & =  \frac{2B\sin(\pi/(2T))\sin((2k-1)\pi/(2T))}{\left(A + B\cos((k-1)\pi/T)\right)} \leq \frac{2B \left(\frac{\pi}{2T}\right)}{\left(A + B\cos(k\pi/T)\right)} = \frac{2B \left(\frac{\pi}{2T}\right)}{\eta_k}
    \end{align*}
  Let $\eta_{\min} = c/\sqrt{T}$. To make sure that $1- \tau_k\leq \frac{\theta_1}{2}\eta_k$, we let $c \geq \left(\frac{\eta_{\max}^2\pi^2}{2\theta_1}\right)^{1/2} $. Then for any $k \in [1, T)$, we have
  \begin{align*}
         \E[W_{k+1} \mid \mathcal{F}_k] & \leq  W_k - \theta_1 \eta_k\left\| x_k - x^{\ast} \right\|^2  + 
   2\theta_2\eta_k  + \left( (1-\beta)^2+1\right)\sigma^2\eta_k^2.
  \end{align*}
% \textcolor{red}{Note that when $k = T$, $\eta_T\tau_T- \eta_{T-1}\tau_{T-1} =  > 0$. }
    \item Exponential decaying step-size $\eta_k = \eta_1/\alpha^{k-1}$ where $\alpha = (\nu/T)^{-1/T} > 1$ and $\nu \geq 1$. In this case, we have $\tau_k = \eta_k/\eta_{k-1} = 1/\alpha$ and 
    \begin{align*}
       1 - \tau_k = 1-1/\alpha \leq \frac{\ln\left(T/\nu\right)}{T}
    \end{align*}
    where $1-x \leq \ln(\frac{1}{x})$ for any $x > 0$. 
    \begin{align*}
    \eta_k\tau_k - \eta_{k-1}\tau_{k-1} = \eta_1\left(\frac{1}{\alpha^{k-1}}\frac{1}{\alpha} - \frac{1}{\alpha^{k-2}}\frac{1}{\alpha} \right)= \frac{\eta_1}{\alpha^{k}}(1-\alpha) < 0
    \end{align*}
%     Then we have
%     \begin{align}
%     \E[W_{k+1} \mid \mathcal{F}_k] & \leq  W_k -\left(\gamma_1\left(\frac{\theta_1 \eta_1}{\alpha^{(k-1)}} - \frac{\ln^2\left(T/\nu \right)}{T^2} \right) - (\rho+1)\left((1-\beta)^2 +\gamma_1\right)\frac{L^2\eta_1^2}{\alpha^{2(k-1)}} \right)\left\| x_k - x^{\ast} \right\|^2 \notag \\
%     & \quad + 
%   \frac{2\gamma_1\theta_2\eta_1}{\alpha^{(k-1)}}  + \frac{\eta_1^2}{\alpha^{2(k-1)}}\left( (1-\beta)^2+\gamma_1\right)\sigma^2.
%     \end{align}  
% For any 
% \begin{align}
%  k & \geq  \max\left\lbrace \log_{\alpha}\left( \frac{4(\rho+1)((1-\beta)^2+\gamma_1)L^2\eta_1}{\gamma_1\theta_1}\right) + 1 \right\rbrace
% \end{align}
% we have
% \begin{align}
%     (\rho+1)\left((1-\beta)^2 +\gamma_1\right)\frac{L^2\eta_1^2}{\alpha^{2(k-1)}}  \leq \frac{\gamma_1\theta_1\eta_1}{4\alpha^{(k-1)}}
% \end{align}
Let $ \eta_1 \geq \frac{2\ln (T/\nu)}{\theta_1 \nu}$, then
\begin{align*}
1-\tau_k \leq \frac{\ln(T/\nu)}{T} \leq \frac{\theta_1}{2}\eta_k
\end{align*}
  for any $ 1 \leq k \leq T$ and $\nu \in [1, T]$, we have
  \begin{align*}
   \E[W_{k+1} \mid \mathcal{F}_k] & \leq  W_k - \theta_1 \eta_k\left\| x_k - x^{\ast} \right\|^2  + 
   2\theta_2\eta_k  + \left( (1-\beta)^2+1\right)\sigma^2\eta_k^2.
   \end{align*}
 \end{itemize}
 For the above four cases, we let 
  \begin{align*}
       r^2 = \max\left\lbrace R^2, \frac{2\theta_2}{\theta_1} +  \frac{((1-\beta^2)+1)\sigma^2\eta_k}{\theta_1} \right\rbrace,
  \end{align*}
 the step-size satisfies the following condition:
  \begin{align*}
\eta_k \leq  \min \left\lbrace \frac{\theta_1}{2(\rho+1)((1-\beta)^2+1)L^2}, \frac{1-\beta^2}{L\beta(1-\beta +(1-\beta)^{-1})} \right\rbrace.
\end{align*}
  Once $\E[\left\|x_k - x^{\ast}\right\|^2] \geq r^2$, we have $\E[W_{k+1}]$ is decreasing. Following the same discussion as the constant step-size in Theorem \ref{thm:mom:const}, for the decaying modes mentioned above, we can conclude that the quantities $\E[W_k]$, $\E[\left\| x_k - x^{\ast} \right\|^2]$, $\E[f(x_k) - f^{\ast}]$ is uniformly bounded. Note that for the cosine decaying mode, the above results only hold for $k \in [1, T)$. At the final iterate $k= T$, we next show that $\E[W_{T+1}]$ is also bounded:
  \begin{align*}
      \E[W_{T+1}] & \leq \E[W_T] - \left(\frac{3\theta_1\eta_T}{2} -(1-\tau_T)\right) \E[\left\| x_T - x^{\ast} \right\|^2] + 2\theta_2\eta_T + \left( (1-\beta)^2+ 1\right)\sigma^2\eta_T^2 \notag \\
& \quad   + 2\gamma_{\beta}\left(\eta_T\tau_T - \eta_{T-1}\tau_{T-1} \right)\E[f(x_{T-1}) - f^{\ast}].
  \end{align*}
  where $\tau_T = \eta_{T}/\eta_{T-1} \leq 1 $, $\eta_T\tau_T - \eta_{T-1}\tau_{T-1} > 0$, and $\eta_T=\eta_{\min}$. Because $\E[W_{k+1}]$ is uniformly bounded for all $k \in [1, T)$, by the previous discussion, we know that $\E[W_T]$, $\E[\left\| x_T-x^{\ast} \right\|^2]$, and $\E[f(x_{T-1})-f^{\ast}]$ are uniformly bounded. Firstly, we know that
  \begin{align*}
      1- \tau_{T} & =  1- \frac{\eta_{T}}{\eta_{T-1}} = \frac{B(\cos((T-1)\pi/T)+1)}{A + B\cos((T-1)\pi/T)} = \frac{B(-\cos(\pi/T)+1)}{A - B\cos(\pi/T)} \notag \\
     & \mathop{\leq}^{(a)} \frac{B \left(\frac{\pi}{T} \right)^2/2}{A- B} = \frac{\eta_{\max} - \eta_{\min}}{2\eta_{\min}} \cdot \frac{\pi^2}{2T^2} \mathop{\leq}^{(b)} \frac{\eta_{\max}}{4c} \frac{\pi^2}{T^{3/2}}
  \end{align*}
where $(a)$ follows from the fact that by the Taylor series of $\cos(x) = 1-\frac{x^2}{2} + \frac{x^4}{4!}+\cdots$, we have $1-\cos(\pi/T) \leq \frac{\pi^2}{2T^2}$ and $(b)$ uses the fact that $\eta_{\min} = c/\sqrt{T}$.
 Recalling that $c \geq \left(\frac{\eta_{\max}^2\pi^2}{2\theta_1} \right)^{1/2}$, we know that this bound $ \frac{\eta_{\max}}{4c} \frac{\pi^2}{T^{3/2}} \ll  \frac{3\theta_1 \eta_T}{2}$. This means that the scalar term $\frac{3\theta_1 \eta_T}{2} - (1-\tau_T)$ of $\E[\left\| x_T - x^{\ast} \right\|^2] $ is positive.  Next we turn to estimate $\eta_T\tau_T - \eta_{T-1}\tau_{T-1} $:
 \begin{align*}
    \eta_T\tau_T - \eta_{T-1}\tau_{T-1} = \tau_{T-1}\eta_{T-1}\left(\frac{\tau_T\eta_T}{\tau_{T-1}\eta_{T-1}}-1\right) = \tau_{T-1}\eta_{T-1}\left(\frac{\eta_T^2}{\eta_{T-1}^2} \cdot \left(\frac{\eta_{T-1}}{\eta_{T-2}}\right)^{-1} -1\right)
 \end{align*}
 Recall the definition of $\psi_k$ in (\ref{inequ:psi:k}) at $k=T$, we have
 \begin{align*}
    \psi_{T} & = \frac{1- \frac{\eta_{T-1}}{\eta_{T-2}}}{1-\left(\frac{\eta_{T}}{\eta_{T-1}}\right)^2} = \frac{ \frac{\eta_{T-1}}{\eta_{T-2}}}{\frac{\eta_T}{\eta_{T-1}}+1}  \cdot \frac{\eta_{T-2} - \eta_{T-1}}{\eta_{T-1} - \eta_T}  \leq \frac{1}{2}\frac{\eta_{T-2} - \eta_{T-1}}{\eta_{T-1} - \eta_T}  \notag \\
    & \leq \frac{\sin(\frac{(2T-3)\pi}{2T})}{2\sin(\frac{(2T-1)\pi}{2T})} = \frac{\sin(\frac{3\pi}{2T})}{2\sin(\frac{\pi}{2T})} = \frac{1}{2}\left( \cos(\frac{\pi}{T})+\frac{\cos(\frac{\pi}{2T})}{\sin(\frac{\pi}{2T})}\sin(\frac{\pi}{T})\right) \notag \\
    & = \frac{1}{2}\left( \cos(\frac{\pi}{T})+\frac{2\cos(\frac{\pi}{2T})}{\sin(\frac{\pi}{2T})}\sin(\frac{\pi}{2T})\cos(\frac{\pi}{2T})\right) = \frac{1}{2}\left( 2\cos(\frac{\pi}{T}) +1 \right) < \frac{3}{2}. 
 \end{align*}
Then
 \begin{align*}
    \left(\frac{\eta_T}{\eta_{T-1}} \right)^2\cdot \left(\frac{\eta_{T-1}}{\eta_{T-2}}\right)^{-1} -1 & \leq \frac{\left(\frac{\eta_T}{\eta_{T-1}} \right)^2}{1-\frac{3}{2}(1-\left(\frac{\eta_T}{\eta_{T-1}} \right)^2)} -1 = \frac{(1-\left(\frac{\eta_T}{\eta_{T-1}} \right)^2}{3\left(\frac{\eta_T}{\eta_{T-1}} \right)^2 - 1}  \notag \\
    & \leq \frac{\frac{B\pi^2}{2(A-B)T^2}}{2- \frac{B\pi^2}{2(A-B)T^2}} \leq 1
 \end{align*}
 where $\eta_{T-1}/\eta_{T} = 1 + \frac{B}{A-B}(1-\cos(\pi/T)) \leq 1+ \frac{B\pi^2}{2(A-B)T^2} \leq 2$ for sufficient large $c \geq \frac{\eta_{\max}\pi^2}{4T^{3/2}}$. Then we can see that
 $ \eta_T\tau_T - \eta_{T-1}\tau_{T-1} \leq \eta_{T-1}\tau_{T-1}$. Recall the definition of $\E[W_T]$, we know that $2\gamma_{\beta}\eta_{T-1}\tau_{T-1}\E[f(x_{T-1})-f^{\ast}] \leq \E[W_T]$. By the above analysis, we can get that
 \begin{align*}
    \E[W_{T+1}] & \leq \E[W_T] + 2\gamma_1\theta_2\eta_T + \left( (1-\beta)^2+ 1\right)\sigma^2\eta_T^2 + 2\gamma_{\beta}\left(\eta_T\tau_T - \eta_{T-1}\tau_{T-1} \right)\E[f(x_{T-1}) - f^{\ast}] \notag \\
    & \leq \E[W_T] + 2\theta_2\eta_{\min} + \left( (1-\beta)^2+ 1\right)\sigma^2\eta_{\min}^2 + 2\gamma_{\beta} \eta_{T-1}\tau_{T-1}\E[f(x_{T-1}) - f^{\ast}] \notag \\
    & \leq \E[W_T] + 2\theta_2\eta_{\min} + \left( (1-\beta)^2+ 1\right)\sigma^2\eta_{\min}^2 + \E[W_T] \notag \\
    & = 2\E[W_T] + 2\theta_2\eta_{\min} + \left( (1-\beta)^2+ 1\right)\sigma^2\eta_{\min}^2
 \end{align*}
For the cosine decaying step-size, $\E[W_{T+1}]$ is also bounded.  That is to say, for all $k \in [1, T+1]$, the quantities $\E[W_{k}]$ generated from the cosine decaying mode is uniformly bounded.
\end{proof}

\begin{proof}(Bandwidth-based Step-Size)

At each stage, we assume that the step-size is decreasing ($\eta_k \leq \eta_{k-1} \leq \eta_{k-2}$) for all $k \in ((t-1)S, tS]$. Then considering the step-size modes discussed in Theorems \ref{thm:mom:const} and \ref{thm:mom:decay}, we can see that:
\begin{itemize}
      \item Step-decay step-size: $\eta_k = \eta_1/\alpha^{t-1}$ for $ k \in (S(t-1), St]$ where $S = T/N$ and $t \in [N]$. At each stage $t$, we know that the step-size is a constant for all $k \in ((t-1)S, tS]$. By applying the results for constant step-sizes in Theorem \ref{thm:mom:const}, we can see that at each stage $t \in [N]$, the quantity $W_k$ is uniformly bounded for $k \in ((t-1)S+2, tS]$, that is
    \begin{align*}
     \E[W_{k+1}] \leq \left(\frac{1+\beta^2}{(1-\beta)^2} + 2  \right)\left(\Delta_{t}^2 + r^2 \right) + \eta_k L \beta(1-\beta + (1-\beta)^{-1}) r^2
    \end{align*}
    where $\Delta_{t}^2 = 3\left( (1+\beta)^2 + \beta^2\right)r^2 + 3 \eta_k^2(1-\beta)^2(\sigma^2 + (\rho+1)L^2r^2)$. Therefore, as we discussed in Theorem \ref{thm:mom:const}, at each stage, the quantities
    \begin{align*}
    \E[\left\|\tilde{x}_{k+1} - x^{\ast} \right\|] \leq \E[W_{k+1}], & \quad \E[\left\| x_{k+1} - x_{k} \right\|^2] \leq \E[W_{k+1}]; \notag \\
    \E[\left\| x_{k+1} - x^{\ast} \right\|^2] \leq 2\left(1+ \frac{\beta^2}{(1-\beta)^2}\right)\E[W_{k+1}], & \quad \E[f(x_{k+1}) - f^{\ast}] \leq L\left(1+ \frac{\beta^2}{(1-\beta)^2}\right)\E[W_{k+1}]
    \end{align*}
    are also uniformly bounded for all $k \in ((t-1)S+2, tS]$.
 By the definition of $\tau_k$, at $k = (t-1)S+1$, we have $\tau_{(t-1)S+1} = \eta_{(t-1)S+1}/\eta_{(t-1)S} = 1/\alpha$, then we define
    \begin{align*}
    W_{(t-1)S+2} = \left\| \tilde{x}_{(t-1)S+2} - x^{\ast} \right\|^2 + \left\| x_{(t-1)S+2} - x_{(t-1)S+1} \right\|^2  + 2 \frac{\eta_{(t-1)S+1}}{\alpha}\gamma_{\beta}\left(f(x_{(t-1)S+1})  -f^{\ast}\right).
    \end{align*}
%  As we proved that  $f(x_{(t-1)S}) - f^{\ast}$ is uniformly bounded and $\tau_{(t-1)S} = \eta_{(t-1)S+1}/\eta_{(t-1)S} = 1/\alpha$,  
Similar to Theorem \ref{thm:mom:decay} (see (\ref{inequ:main:mom})), we can estimate $\E[W_{(t-1)S+2} ]$ as
\begin{align}
 & \E[W_{(t-1)S+2}]  \notag \\
& :=  \E\left[\left\| \tilde{x}_{(t-1)S+2} - x^{\ast} \right\|^2 + \left\| x_{(t-1)S+2} - x_{(t-1)S+1} \right\|^2  +  \frac{2\eta_{(t-1)S+1}}{\alpha}\gamma_{\beta}\left(f(x_{(t-1)S+1})  -f^{\ast}\right)\right] \notag \\
& = W_{(t-1)S+1}- \left(\frac{3\theta_1\eta_{(t-1)S+1}}{2} -\left(1-\frac{1}{\alpha}\right)\right) \left\| x_{(t-1)S+1} - x^{\ast} \right\|^2 +  \left((1-\beta)^2+ 1\right)\sigma^2\eta_{(t-1)S+1}^2 \notag \\
& \quad + 2\theta_2\eta_{(t-1)S+1}  + 2\gamma_{\beta}\left(\eta_{(t-1)S+1}\tau_{(t-1)S+1} - \eta_{{(t-1)S}}\tau_{{(t-1)S}} \right)(f(x_{{(t-1)S}}) - f^{\ast}) \notag \\
& \leq W_{(t-1)S+1} + 2(1-\frac{1}{\alpha})^2 \left(1+ \frac{\beta^2}{(1-\beta)^2}\right)\E[W_{(t-1)S+1}] + \frac{2\theta_2\eta_1}{\alpha^{t-1}} + \frac{\left((1-\beta)^2+ 1\right)\sigma^2\eta_1^2}{\alpha^{2(t-1)}} \notag 
%\\
%& \quad  + 2\beta(1-\beta) \frac{\eta_1}{\alpha^{t-2}}\left( \frac{1}{\alpha^2}-1 \right)(f(x_{{(t-1)S}}) - f^{\ast}) 
% & \leq W_{(t-1)S+1}  + \gamma_1\left(1-\frac{1}{\alpha}\right)^2\left\| x_{(t-1)S+1} - x^{\ast} \right\|^2 +  \frac{2\gamma_1\theta_2\eta_1}{\alpha^{t-1}} \notag \\
% & \quad + \frac{\left((1-\beta)^2+ \gamma_1\right)\sigma^2\eta_1^2}{\alpha^{2(t-1)}}  + 2\beta(1-\beta) \frac{\eta_1}{\alpha^{t-2}}\left( \frac{1}{\alpha^2}-1 \right)L\left(1+ \frac{\beta^2}{(1-\beta)^2}\right)\E[W_{(t-1)S}]
\end{align}
where $\E[\left\| x_{(t-1)S+1} - x^{\ast} \right\|^2] \leq 2\left(1+ \frac{\beta^2}{(1-\beta)^2}\right)\E[W_{(t-1)S+1}]$ which are uniformly bounded, then we can conclude that $\E[W_{(t-1)S+2}]$ can be bounded by $\E[W_{(t-1)S+1}]$ plus a constant term. 
Thus in this case, we can derive that $\E[W_{k}]$ is uniformly bounded for all $k \in [1, T+1]$.
\item We then consider the three decaying modes: polynomial decay, linear decay, and cosine decay modes discussed in Theorem \ref{thm:mom:decay} at each stages. First, we consider polynomial decaying and linear decay. From the results of Theorem \ref{thm:mom:decay}, we can see that $\E[W_{k+1}]$ is uniformly bounded for $k \in ((t-1)S+2, tS]$ and its bound can be  derived similar to Theorem \ref{thm:mom:const}, that is
\begin{align*}
     \E[W_{k+1}] \leq \left(\frac{1+\beta^2}{(1-\beta)^2} + 2  \right)\left(\Delta_{t}^2 + r^2 \right) + \eta_k \tau_k L\gamma_{\beta}r^2 
\end{align*}
 where $\Delta_{t}^2 = 3\left( (1+\tau_k\beta)^2 + \beta^2\tau_k^2\right)r^2 + 3 \eta_k^2(1-\beta)^2(\sigma^2 + (\rho+1)L^2r^2)$ and $0 < \tau_k \leq 1$.  Because $\eta_{\min}^t \leq \eta_k \leq \eta_{\max}^t$ for all $k$, the upper bound is non-expanding for each stage $t$. However, at $k = (t-1)S+1$, we have $\tau_{(t-1)S+1} = \eta_{(t-1)S+1}/\eta_{(t-1)S} \leq \eta_{\max}^t/\eta_{\min}^{t} \leq s_0$. Similar to the above discussion for step-decay, we get that
 \begin{align*}
 & \E[W_{(t-1)S+2}]  \notag \\
& :=  \E\left[\left\| \tilde{x}_{(t-1)S+2} - x^{\ast} \right\|^2 + \left\| x_{(t-1)S+2} - x_{(t-1)S+1} \right\|^2  +  \frac{2\eta_{(t-1)S+1}}{\alpha}\gamma_{\beta}\left(f(x_{(t-1)S+1})  -f^{\ast}\right)\right] \notag \\
& = W_{(t-1)S+1}- \left(\frac{3\theta_1\eta_{(t-1)S+1}}{2} -(1-s_0)\right) \left\| x_{(t-1)S+1} - x^{\ast} \right\|^2  + \left((1-\beta)^2+ 1\right)\sigma^2\eta_{(t-1)S+1}^2 \notag \\
& \quad  + 2\theta_2\eta_{(t-1)S+1}  + 2\gamma_{\beta}\left(\eta_{(t-1)S+1}\tau_{(t-1)S+1} - \eta_{{(t-1)S}}\tau_{{(t-1)S}} \right)(f(x_{{(t-1)S}}) - f^{\ast}) \notag \\
& \leq W_{(t-1)S+1} + 2(1-s_0) \left(1+ \frac{\beta^2}{(1-\beta)^2}\right)\E[W_{(t-1)S+1}] + 2\theta_2\eta_{\max}^t + \left((1-\beta)^2+ 1\right)\sigma^2 \left(\eta_{\max}^t\right)^2 \notag \\
& \quad + 2\gamma_{\beta} s_0 \eta_{\max}^t \left(f(x_{(t-1)S}) - f^{\ast}\right). 
\end{align*}
As we know $\E[W_{(t-1)S+1}]$ and $\E[f(x_{(t-1)S}) - f^{\ast}]$ are uniformly bounded and $\eta_{\max}^t$ is also upper bounded, we can conclude that
$\E[W_{(t-1)S+2}]$ is also bounded. 

However, for the cosine decay mode, the only difference from polynomial and linear decay modes lies on the last iterate of each stage. In Theorem \ref{thm:mom:decay}, we show that
\begin{align*}
\E[W_{(t-1)S+1}] \leq 2\E[W_{(t-1)S}]  + 2  \theta_2 \eta_{\min}^t + ((1-\beta)^2+1)\sigma^2 \left(\eta_{\min}^t\right)^2.
\end{align*}
The estimation of $\E[W_{(t-1)S+2}]$ for the cosine decay mode is same as polynomial and linear modes. Thus, we also can make a conclusion that 
$\E[W_{k}]$ is uniformly bounded for the bandwidth step-size with the cosine decay mode.

\end{itemize}
\end{proof}

\section{Supplementary Material for Functions with Slower Growth}

\begin{proof}(of Remark \ref{rem:general:dissipative})
When $x^{\ast}=0$, the two conditions are the same. Suppose that $\left\| x^{\ast} \right\| >0$, then
\begin{align}\label{inequ:rem:grad:x}
 \left\langle \nabla f(x), x \right\rangle \geq \theta_1^{'}\left\| x \right\|^{p} - \theta_2^{'} = \theta_1^{'}\left(\left\| x -x^{\ast} + x^{\ast}\right\|^{p} \right) - \theta_2^{'}
\end{align}
We apply Young's inequality and $\left\| \nabla f(x) \right\|^2 \leq \theta_3(1+\left\|x -x^{\ast}\right\|^{2\tau})$ and assume that $R \geq 1$, then for any $x \in \R^d$ that $\left\| x - x^{\ast} \right\| \geq R$ 
\begin{align}\label{inequ:rem:grad:xast}
 \left\langle \nabla f(x), x^{\ast} \right\rangle  & \leq \frac{s^{\alpha_1}\left\|\nabla f(x) \right\|^{\alpha_1}}{\alpha_1} + \frac{\left\| x^{\ast} \right\|^{\alpha_2}}{s^{\alpha_2}\alpha_2} \leq \frac{s^{\alpha_1}\theta_3^{\alpha_1/2}\left(1+\left\| x -x^{\ast}\right\|^{2\tau}\right)^{\alpha_1/2}}{\alpha_1} + \frac{\left\| x^{\ast} \right\|^{\alpha_2}}{s^{\alpha_2}\alpha_2} \notag \\
 & \leq \frac{s^{\alpha_1}(2\theta_3)^{\alpha_1/2}\left\| x -x^{\ast}\right\|^{\tau\alpha_1}}{\alpha_1} + \frac{\left\| x^{\ast} \right\|^{\alpha_2}}{s^{\alpha_2}\alpha_2}
\end{align}
where $1/\alpha_1 + 1/\alpha_2 = 1$, $\alpha_1,\alpha_2>0$ and $s >0$.

\begin{align}\label{inequ:rem:x:xast}
\left\| x -x^{\ast} + x^{\ast}\right\|^{p}  = \left(\left\| x -x^{\ast} + x^{\ast}\right\|^2\right)^{p/2}  & \geq \left( \left\| x - x^{\ast} \right\|^2 + \left\|x^{\ast} \right\|^2 - 2 \left\|x-x^{\ast} \right\|\left\| x^{\ast} \right\|\right)^{p/2} \notag \\
& = \left(\left(\left\| x - x^{\ast} \right\| - \left\| x^{\ast} \right\|\right)^{2}\right)^{ p/2} = |\left\| x - x^{\ast} \right\| - \left\| x^{\ast} \right\||^p
\end{align}
If we assume that $R \geq \max\left\lbrace 2\left\|x^{\ast} \right\|, 1 \right\rbrace$, then for any $\left\| x - x^{\ast}\right\| \geq R$, applying (\ref{inequ:rem:x:xast}) we have 
\begin{align}
\left\| x -x^{\ast} + x^{\ast}\right\|^{p} \geq \frac{1}{2^p}\left\| x -x^{\ast} \right\|^p.
\end{align}
Then applying the above result and  (\ref{inequ:rem:grad:x}) and (\ref{inequ:rem:grad:xast}), we can get that 
\begin{align*}
\left\langle \nabla f(x), x - x^{\ast}\right\rangle \geq \frac{\theta_1^{'}}{2^p}\left\| x -x^{\ast} \right\|^p - \theta_2^{'} - \frac{s^{\alpha_1}(2\theta_3)^{\alpha_1/2}\left\|x-x^{\ast} \right\|^{\tau\alpha_1}}{\alpha_1} - \frac{\left\| x^{\ast} \right\|^{\alpha_2}}{s^{\alpha_2}\alpha_2} 
\end{align*}
Let $\alpha_1 = p/\tau$ and $s =\left(\frac{\theta_1^{'}p}{\tau2^{p+1}}\right)^{\tau/p}/\sqrt{2\theta_3}$, we have 
\begin{align*}
\frac{s^{\alpha_1}(2\theta_3)^{\alpha_1/2}\left\|x-x^{\ast} \right\|^{\tau\alpha_1}}{\alpha_1} = \frac{1}{2}\cdot \frac{\theta_1^{'}}{2^p}\left\| x -x^{\ast} \right\|^p
\end{align*}
Therefore, 
\begin{align}\label{inequ:rem}
\left\langle \nabla f(x), x - x^{\ast}\right\rangle \geq \frac{\theta_1^{'}}{2^{p+1}}\left\| x -x^{\ast} \right\|^p - \theta_2^{'}  - \frac{\left\| x^{\ast} \right\|^{\alpha_2}}{s^{\alpha_2}\alpha_2}: = \theta_1 \left\| x -x^{\ast} \right\|^p - \theta_2
\end{align}
where $\theta_1 = \frac{\theta_1^{'}}{2^{p+1}}$ and $\theta_2 = \theta_2^{'}  + \frac{\left\| x^{\ast} \right\|^{\alpha_2}}{s^{\alpha_2}\alpha_2}$ with $\alpha_2 = \frac{p}{p-\tau}$ and $s=\left(\frac{\theta_1^{'}p}{\tau2^{p+1}}\right)^{\tau/p}/\sqrt{2\theta_3}$. 
%Because $\tau \leq p/2$, we can simply $s \geq s_0 = \left(\frac{\theta_1^{'}}{2^{p}}\right)^{\frac{\tau}{p}}/\sqrt{2\theta_3}$ and $1 < \alpha_2 = \frac{p}{p-\tau} \leq 2$, then $\theta_2 \leq \theta_2^{'} + \frac{\left\| x^{\ast}\right\|^2}{s_0^{\alpha_2}}$. Then (\ref{inequ:rem}) can be simplified as
% \begin{align}
%     \left\langle \nabla f(x), x - x^{\ast}\right\rangle \geq \theta_1 \left\| x -x^{\ast} \right\|^p - \theta_2
% \end{align}
% where $\theta_1 = \frac{\theta_1^{'}}{2^{p+1}}$, $\theta_2 = \theta_2^{'} + \frac{\left\| x^{\ast}\right\|^2}{s_0^{\alpha_2}}$ with $s_0 = \left(\frac{\theta_1^{'}}{2^{p}}\right)^{\frac{\tau}{p}}/\theta_3$ and $1 < \alpha_2 = \frac{p}{p-\tau} \leq 2$.

\end{proof}
\subsection{Proofs of SGD Under Generalized $R$-dissipativity }\label{subsec:sgd}
\begin{proof}(of Theorem \ref{thm:general:sgd})
Let 
\begin{align}\label{inequ:r2}
 r^2 = \max \left\lbrace R^2, \left(\eta_{\max}(\rho+1)\frac{\theta_3}{\theta_1}\right)^{1/(p-2\tau)}, \left(\frac{2\theta_2}{\theta_1} + \frac{\left(\sigma^2 + (\rho+1)\theta_3  \right)\eta_{\max}}{\theta_1}\right)^{2/p} \right\rbrace.
\end{align}
First, we consider the initial point $x_1$ that $\left\|x_1 - x^{\ast} \right\| \leq r$. 
Suppose that there exist an iteration $k$ such that $\left\|x_{k-1} - x^{\ast} \right\|^2 \leq r^2$ and $\left\|x_{k} - x^{\ast} \right\|^2 > r^2$, then the generalized $R$-dissipativity holds at $x_k$ and its gradient is also $\tau$-growth. Applying the recursion of SGD, we have
\begin{align*}
\E[\left\| x_{k+1} - x^{\ast} \right\|^2 ] & = \E[\left\| x_k - \eta_k g_k - x^{\ast} \right\|^2] = \left\| x_k - x^{\ast} \right\|^2 - 2 \eta_k \E[\left\langle x_k - x^{\ast}, g_k \right\rangle] + \eta_k^2\E[\left\| g_k \right\|^2] \notag \\
& \leq \left\| x_k - x^{\ast} \right\|^2 - 2 \eta_k \left\langle x_k - x^{\ast}, \nabla f(x_k) \right\rangle + \eta_k^2\left(\E[\left\| g_k - \nabla f(x_k) \right\|^2] + \left\|\nabla f(x_k) \right\|^2 \right) \notag \\
& \leq \left\| x_k - x^{\ast} \right\|^2 - 2 \eta_k \left(  \theta_1\left\|x_k -x^{\ast}\right\|^{p} - \theta_2 \right) + \eta_k^2\left(\sigma^2 + (\rho+1)\left\|\nabla f(x_k) \right\|^2 \right) \notag \\
& \leq \left\| x_k - x^{\ast} \right\|^2 - 2 \eta_k \left(  \theta_1\left\|x_k -x^{\ast}\right\|^{p} - \theta_2 \right) + \eta_k^2\left(\sigma^2 + (\rho+1) \theta_3\left( 1+ \left\| x_k -x^{\ast} \right\|^{2\tau} \right)\right)\notag \\
& = \left\| x_k - x^{\ast} \right\|^2 - \eta_k \left(2\theta_1 \left\|x_k -x^{\ast}\right\|^{p} - \eta_k(\rho+1)\theta_3\left\| x_k - x^{\ast} \right\|^{2\tau}  \right)  + 2 \theta_2\eta_k + \left(\sigma^2 + (\rho+1)\theta_3  \right)\eta_k^2.
\end{align*}
By the assumption that $\tau \leq p/2$ and $\left\| x_k - x^{\ast} \right\|^2 \geq r^2$ at iteration $x_k$ and 
 $r^{2(p-2\tau)} \geq  \eta_k(\rho+1)\frac{\theta_3}{\theta_1} %\left\| x_k - x^{\ast} \right\|^{p-2\tau} \geq
$,
we can achieve that 
\begin{align*}
2\theta_1 \left\|x_k -x^{\ast}\right\|^{p} - \eta_k(\rho+1)\theta_3\left\| x_k - x^{\ast} \right\|^{2\tau} & \geq  \theta_1 \left\|x_k -x^{\ast}\right\|^{p},
\end{align*}
Then
\begin{align*}
\E[\left\| x_{k+1} - x^{\ast} \right\|^2 ] &  \leq \E[\left\| x_k - x^{\ast} \right\|^2 - \eta_k \theta_1 \left\|x_k -x^{\ast}\right\|^{p}   + 2 \theta_2\eta_k + \left(\sigma^2 + (\rho+1)\theta_3  \right)\eta_k^2.
\end{align*} 
Once 
$\left\| x_k - x^{\ast} \right\|^2 \geq r^2$, by the definition of $r^2$, we know 
\begin{align*}
r^p \geq \frac{2\theta_2}{\theta_1} + \frac{\left(\sigma^2 + (\rho+1)\theta_3  \right)\eta_k}{\theta_1}
\end{align*}
then $\E[\left\| x_{k+1} - x^{\ast}\right\|^2]$ is decreasing until $\E[\left\| x_{k+1} - x^{\ast}\right\|^2] \leq r^2$. 
At the iterates $k$, we know that $\left\|x_{k-1} - x^{\ast} \right\| \leq r^2$, then the distance of the next iteration with optimal solution $x^{\ast}$ can be bounded via:
\begin{align*}
 \E[\left\| x_k - x^{\ast}\right\|^2 \mid \mathcal{F}_{k-1}] & =  \E[\left\| x_{k-1} - \eta_{k-1} g_{k-1} - x^{\ast} \right\|^2 \mid \mathcal{F}_{k-1}] = 2\left\| x_{k-1} - x^{\ast} \right\|^2 + 2\eta_{k-1}^2\E[\left\| g_{k-1} \right\|^2 \mid \mathcal{F}_{k-1}] \notag \\
 & \leq 2 r^2 + 2\eta_{k-1}^2 \left(\E[\left\| g_{k-1} - \nabla f(x_{k-1})\right\|^2] + \left\| \nabla f(x_{k-1}) \right\|^2\right) \notag \\ 
 & \leq 2 r^2 + 2\eta_{k-1}^2 \left(\sigma^2 + (\rho+1)\left\| \nabla f(x_{k-1}) \right\|^2\right) \notag \\
 & \leq 2r^2 + 2 \eta_{\max}^2 \left( \sigma^2 + (\rho+1)\theta_3 \left(1+ \left\| x_{k-1} -x^{\ast} \right\|^2 \right) \right) \notag \\
 & \leq 2(1 + \eta_{\max}^2(\rho+1)\theta_3)r^2 + 2 \eta_{\max}^2 \left(\sigma^2 +  (\rho+1)\theta_3\right).
\end{align*}

If the initial point $x_1$ is far from the optimal point $x^{\ast}$, that is $\left\| x_1 -x^{\ast} \right\| > r$ but is bounded, by the above analysis, we can see that the follow-up $\E[\left\| x_k -x^{\ast} \right\|^2]$ is decreasing until $\E[\left\| x _k - x^{\ast} \right\|^2] \leq r^2$.

Therefore, by the above discussion, we can conclude that 
\begin{align*}
\E[\left\| x_{k} - x^{\ast} \right\|^2 ] \leq \min \left\lbrace \left\|x_1-x^{\ast} \right\|^2,   2(1 + \eta_{\max}^2(\rho+1)\theta_3)r^2 + 2 \eta_{\max}^2 \left(\sigma^2 +  (\rho+1)\theta_3\right)\right\rbrace
\end{align*}
where $r^2$ is given in (\ref{inequ:r2}).

\end{proof}

\subsection{Proofs of Momentum Under Generalized $R$-dissipativity}

Now, we turn to estimate the SGD algorithm with momentum under Assumption \ref{assump:dissipative:general}. The basic assumptions include the stochastic gradient $g_k$ satisfies Assumption \ref{assump:gradient} and the objective function is $L$-smooth. Beyond these, here we also assume that Assumption \ref{assump:dissipative:general} holds. The procedure of the proof in this part is similar to the proof of momentum under $R$-dissipativity. Therefore, we will omit proofs which are similar to Section \ref{sec:mom} but only stress the differences.
% \begin{lemma}\label{grad:lem}
% Suppose that 
% \begin{align}
%  \left\| \nabla f(x) - \nabla f(y) \right\| \leq \theta_3 \left\|x-y\right\|^{\tau} 
% \end{align}
% for all $x, y \in \R^d$ and $\tau \leq p/2$ where $p\in [0,2)$, we have
% \begin{align}
% |f(y) - f(x) - \left\langle \nabla f(x), y-x \right\rangle| \leq \theta_3 \left\| x - y \right\|^{\tau+1}.
% \end{align}
% \begin{proof}(of Lemma \ref{grad:lem})
% We assume that 
% \begin{align}
%  \left\| \nabla f(x) - \nabla f(y) \right\| \leq \theta_3 \left\|x-y\right\|^{\tau} 
% \end{align}
% for all $x, y \in \R^d$ and $\tau \leq p/2$ where $p\in [0,2)$, then we can achieve that
% \begin{align}
% |f(y) - f(x) - \left\langle \nabla f(x), y-x \right\rangle| & = \left|\int_{t=0}^{1} (\nabla f(x_t) - \nabla f(x))(y-x)dt \right| \notag \\
% & \leq \theta_3 \left\| x - y \right\|^{\tau+1}
% \end{align}
% where $x_t = (1-t)x+ty$ with $t \in (0,1)$. It implies that
% \begin{align}
%  |f(y) - f(x) - \left\langle f(x), y-x \right\rangle| \leq  \theta_3 \left\| x - y \right\|^{\tau+1}.
% \end{align}
% \end{proof}
% \end{lemma}

Because the $L$-smoothness holds, we still can apply Lemma \ref{lem:mom:1} to estimate $\E[\left\| x_{k} -x_{k-1}\right\|^2]$.

We assume the function value and its gradient satisfy the following assumptions: there exist $\theta_1 > 0, \theta_2 \geq 0, \theta_3 > 0$ and $R \geq 0$, such that for all $x$ that $\left\| x - x^{\ast} \right\| \geq R$
\begin{align*}
 \left\langle \nabla f(x), x - x^{\ast} \right\rangle \geq \theta_1\left\| x - x^{\ast} \right\|^{p} - \theta_2,\,\,\, \left\| \nabla f(x) \right\|^2 \leq \theta_3(1 + \left\| x - x^{\ast}\right\|^{2\tau})
\end{align*}
with $ 0 \leq \tau \leq p/2$ and $p \in [0,2)$.

One obvious change we made is how to estimate $\E[\left\|g_k \right\|^2 \mid \mathcal{F}_k]$:
\begin{align}\label{grad:general}
\E[\left\| g_k \right\|^2 \mid \mathcal{F}_k] & = \E[\left\| g_k - \nabla f(x_k) + \nabla f(x_k) \right\|^2] = \E[\left\| g_k -\nabla f(x_k) \right\|^2] + \left\| \nabla f(x_k) \right\|^2 \notag\\
& \leq (\rho+1)\left\|\nabla f(x_k) \right\|^2 + \sigma^2 \leq (\rho+1)\theta_3 \left\| x_k - x^{\ast} \right\|^{2\tau} + (\rho+1)\theta_3 + \sigma^2.
\end{align}

Then we turn to estimate $\E[\left\| \tilde{x}_{k+1} - x^{\ast} \right\|^2]$. The only change compared to Lemma \ref{lem:mom:2} is to estimate 
\begin{align*}
   - \E[\left\langle x_k -x^{\ast}, g_k\right\rangle  = - \left\langle x_k -x^{\ast}, \nabla f(x_k)\right\rangle  \leq - \theta_1 \left\| x_k -x^{\ast} \right\|^p + \theta_2.
\end{align*}
We then make these changes in Lemma \ref{lem:mom:2} and give the following results.
\begin{lemma}\label{mom:lem2:1}
We define $\tilde{x}_{k+1} := \frac{x_{k+1} - \beta x_k}{1-\beta}$. Suppose that Assumption \ref{assump:gradient} and (\ref{assump:dissipative:general}) hold at current iterations $x_k$ and $\tau_k =\eta_k/\eta_{k-1} \in (0,1]$, then
\begin{itemize}
    \item[(i)] If $\tau_k = 1$, \emph{i.e.}, $\eta_k = \eta$ for all $k \geq 1$, we have
    \begin{align*}
      \E[\left\| \tilde{x}_{k+1} - x^{\ast} \right\|^2 \mid \mathcal{F}_k] & \leq \left\|\tilde{x}_k - x^{\ast} \right\|^2 - 2\theta_1 \eta_k \left\| x_k - x^{\ast} \right\|^p  + 2\theta_2 \eta_k  + \eta_k^2 \E[\left\|g_k\right\|^2 \mid \mathcal{F}_k] \notag \\
      & \quad + \frac{2\beta\eta_k}{1-\beta}\left(f(x_{k-1}) - f(x_k) + \frac{L}{2}\left\| x_k - x_{k-1} \right\|^2\right).
    \end{align*}
    \item[(ii)] else if $\tau_k \in (0,1)$, we have
    \begin{align*}
     \E[\left\| \tilde{x}_{k+1} - x^{\ast} \right\|^2 \mid \mathcal{F}_k] & \leq \tau_k \left\|\tilde{x}_k - x^{\ast} \right\|^2 - 2\eta_k\theta_1 \left\| x_k - x^{\ast} \right\|^p + \left(1- \tau_k\right) \left\| x_k - x^{\ast} \right\|^2   + 2\eta_k \theta_2 \notag \\
& \quad   + \eta_k^2 \E[\left\|g_k\right\|^2 \mid \mathcal{F}_k] + \frac{2\beta L }{1-\beta}\eta_k\tau_k \left\|x_{k-1}-x_{k}\right\|^2 + \frac{2\beta \tau_k \eta_k}{1-\beta}\left( f(x_{k-1}) - f(x_k) \right).
     \end{align*}
     for any $\omega_1 >0$.
\end{itemize}
\end{lemma}

\begin{proof}(of Theorem \ref{mom:general:const})
In this case, we consider the constant step-size, that is $\eta_k = \eta$ and $\tau_k = \eta_k/\eta_{k-1} =1$. Here we define the Lyapunov function $W_{k+1}$ as 
\begin{align*}
\E[W_{k+1}] = \E[\left\|\tilde{x}_{k+1} - x^{\ast} \right\|^2 \mid \mathcal{F}_k] + \E[\left\|x_{k+1} - x_k \right\|^2 \mid \mathcal{F}_k] + \gamma_{\beta} \E[f(x_k) - f^{\ast}] 
\end{align*}
where $\gamma_{\beta} = \beta(1-\beta + (1-\beta)^{-1} )$.
Incorporating the results of Lemma \ref{lem:mom:1}, (\ref{grad:general}) and (i) of Lemma \ref{mom:lem2:1}, we can get that 
\begin{align}\label{mom:inequ:general:const}
& \E[W_{k+1}\mid \mathcal{F}_k] = \E[\left\| \tilde{x}_{k+1} - x^{\ast} \right\|^2 \mid \mathcal{F}_k] +   \E[\left\| x_{k+1} - x_k \right\|^2 \mid \mathcal{F}_k ] + 2 \eta \gamma_{\beta} \left(f(x_{k})  -f^{\ast}\right) \notag \\
& \leq \left\|\tilde{x}_k - x^{\ast} \right\|^2 + \left(\beta^2 + \eta\beta(1-\beta)L + \eta\frac{\beta L}{(1-\beta)}\right)\left\|x_{k-1}-x_{k}\right\|^2 - 2\theta_1\eta\left\|x_k - x^{\ast} \right\|^p + 2\eta \theta_2 \notag \\
& \quad + \eta^2\left((1-\beta)^2 + 1 \right)\left( \sigma^2 + (\rho+1)\theta_3 +  (\rho+1)\theta_3\left\| x_k - x^{\ast}\right\|^{2\tau} \right) + 2 \eta \gamma_{\beta} \left(f(x_{k-1})  -f^{\ast}\right) \notag \\
& \mathop{\leq}^{(a)} \E[W_k] - 2\theta_1\eta\left\|x_k - x^{\ast} \right\|^p + \eta^2\left((1-\beta)^2 + 1 \right)(\rho+1)\theta_3\left\| x_k - x^{\ast}\right\|^{2\tau} \notag \\
& \quad + \eta^2\left((1-\beta)^2 + 1 \right)\left( \sigma^2 + (\rho+1)\theta_3 \right) + 2\theta_2 \eta.
\end{align}
where $(a)$ follows from $\eta \leq  \frac{1-\beta^2}{L \gamma_{\beta}} $
to make sure that
\begin{align*}
\beta^2 + \eta\beta(1-\beta)L + \eta\frac{\beta L}{(1-\beta)} \leq 1
\end{align*}
First, we consider the initial point $x_1$ such that $\left\| x_1 - x^{\ast} \right\|\leq r$. Assume that there exist an iteration $k$ such that $\left\| x_{k-1} -x^{\ast} \right\|^2 \leq r^2$ and   $\left\| x_{k} -x^{\ast} \right\|^2 > r^2$, then the above analysis holds. 
By the assumption that $ \tau \leq p/2$,  $\left\|x_k - x^{\ast} \right\|^2 > r^2$, and $r^{2(p-2\tau)} \geq \eta (1+(1-\beta)^2)(\rho+1)\theta_3/\theta_1$, then \begin{align*}
-\theta_1 \eta \left\|x_k - x^{\ast} \right\|^p + \eta^2\left((1-\beta)^2 + 1 \right)(\rho+1)\theta_3\left\| x_k - x^{\ast}\right\|^{2\tau} < 0
\end{align*}
The main inequality (\ref{mom:inequ:general:const}) can be re-written as
\begin{align*}
 \E[W_{k+1}\mid \mathcal{F}_k]  \leq \E[W_k] - \theta_1 \eta  \left\|x_k - x^{\ast} \right\|^p + \eta^2\left((1-\beta)^2 + 1 \right)\left( \sigma^2 + (\rho+1)\theta_3 \right) + 2\theta_2 \eta.
\end{align*}
We can see that if 
\begin{align*}
r^2 = \max \left\lbrace R^2, \left(\frac{\eta (1+(1-\beta)^2)(\rho+1)\theta_3}{\theta_1}\right)^{1/(p-2\tau)}, \left(\frac{2\theta_2}{\theta_1} + \left((1-\beta)^2 + 1 \right)\left( \sigma^2 + \eta(\rho+1)\theta_3 \right)\right)^{2/p}\right\rbrace
\end{align*}
once $\E[\left\| x_k -x^{\ast} \right\|^2] > r^2$, the sequence of $\E[W_{k+1}]$ will decrease until $\E[\left\| x_k -x^{\ast} \right\|^2] \leq r^2$. Similar to the analysis of Theorem \ref{thm:mom:const}, we get the conclusion that the sequences of $\E[W_k]$, $\E[\left\|x_k - x^{\ast} \right\|^2]$, and $\E[f(x_k) - f^{\ast}]$ are uniformly bounded for all $k \geq 1$. 
\end{proof}
\begin{proof}(of Theorem \ref{mom:general:decaying})
Different from the proof of Theorem \ref{mom:general:const}, we have additional positive term $(1-\tau_k)\E[\left\|x_k - x^{\ast} \right\|^2]$ to deal with. We then need to estimate $\E[\left\| x_{k+1} - x^{\ast} \right\|^2 \mid \mathcal{F}_k]$.
\begin{align*}
\E[\left\| x_{k+1} - x^{\ast} \right\|^2 \mid \mathcal{F}_k] & = \E\left[\left\| x_k - \eta_k \left( \beta \frac{x_{k-1} - x_{k}}{\eta_{k-1}} + (1-\beta) g_k \right) - x^{\ast}\right\|^2\right] \notag \\
& = \E\left[\left\| x_k -x^{\ast} - \beta \tau_k(x_{k-1}-x_k) - (1-\beta)\eta_k g_k \right\|^2 \right] \notag \\ 
& =  \E[\left\| x_k -x^{\ast} - \beta \tau_k(x_{k-1}-x_k)\right\|^2] + \eta_k^2(1-\beta)^2\E[\left\|g_k \right\|^2 \mid \mathcal{F}_k]  \notag \\
& \quad + 2\beta(1-\beta)\tau_k
\eta_k\E[\left\langle x_{k-1}-x_k, g_k\right\rangle]- 2(1-\beta)\eta_k \E[\left\langle x_k -x^{\ast}, g_k \right\rangle] \notag \\
& = \E\left[\left\| x_k -x^{\ast} + \beta \tau_k \frac{1-\beta}{\beta}(\tilde{x}_k -x_k) \right\|^2 \right] + \eta_k^2(1-\beta)^2\E[\left\|g_k \right\|^2 \mid \mathcal{F}_k]  \notag \\
& \quad + 2\beta(1-\beta)\tau_k
\eta_k\E[\left\langle x_{k-1}-x_k, g_k\right\rangle]- 2(1-\beta)\eta_k \E[\left\langle x_k -x^{\ast}, g_k \right\rangle].
\end{align*}
By the convexity of $\left\| \cdot \right\|^2$, we have
\begin{align*}
\E\left[\left\| x_k -x^{\ast} + \beta \tau_k (1-\beta)\beta^{-1}(\tilde{x}_k -x_k) \right\|^2 \right] & = \E\left[\left\|(1-\tau_k(1-\beta))( x_k -x^{\ast}) +  \tau_k(1-\beta)(\tilde{x}_k -x^{\ast}) \right\|^2 \right] \notag \\
& \leq (1-\tau_k(1-\beta))\E\left[\left\|( x_k -x^{\ast}) \right\|^2\right] +  \tau_k(1-\beta)\E\left[\left\|\tilde{x}_k -x^{\ast} \right\|^2 \right].
\end{align*}
Then we turn to estimate 
\begin{align*}
- \E[\left\langle x_k -x^{\ast}, g_k \right\rangle] = -\left\langle x_k -x^{\ast}, \nabla f(x_k)\right\rangle \leq - \theta_1 \left\| x_k -x^{\ast} \right\|^p + \theta_2. 
\end{align*}
and 
\begin{align*}
\E[\left\langle x_{k-1}-x_k, g_k\right\rangle] = \left\langle x_{k-1}-x_k, \nabla f(x_k)\right\rangle \leq f(x_{k-1}) - f(x_k) + \frac{L}{2}\left\|x_k -x_{k-1} \right\|^2.
\end{align*}
% \begin{align}
%  \E[\left\| x_{k+1} - x_k \right\|^2 \mid \mathcal{F}_k ] 
% & \leq \tau_k^2\beta^2 \left\| x_k - x_{k-1} \right\|^2 + \eta_k^2(1-\beta)^2\E[\left\| g_k \right\|^2 \mid \mathcal{F}_k] \notag \\
% & \quad + 2\beta(1-\beta)\tau_k\eta_k\left(f(x_{k-1}) - f(x_k) + \frac{L}{2}\left\| x_k - x_{k-1}\right\|^2\right).
% \end{align}
Then we define the Lyapunov function $W_{k+1}$ below:
\begin{align*}
\E[W_{k+1} \mid \mathcal{F}_k] & =  \E[\left\|\tilde{x}_{k+1} - x^{\ast} \right\|^2] + c_{k+1}\E[\left\| x_{k+1} - x^{\ast} \right\|^2\mid \mathcal{F}_k] + \E[\left\| x_{k+1} -x_{k} \right\|^2\mid \mathcal{F}_k] + u_{k}\E[f(x_k) - f^{\ast}\mid \mathcal{F}_k]
\end{align*}
where  $c_{k+1}>0$, $\gamma_{\beta} = \beta(1-\beta + (1-\beta)^{-1})$, $u_k = 2\gamma_{\beta} \tau_k\eta_k + 2 \beta (1-\beta)\tau_k \eta_k c_{k+1} $. %and the positive sequence $c_k$ is decreasing. 
Then incorporating the above results gives:
\begin{align}\label{mom:general:decay:2}
& \E[W_{k+1} \mid \mathcal{F}_k]  =  \E[\left\|\tilde{x}_{k+1} - x^{\ast} \right\|^2] + c_{k+1} \E[\left\| x_{k+1} - x^{\ast} \right\|^2\mid \mathcal{F}_k] + \E[\left\| x_{k+1} -x_{k} \right\|^2\mid \mathcal{F}_k] + u_{k}\E[f(x_k) - f^{\ast}\mid \mathcal{F}_k] \notag \\
& \leq \left(\tau_k  + c_{k+1} \tau_k(1-\beta)\right) \left\|\tilde{x}_{k} - x^{\ast} \right\|^2 +  c_k\left\| x_{k} - x^{\ast} \right\|^2   + \left(c_{k+1} - c_k -c_{k+1}\tau_k(1-\beta) + (1-\tau_k)\right)\left\| x_{k} - x^{\ast} \right\|^2 \notag \\
& \quad +  \left(\tau_k^2\beta^2 + \beta(1-\beta)\tau_k\eta_k L + c_{k+1} \beta(1-\beta)\tau_k \eta_k L + \frac{2 \beta L\eta_k\tau_k }{1-\beta}\right)\left\| x_{k} - x_{k-1} \right\|^2 - 2 \eta_k\theta_1\left\| x_k - x^{\ast} \right\|^p + 2 \eta_k \theta_2 \notag \\
& \quad   + 2(1-\beta)\eta_k c_{k+1} \left(- \theta_1\left\| x_k - x^{\ast} \right\|^p + \theta_2\right) + \left(\eta_k^2(1-\beta)^2 + \eta_k^2 + \eta_k^2 c_{k+1} (1-\beta)^2 \right)\E[\left\| g_k \right\|^2] \notag \\
& \quad   + 2\gamma_{\beta}\left( \tau_k\eta_k - \tau_{k-1}\eta_{k-1} \right)\left(f(x_{k-1}) - f^{\ast} \right) + 2\beta(1-\beta)\left( \tau_k\eta_k c_{k+1} - \tau_{k-1}\eta_{k-1} c_k\right)\left(f(x_{k-1}) - f^{\ast} \right) \notag \\
& \quad + u_{k-1}(f(x_{k-1}) - f^{\ast}).
\end{align}
 First, we properly choose $c_{k+1} = (1-\tau_k)/(\tau_k(1-\beta))$ to make the scalar  of $\left\|\tilde{x}_{k} - x^{\ast} \right\|^2$ is smaller than 1.
 %, we then try to estimate the ratio
% \begin{align}
%  \chi = \frac{(1-\tau_k)^2}{c_{k} - c_{k+1}} = \frac{(1-\tau_k)^2}{\eta_{k-1} - \eta_{k}}
% \end{align}
We make the following discussion on the two popular decaying modes:
\begin{itemize}
    \item The step-size is polynomial decaying $\eta_k = \eta_1/k^r$ for all $r \in (0,1]$, then $c_{k+1} = (1-\tau_k)/(\tau_k(1-\beta)) \leq \frac{1}{(k-1)^r (1-\beta)}$ is decreasing with $k$. The scalar $ c_{k+1} - c_k -c_{k+1}\tau_k(1-\beta) + (1-\tau_k)$ is negative. By properly choosing the step-size \begin{align*}
        \eta_k \leq  \frac{1-\beta^2}{\beta L((1-\beta)(1+c_{k+1}) + 2(1-\beta)^{-1})} \leq \frac{1-\beta^2}{2\beta L((1-\beta) + (1-\beta)^{-1})}
    \end{align*}
    to ensure the coefficient of $\E[\left\|x_k - x_{k-1} \right\|^2]$ is smaller than 1. By case 1 of Theorem \ref{thm:mom:decay}, we know that $\tau_k\eta_k - \tau_{k-1}\eta_{k-1}$ is negative for all $k$. Because $c_{k+1}$ is decreasing, we can get that $\tau_k\eta_kc_{k+1}-\tau_{k-1}\eta_{k-1}c_k \leq c_k(\tau_k\eta_k-\tau_{k-1}\eta_{k-1})  $ which is also negative.
    Then the main inequality (\ref{mom:general:decay:2}) can be evaluated as:
    \begin{align*}
      \E[W_{k+1} \mid \mathcal{F}_k] & \leq W_k -   2\eta_k\theta_1\left\| x_k -x^{\ast} \right\|^p + 2\eta_k\theta_2 + 2(1-\beta)\eta_k c_{k+1} \left(- \theta_1\left\| x_k - x^{\ast} \right\|^p + \theta_2\right)  \notag \\
      & \quad + \left(\eta_k^2(1-\beta)^2 + \eta_k^2 + \eta_k^2 c_{k+1} (1-\beta)^2 \right)\E[\left\| g_k \right\|^2]
    \end{align*}
Incorporating (\ref{grad:general}) into the above inequality, we get that
\begin{align}\label{inequ:mom:general:key}
   \E[W_{k+1} \mid \mathcal{F}_k] & \leq W_k -   2\eta_k\theta_1\left\| x_k -x^{\ast} \right\|^p + 2\eta_k\theta_2 + 2(1-\beta)\eta_k c_{k+1} \left(- \theta_1\left\| x_k - x^{\ast} \right\|^p + \theta_2\right)  \notag \\
   & \quad + \left(\eta_k^2(1-\beta)^2 + \eta_k^2 + \eta_k^2 c_{k+1} (1-\beta)^2 \right)\left( (\rho+1)\theta_3 \left\| x_k - x^{\ast} \right\|^{2\tau} + (\rho+1)\theta_3 + \sigma^2\right)
\end{align}
Assume that there exist an iteration $k$ such that $\left\| x_{k-1}-x^{\ast} \right\|^2 \leq r^2$ and $\left\|x_k -x^{\ast}\right\|^2 > r^2$, then the above analysis holds. Due to the assumption that $\tau \leq p/2$, then 
we choose 
\begin{align*}
r^{(p-2\tau)} \geq \max\left\lbrace \frac{\eta_{1}((1-\beta)^2+1)(\rho+1)\theta_3}{\theta_1}, \frac{\eta_{1}(1-\beta)(\rho+1)\theta_3}{2\theta_1} \right\rbrace = \frac{\eta_{1}((1-\beta)^2+1)(\rho+1)\theta_3}{\theta_1}
\end{align*}
to make sure that $(\eta_k^2(1-\beta)^2 + \eta_k^2)(\rho+1)\theta_3\left\| x_k -x^{\ast} \right\|^{2\tau} \leq \theta_1\eta_k\left\|x_k -x^{\ast} \right\|^p$ and $\eta_k^2(1-\beta)^2c_{k+1}(\rho+1)\theta_3\left\| x_k -x^{\ast} \right\|^{2\tau} \leq 2(1-\beta)\eta_kc_{k+1}\theta_1\left\|x_k -x^{\ast} \right\|^p$. Then the inequality (\ref{inequ:mom:general:key}) can be estimated as
\begin{align*}
    \E[W_{k+1} \mid \mathcal{F}_k] & \leq W_k - \eta_k\theta_1\left\| x_k -x^{\ast} \right\|^p + 2\eta_k\theta_2 + 2(1-\beta)\eta_k c_{k+1}\theta_2 \notag \\
   & \quad + \left(\eta_k^2(1-\beta)^2 + \eta_k^2 + \eta_k^2 c_{k+1} (1-\beta)^2 \right)\left((\rho+1)\theta_3 + \sigma^2\right)
\end{align*}
By properly choosing 
\begin{align*}
r^p \leq \frac{4\theta_2}{\theta_1} + \frac{\eta_1(2-\beta)^2((\rho+1)\theta_3 + \sigma^2)}{\theta_1}
\end{align*}
once $\E[\left\|x_k - x^{\ast} \right\|^2] > r^2$, we have $\E[W_k]$ is decreasing. Overall, we set 
\begin{align*}
r^2 = \max \left\lbrace R^2, \left(\frac{4\theta_2}{\theta_1} + \frac{\eta_1(2-\beta)^2((\rho+1)\theta_3 + \sigma^2)}{\theta_1}\right)^{2/p}, \left(\frac{\eta_{1}((1-\beta)^2+1)(\rho+1)\theta_3}{\theta_1} \right)^{2/(p-2\tau)} \right\rbrace.
\end{align*}
Similar to the discussion of Theorem \ref{thm:mom:decay}, we can conclude that $\E[\left\|x_k -x^{\ast} \right\|^2]$ and $\E[f(x_k) - f^{\ast}]$ are uniformly bounded for all $k \geq 1$. 
    \item The step-size is exponentially decaying with $\eta_k = \eta_1/\alpha^{k-1}$ and $\alpha = (\nu/T)^{-1/T}$ with $\nu \geq 1$. In this case, $c_{k+1} = (1-\tau_k)/(\tau_k(1-\beta)) = (1-\alpha)/(\alpha(1-\beta))$ where $\tau_k = \eta_k/\eta_{k-1}=\alpha$. The the coefficient of $c_{k+1} -c_k - c_{k+1}\tau_k(1-\beta) + (1-\tau_k)=0$. By properly choosing step-size
    \begin{align*}
        \eta_k \leq \frac{1-\beta^2\tau_k^2}{\beta L((1-\beta)\tau_k + (1-\tau_k) + 2\tau_k(1-\beta)^{-1})} = \frac{1-\beta^2\tau_k^2}{\beta L((1-\beta\tau_k  + 2\tau_k(1-\beta)^{-1})},
    \end{align*}
   which is decreasing with $\tau_k$, that is \begin{align}
    \eta_k \leq \frac{1-\beta^2}{\beta L(1-\beta + 2(1-\beta)^{-1})}
    \end{align}
    we can make sure the coefficient of $\E[\left\|x_k -x_{k-1} \right\|^2]$ is smaller than 1.  By the Case 3 of Theorem \ref{thm:mom:decay}, we know that $\tau_k\eta_k - \tau_{k-1}\eta_{k-1}$ is negative. Similar to the polynomial step-size, combining (\ref{grad:general}), the main inequality (\ref{inequ:mom:general:key}) will be
    \begin{align*}
      \E[W_{k+1} \mid \mathcal{F}_k] & \leq W_k -   2\eta_k\theta_1\left\| x_k -x^{\ast} \right\|^p + 2\eta_k\theta_2 + 2(1-\beta)\eta_k c_{k+1} \left(- \theta_1\left\| x_k - x^{\ast} \right\|^p + \theta_2\right)  \notag \\
   & \quad + \left(\eta_k^2(1-\beta)^2 + \eta_k^2 + \eta_k^2 c_{k+1} (1-\beta)^2 \right)\left( (\rho+1)\theta_3 \left\| x_k - x^{\ast} \right\|^{2\tau} + (\rho+1)\theta_3 + \sigma^2\right)
    \end{align*}
  Assume that there exists the iterate $k$ such that $\E[\left\| x_{k-1} -x^{\ast}\right\|^2] \leq r^2$ and $\E[\left\|x_k-x^{\ast} \right\|^2] > r^2$, then by properly setting   
 \begin{align*}
    r^2 =  \max \left\lbrace R^2, \left(\frac{2\theta_2}{\alpha\theta_1} + \frac{\eta_1((1-\beta)^2+\alpha^{-1})((\rho+1)\theta_3 + \sigma^2)}{\theta_1}\right)^{2/p}, \left(\frac{\eta_{1}((1-\beta)^2+1)(\rho+1)\theta_3}{\theta_1} \right)^{2/(p-2\tau)} \right\rbrace,
 \end{align*}
 we can conclude that once $\E[\left\| x_k -x^{\ast} \right\|^2] > r^2$, we have $\E[W_{k+1}]$ is decreasing. Similar to Case 3 of Theorem \ref{thm:mom:decay}, the quantities $\E[\left\|x_k-x^{\ast} \right\|^2]$ and $\E[f(x_k) - f^{\ast}]$
are uniformly bounded for all $k \geq 1$.
\end{itemize}

%   \begin{align*}
%      \E[\left\| \tilde{x}_{k+1} - x^{\ast} \right\|^2 \mid \mathcal{F}_k] & \leq \left\|\tilde{x}_k - x^{\ast} \right\|^2 - 2\eta_k\theta_1 \left\| x_k - x^{\ast} \right\|^p + \left(1- \tau_k\right)^2 \left\| x_k - x^{\ast} \right\|^2   + 2\eta_k \theta_2 + \eta_k^2 \E[\left\|g_k\right\|^2 \mid \mathcal{F}_k]\notag \\
% & \quad + \frac{\beta\tau_k }{1-\beta}\left(2\eta_k L +\frac{\beta\tau_k}{(1-\beta)} \right)\left\|x_{k-1}-x_{k}\right\|^2 + \frac{2\beta \tau_k \eta_k}{1-\beta}\left( f(x_{k-1}) - f(x_k) \right).
%      \end{align*}

% \begin{align}
% \E[\left\| x_{k+1} - x_{k} \right\|^2]    \leq  \frac{\eta_k^2\beta^2}{\eta_{k-1}^2} \left\| x_k - x_{k-1} \right\|^2 + \eta_k^2(1-\beta)^2\E[\left\| g_k \right\|^2\mid \mathcal{F}_k ] + 2\frac{\eta_k^2\beta(1-\beta)}{\eta_{k-1}}\left(f(x_{k-1}) - f(x_k) + \frac{L}{2}\left\| x_k - x_{k-1}\right\|^2\right).
% \end{align} 

\end{proof}

\section{Applications}
\label{supple:application}
\subsection{Examples in Section \ref{sec:dissipative}}\label{dissipative:example}
\begin{itemize}
        \item A fully connected 2-layer neural network (hidden nodes $m$) with a weight decay: let $\sigma_1$ and $\sigma_2(\cdot)$ denote the inner and output layer activation functions, respectively. We use cross-entropy to measure the output with its true label $CE(\cdot)$. Given $n$ pairs input data $\left\lbrace a_j, b_j\right\rbrace_{j=1}^n$, we revise the input data as $a_j = [a_j, 1] \in \R^d$ then
        \begin{align*}
            F(x) = \frac{1}{n}\sum_{j=1}^n CE(\sigma_2(X_2\sigma_1(X_1a_j))) +  \frac{\lambda}{2}\left\|X \right\|^2
        \end{align*}
        where $X = [X_1, X_2] \in \R^{dm + sm}$, $X_1 \in \R^{m \times d}$ and $X_2 \in \R^{s\times m}$. 
       For simplicity, we consider the binary dataset $b_j \in \left\lbrace -1, +1\right\rbrace$ and $s=1$, and $\sigma_2(\cdot) = 1/(1+\exp(-\cdot))$, then
\begin{align*}
            F(x) = \frac{1}{n}\sum_{j=1}^n\log\left(1+\exp(-b_j\sum_{i=1}^{m}X_2[i]\sigma_1(X_1[i]^{T}a_j) )\right) + \frac{\lambda}{2}\left\|X \right\|^2
        \end{align*}
% \begin{align}
%  \nabla F_i(X) = \left[\frac{\partial F_i}{\partial X_1},\frac{\partial F_i}{\partial X_2} \right]
% \end{align}  
Then we turn to compute the gradient (sub-gradient) of each component function $F_j$ with regarding to data $(a_j, b_j)$
\begin{align*}
\frac{\partial F_j}{\partial X_1[i]} & = \frac{b_j\exp(-b_j\sum_{i=1}^{m}X_2[i]\sigma_1(X_1[i]^{T}a_j))}{1 + \exp(-b_j\sum_{i=1}^{m}X_2[i]\sigma_1(X_1[i]^{T}a_j)}\cdot -X_2[i]\partial \sigma_1(X_1[i]^{T}a_j)a_j + \lambda X_1[i] \in \R^d \notag \\
\frac{\partial F_j}{\partial X_2[i]} & =\frac{b_j\exp(-b_j\sum_{i=1}^{m}X_2[i]\sigma_1(X_1[i]^{T}a_j))}{1 + \exp(-b_j\sum_{i=1}^{m}X_2[i]\sigma_1(X_1[i]^{T}a_j)}\cdot - \sigma_1(X_1[i]^{T}a_j) + \lambda X_2[i] \notag  
\end{align*}
%Our purpose here is to estimate the lower bound of $\left\langle X, \nabla F_j(X)\right\rangle$. 
Let $u = b_j\sum_{i=1}^{m}X_2[i]\sigma_1(X_1[i]^{T}a_j)$, we have
\begin{align*}
\left\langle X, \nabla F_j(X)\right\rangle & = \sum_{i=1}^m \left(\left\langle X_1[i], \frac{\partial F_j}{\partial X_1[i]}\right\rangle + \left\langle X_2[i], \frac{\partial F_j}{\partial X_2[i]}\right\rangle \right) \notag \\
& = \lambda \left\|X\right\|^2 - \frac{b_j\exp(-u)}{1+\exp(-u)}\sum_{i=1}^m\left( X_2[i]\partial \sigma_1(X_1[i]^{T}a_j)X_1[i]^{T}a_j + X_2[i]\sigma_1(X_1[i]^{T}a_j) \right).
\end{align*}
\begin{itemize}
    \item If $\sigma_1(x) = \max(0, x)$, then its gradient $\partial \sigma_1(x) = \chi_{x \geq 0}$. We get the relationship of $\sigma_1$ with its gradient $\sigma_1(x) = (x)_{+} = x \cdot \chi_{x \geq 0}= x  \cdot \partial \sigma_1(x)$.  Then the above product can be estimated as 
% \begin{itemize}
%     \item If  $\sigma_1(X_1[i]^{T}a_j)$ is zero, then the corresponding gradient $\partial \sigma_1(X_1[i]^{T}a_j)$ is also zero. So we consider the case that each element of $\sigma_1(X_1a_j) $ is not zero. Then $\sigma_1(X_1a_j) =X_1a_j$ and its derivation is 1. 
%     Then we have
%   \begin{align}
% \frac{\partial F_j}{\partial X_1[i]} & = \frac{b_j\exp(-b_j\sum_{i=1}^{m}X_2[i](X_1[i]^{T}a_j))}{1 + \exp(-b_j\sum_{i=1}^{m}X_2[i](X_1[i]^{T}a_j))}\cdot -X_2[i] a_j + \lambda X_1[i] \notag \\
% \frac{\partial F_j}{\partial X_2[i]} & =\frac{b_j\exp(-b_j\sum_{i=1}^{m}X_2[i](X_1[i]^{T}a_j))}{1 + \exp(-b_j\sum_{i=1}^{m}X_2[i](X_1[i]^{T}a_j))}\cdot -(X_1[i]^{T}a_j) + \lambda X_2[i] \notag \end{align}
% Let $\rho = \frac{\exp(-b_j\sum_{i=1}^{m}X_2[i](X_1[i]^{T}a_j))}{1 + \exp(-b_j\sum_{i=1}^{m}X_2[i](X_1[i]^{T}a_j))}$, we have
% \end{itemize}    
  \begin{align*}
    \left\langle X, \nabla F_j(X)\right\rangle & = \lambda\left\| X\right\|^2 - \frac{b_j\exp(-u)}{1+\exp(-u)}\sum_{i=1}^m\left( X_2[i]\partial \sigma_1(X_1[i]^{T}a_j)X_1[i]^{T}a_j + X_2[i]\sigma_1(X_1[i]^{T}a_j) \right) \notag \\
    &  = \lambda\left\| X\right\|^2  - \frac{b_j\exp(-u)}{1+\exp(-u)}\sum_{i=1}^m\left( X_2[i] \sigma_1(X_1[i]^{T}a_j) + X_2[i]\sigma_1(X_1[i]^{T}a_j) \right) \notag \\
    & = \lambda\left\| X\right\|^2  - \frac{2u \exp(-u)}{1+\exp(-u)}.  
    \end{align*}  
 No matter $u \geq 0$ or not, we all have
\begin{align*}
     \frac{\exp(-u) u }{1+\exp(-u)} = \frac{u}{1+\exp(u)} < 1.
 \end{align*}
 In this case, we have
 \begin{align*}
    \left\langle X, \nabla F(X)\right\rangle = \frac{1}{n}\sum_{j=1}^n \left\langle X, \nabla F_j(X)\right\rangle \geq \lambda\left\| X\right\|^2  - 2. 
 \end{align*}
    \item If $\sigma_1(x) = 1/(1+\exp(-x))$ is a sigmoid loss function, we have $\frac{\partial \sigma_1(x)}{\partial x} = \sigma_1(x) (1-\sigma_1(x))$. In this case, we have
    \begin{align*}
      \left\langle X, \nabla F_j(X)\right\rangle & = \lambda\left\| X\right\|^2 - \frac{b_j\exp(-u)}{1+\exp(-u)}\sum_{i=1}^m\left( X_2[i]\partial \sigma_1(X_1[i]^{T}a_j)X_1[i]^{T}a_j + X_2[i]\sigma_1(X_1[i]^{T}a_j) \right) \notag \\
      & = \lambda\left\| X\right\|^2 - \frac{u\exp(-u)}{1+\exp(-u)} - \frac{b_j\exp(-u)}{1+\exp(-u)}\sum_{i=1}^m X_2[i] \sigma_1(X_1[i]^{T}a_j)(1-\sigma_1(X_1[i]^{T}a_j))X_1[i]^{T}a_j
    \end{align*}
    We consider different situations to show how to estimate the last term. We let $A = \left\lbrace i : b_jX_2[i] \geq 0\right\rbrace $.
    \begin{itemize}
        \item  If $A = [m]$, then
        \begin{align*}
    \frac{\exp(-u)}{1+\exp(-u)}\sum_{i=1}^m b_jX_2[i] \sigma_1(X_1[i]^{T}a_j)(1-\sigma_1(X_1[i]^{T}a_j))X_1[i]^{T}a_j    & \leq  \frac{\exp(-u)}{1+\exp(-u)}\sum_{i=1}^m b_jX_2[i] \notag \\
    & = \frac{u\exp(-u)}{1+\exp(-u)}. 
        \end{align*}
 where the inequality follows from the facts that  $\sigma_1(x) \in [0,1]$ and 
    \begin{align}\label{inequ:xa}
    (1-\sigma_1(X_1[i]^{T}a_j))X_1[i]^{T}a_j = \frac{X_1[i]^{T}a_j}{1+\exp(X_1[i]^{T}a_j)} <  1
    \end{align}
        \item If some of $b_jX_2[i] \geq 0$ are negative. That is $A \subseteq [m]$. Then we can split the sum into two cases
        \begin{itemize}
            \item For the term that $i \in A$, we have $b_jX_2[i] \geq 0$. Let $B = \left\lbrace i:  X_1[i]^{T}a_j \geq 0 \right\rbrace$. We can also achieve that
            \begin{align}\label{inequ:neural:layer}
       & \frac{\exp(-u)}{1+\exp(-u)}\sum_{i=1}^m b_jX_2[i] \sigma_1(X_1[i]^{T}a_j)(1-\sigma_1(X_1[i]^{T}a_j))X_1[i]^{T}a_j  \notag \\
       & = \frac{\exp(-u)}{1+\exp(-u)}\sum_{i\in A}^m b_jX_2[i]\sigma_1(X_1[i]^{T}a_j)(1-\sigma_1(X_1[i]^{T}a_j))X_1[i]^{T}a_j  \notag \\
       & \quad +  \frac{\exp(-u)}{1+\exp(-u)}\sum_{i \in ([m]\setminus A) \cap B} b_jX_2[i]\sigma_1(X_1[i]^{T}a_j)(1-\sigma_1(X_1[i]^{T}a_j))X_1[i]^{T}a_j  \notag \\
       & \quad +  \frac{\exp(-u)}{1+\exp(-u)}\sum_{ i \in ([m]\setminus A) \cap (m \setminus B) } b_jX_2[i]\sigma_1(X_1[i]^{T}a_j)(1-\sigma_1(X_1[i]^{T}a_j))X_1[i]^{T}a_j. 
    \end{align}
  By (\ref{inequ:xa}), the first term of (\ref{inequ:neural:layer}) can be bounded by
  \begin{align*}
     \sum_{i\in A}^m b_jX_2[i]\sigma_1(X_1[i]^{T}a_j)(1-\sigma_1(X_1[i]^{T}a_j))X_1[i]^{T}a_j \leq \sum_{i\in A}^m b_jX_2[i]\sigma_1(X_1[i]^{T}a_j).
  \end{align*}
  We turn to estimate the second term of (\ref{inequ:neural:layer}):
  \begin{align*}
 \sum_{i \in ([m]\setminus A) \cap B} b_jX_2[i]\sigma_1(X_1[i]^{T}a_j)(1-\sigma_1(X_1[i]^{T}a_j))X_1[i]^{T}a_j \leq 0. \end{align*}
 The last term of (\ref{inequ:neural:layer}) is estimated as:
  \begin{align*}
 & \sum_{ i \in ([m]\setminus A) \cap (m \setminus B) } b_jX_2[i]\sigma_1(X_1[i]^{T}a_j)(1-\sigma_1(X_1[i]^{T}a_j))X_1[i]^{T}a_j \notag \\
 & 
 =  \sum_{ i \in ([m]\setminus A) \cap (m \setminus B) } (-b_jX_2[i])\sigma_1(X_1[i]^{T}a_j)(1-\sigma_1(X_1[i]^{T}a_j))(-X_1[i]^{T}a_j) \notag \\
 & = \sum_{ i \in ([m]\setminus A) \cap (m \setminus B) } (-b_jX_2[i])(1-\sigma_1(-X_1[i]^{T}a_j))\sigma_1(-X_1[i]^{T}a_j)(-X_1[i]^{T}a_j) \notag \\
 & \leq \sum_{ i \in ([m]\setminus A) \cap ([m] \setminus B) } (-b_jX_2[i])\sigma_1(-X_1[i]^{T}a_j)).
  \end{align*} 
 \end{itemize}
Applying the above results gives 
 \begin{align*}
     & \frac{\exp(-u)}{1+\exp(-u)}\sum_{i=1}^m b_jX_2[i] \sigma_1(X_1[i]^{T}a_j)(1-\sigma_1(X_1[i]^{T}a_j))X_1[i]^{T}a_j  \notag \\
     & \leq  \frac{\exp(-u)}{1+\exp(-u)}\left( \sum_{i\in A}^m b_jX_2[i]\sigma_1(X_1[i]^{T}a_j) + \sum_{ i \in ([m]\setminus A) \cap ([m] \setminus B) } (-b_jX_2[i])\sigma_1(-X_1[i]^{T}a_j))  \right) \notag \\
     & \leq \frac{\exp(-u)}{1+\exp(-u)}\left( \frac{c_1}{2} \left( \sum_{i=1}^m \left\| X_2[i] \right\|^2  \right)+ \frac{1}{2c_1} \left(\sum_{i \in A} \sigma_1^2(X_1[i]^Ta_j) + \sum_{ i \in ([m]\setminus A) \cap ([m] \setminus B) } \sigma_1^2(-X_1[i]^Ta_j)\right) \right) \notag \\
     & \leq \left( \frac{c_1}{2}  \left\| X_2\right\|^2+ \frac{m}{2c_1} \right).
 \end{align*}         
    \end{itemize}
  Let $c_1 = \lambda$, then
 \begin{align*}
      \left\langle X, \nabla F(X)\right\rangle  = \frac{1}{n}\sum_{j=1}^n \left\langle X, \nabla F_j(X)\right\rangle  &\geq \lambda\left\| X\right\|^2  - \frac{u\exp(-u)}{1+\exp(-u)} - \left( \frac{\lambda}{2}  \left\| X_2\right\|^2+ \frac{m}{2\lambda} \right) \notag \\
      &
     \geq \frac{\lambda}{2}\left\| X\right\|^2 - \left(1+ \frac{m}{2\lambda} \right).
 \end{align*}

    \item If the neural networks are three layers,  we consider ReLU as the activation function of the inner layers.
 
  \begin{align*}
            F(X) = \frac{1}{n}\sum_{j=1}^n\log\left(1+\exp\left(-b_j\sum_{i=1}^{m_2}X_3[i]\sigma_2\left(X_2[i]^{T}\sigma_1\left(X_1^{T}a_j \right)\right) \right)\right) + \frac{\lambda}{2}\left\|X \right\|^2.
        \end{align*}
Next we choose one component function $F_j$ which is related to the $j$-th dataset to analyze its gradient and product term.       
Let $ u = b_j\sum_{i=1}^{m_2}X_3[i]\sigma_2\left(X_2[i]^{T}\sigma_1\left(X_1^{T}a_j \right)\right)$        
   \begin{align*}
        \frac{\partial F_j}{\partial X_1[l]} & = \frac{\exp(-u)}{1+\exp(-u)} \cdot  - b_j \sum_{i=1}^{m_2}X_3[i]\partial \sigma_2\left(X_2[i]^{T}\sigma_1\left(X_1^{T}a_j \right)\right) X_2[i][l]\partial \sigma_1(X_1[l]^{T}a_j)a_j + \lambda X_1[l]\notag \\
        \frac{\partial F_j}{\partial X_2[i]} & =  \frac{\exp(-u)}{1+\exp(-u)} \cdot - b_j X_3[i]\partial \sigma_2\left(X_2[i]^{T}\sigma_1\left(X_1^{T}a_j \right)\right)\sigma_1(X_1^{T}a_j) + \lambda X_2[i] \notag \\
        \frac{\partial F_j}{\partial X_3[i]} & = \frac{\exp(-u)}{1+\exp(-u)} \cdot - b_j  \sigma_2\left(X_2[i]^{T}\sigma_1\left(X_1^{T}a_j \right)\right) + \lambda X_3[i].
  \end{align*}     
      Then 
    \begin{align*}
  \left\langle X, \nabla F_j(X)\right\rangle   = \sum_{l=1}^{m_1}\left\langle X_1[l], \frac{\partial F_j(X)}{\partial X_1[l]}\right\rangle + \sum_{i=1}^{m_2}\left\langle X_2[i], \frac{\partial F_j(X)}{\partial X_2[i]}\right\rangle + \sum_{i=1}^{m_2}\left\langle X_3[i], \frac{\partial F_j(X)}{\partial X_3[i]}\right\rangle.
    \end{align*}  
Let $\sigma_1(x) = \sigma_2(x) = \max(0, x)$, then their gradients  $\partial \sigma_1(x) = \partial \sigma_2(x) = \chi_{x \geq 0}$. Then the activation functions with their gradient satisfy that $\sigma_i(x) =  x \partial \sigma_i(x)$  for $i=1$ and $2$.  Then
\begin{align*}
& \sum_{l=1}^{m_1}\left\langle X_1[l], \frac{\partial F_j(X)}{\partial X_1[l]}\right\rangle \notag \\
& = \sum_{l=1}^{m_1}\frac{ - b_j\exp(-u)}{1+\exp(-u)} \cdot  \sum_{i=1}^{m_2}X_3[i]\partial \sigma_2\left(X_2[i]^{T}\sigma_1\left(X_1^{T}a_j \right)\right) X_2[i][l]\partial \sigma_1(X_1[l]^{T}a_j)X_1[l]^{T}a_j + \lambda \left\|X_1[l] \right\|^2 \notag \\
& = \sum_{l=1}^{m_1}\frac{ - b_j\exp(-u)}{1+\exp(-u)} \cdot  \sum_{i=1}^{m_2}X_3[i]\partial \sigma_2\left(X_2[i]^{T}\sigma_1\left(X_1^{T}a_j \right)\right) X_2[i][l] \sigma_1(X_1[l]^{T}a_j) + \lambda \left\|X_1[l] \right\|^2 \notag \\
& =  \frac{ - b_j\exp(-u)}{1+\exp(-u)} \cdot  \sum_{i=1}^{m_2}X_3[i] \sigma_2\left(X_2[i]^{T}\sigma_1\left(X_1^{T}a_j \right)\right) + \lambda \left\|X_1\right\|^2 \notag \\
& = \frac{\exp(-u)}{1+\exp(-u)} \cdot - u + \lambda \left\| X_1 \right\|^2.
\end{align*}
\begin{align*}
\sum_{i=1}^{m_2}\left\langle X_2[i], \frac{\partial F_j(X)}{\partial X_2[i]}\right\rangle & = \sum_{i=1}^{m_2}\frac{ - b_j\exp(-u)}{1+\exp(-u)} \cdot X_3[i]\partial \sigma_2\left(X_2[i]^{T}\sigma_1\left(X_1^{T}a_j \right)\right)X_2[i]^{T}\sigma_1(X_1^{T}a_j) + \lambda \left\|  X_2[i] \right\|^2 \notag \\
& = \sum_{i=1}^{m_2}\frac{ - b_j\exp(-u)}{1+\exp(-u)} X_3[i] \sigma_2\left(X_2[i]^{T}\sigma_1\left(X_1^{T}a_j \right)\right) + \lambda \left\|  X_2[i] \right\|^2 \notag \\
& = \frac{\exp(-u)}{1+\exp(-u)} \cdot -u + \lambda \left\|  X_2\right\|^2.
\end{align*}
\begin{align*}
\sum_{i=1}^{m_2}\left\langle X_3[i], \frac{\partial F_j(X)}{\partial X_3[i]}\right\rangle & = \sum_{i=1}^{m_2}\frac{\exp(-u)}{1+\exp(-u)} \cdot - b_j  X_3[i] \sigma_2\left(X_2[i]^{T}\sigma_1\left(X_1^{T}a_j \right)\right) + \lambda \left\|X_3[i] \right\|^2 \notag \\
& = \frac{\exp(-u)}{1+\exp(-u)} \cdot -u + \lambda \left\|X_3 \right\|^2.
\end{align*}
Incorporating the above results, we can achieve that 
\begin{align*}
   \left\langle X, \nabla F_j(X)\right\rangle  = \lambda\left\|X \right\|^2 - \frac{3u \exp(-u)}{1+\exp(-u)}.
\end{align*}
For any $u $, we can estimate 
\begin{align*}
    \frac{u \exp(-u)}{1+\exp(-u)} = \frac{u }{1+\exp(u)} \leq 1.
\end{align*}
Finally, we get that
\begin{align*}
\left\langle X, \nabla F(X)\right\rangle & = \frac{1}{n}\sum_{j=1}^n \left\langle X, \nabla F_j(X)\right\rangle  \geq \lambda\left\|X \right\|^2 - 3.
\end{align*}

    \item If the neural networks are $S$ layers,  we consider ReLU as the activation function in the inner layers and softmax as the activation function of the output layer. We use cross-entropy to estimate the true label and prediction output. Given the data set $\left\lbrace (a_j, b_j)\right\rbrace_{j=1}^{n}$, the function can be formulated as below:
\begin{align*}
 F(X) = \frac{1}{n}\sum_{j=1}^n\log\left(1+\exp\left(-b_j X_S^{T}\sigma_{S-1}\left(X_{S-1}^{T}\sigma_{S-2}\left(\cdots X_1^{T}a_j \right)\right) \right)\right) + \frac{\lambda}{2}\left\|X \right\|^2.
\end{align*}
 From the output of layer $s$ to the next layer, the product of the weight matrix $X_s$ and $\sigma_{s}(X_{s-1}$ can be re-written as  $X_{s}^{T}\sigma_{s}(X_{s-1}^T\sigma_{s-1}(\cdots)) = \sum_{i=1}^{m_{s-1}}X_s[:][i]\sigma_{s}(X_{s-1}[i]^T\sigma_{s-1}(\cdots)) \in \R^{m_s}$.   
    %   \begin{align}
    %         F(X) = \frac{1}{n}\sum_{j=1}^n\log\left(1+\exp\left(-b_j X_L^{T}\sigma_{L-1}\left(X_{L-1}^{T}\sigma_{L-2}\left(\cdots \sum_{i=1}^{m_1}X_2[:][i]\sigma_1(X_1[i]^{T}a_j )\right)\right) \right)\right) + \frac{\lambda}{2}\left\|X \right\|^2
    %     \end{align}
Let $ u = b_jX_S^{T}\sigma_{S-1}\left(X_{S-2}^{T}\sigma_{S-2}\left(\cdots X_1^{T}a_j \right)\right)$, then we compute the gradient of the function $F_j$ (the $j$-data $(a_j, b_j)$) with regard to the weight matrix of each layers:      
  \begin{align*}
        \frac{\partial F_j}{\partial X_1[i]} & = \frac{-\exp(-u)b_j}{1+\exp(-u)} \cdot  X_S^{T}\partial \sigma_{S-1}\left(\cdots \right)\cdots \left(X_2[:][i]\partial \sigma_1(X_1^{T}[i]a_j)\right) a_j + \lambda X_1[i]\notag \\
          \frac{\partial F_j}{\partial X_2[i]} & = \frac{-\exp(-u)b_j}{1+\exp(-u)} \cdot   X_S^{T}\partial \sigma_{S-1}(\cdots)\cdots \left(X_3[:][i]\partial \sigma_2(X_2[i]^{T}\sigma_1(X_1^{T}a_j))\right)\sigma_1(X_1^{T}a_j) + \lambda X_2[i]\notag \\    
        \vdots & \notag \\
        \frac{\partial F_j}{\partial X_{S-1}[i]} & =  \frac{\exp(-u)}{1+\exp(-u)}
        \cdot - b_j X_S[i]^{T}\partial \sigma_{S-1}\left(X_{S-1}[i]^{T}\cdots \left(X_1^{T}a_j \right)\right)\sigma_{S-2}(X_{S-2}^{T}\cdots \sigma_1\left(X_1^{T}a_j \right)) + \lambda X_{S-1}[i] \notag \\
        \frac{\partial F_j}{\partial X_S} & = \frac{\exp(-u)}{1+\exp(-u)} \cdot - b_j  \sigma_{S-1}\left(X_{S-1}^{T}\sigma_{S-1}\cdots \left(X_1^{T}a_j \right)\right) + \lambda X_S.
  \end{align*}     
The inner product of $X$ and $\nabla F_j(X)$ is 
    \begin{align*}
    & \left\langle X, \nabla F_j(X)\right\rangle \notag \\
    & = \sum_{i=1}^{m_1}\left\langle X_1[i], \frac{\partial F_j}{\partial X_1[i]}\right\rangle + \sum_{i=1}^{m_2}\left\langle X_2[i], \frac{\partial F_j}{\partial X_2[i]}\right\rangle + \cdots + \sum_{i=1}^{m_{S-1}}\left\langle X_{S-1}[i], \frac{\partial F_j}{\partial X_{S-1}[i]}\right\rangle + \left\langle X_S, \frac{\partial F_j}{\partial X_S}\right\rangle
    \end{align*}  
Let $\sigma_1(x) = \sigma_2(x) =\cdots = \sigma_{S-1}(x) =\max(0, x)$, then their gradients  $\partial \sigma_1(x) = \partial \sigma_2(x) =\cdots= \chi_{x \geq 0}$. Then the activation functions with their gradient satisfy that $\sigma_i(x) =  x \partial \sigma_i(x)$  for $i=1, 2, \cdots, S-1$.  Then
\begin{align}
& \sum_{i=1}^{m_1}\left\langle X_1[i], \frac{\partial F_j(X)}{\partial X_1[i]}\right\rangle \notag \\
& = \frac{-\exp(-u)b_j}{1+\exp(-u)} \cdot   X_S^{T}\partial \sigma_{S-1}\left(\cdots \right)X_{S-1}^{T}\partial \sigma_{S-1}(\cdots)\cdots \sum_{i=1}^{m_1}X_2[:][i]\partial \sigma_1(X_1[i]^{T}a_j)X_1[i]^{T}a_j + \lambda \sum_{i=1}^{m_1}\left\|X_1[i] \right\|^2 \notag \\
& = \frac{-\exp(-u)b_j}{1+\exp(-u)} \cdot X_S^{T}\partial \sigma_{S-1}\left(\cdots \right)X_{S-1}^{T}\partial \sigma_{S-1}(\cdots)\cdots \sum_{i=1}^{m_1}X_2[:][i] \sigma_1(X_1[i]^{T}a_j) + \lambda \left\|X_1 \right\|^2 \notag \\
& = \frac{-\exp(-u)b_j}{1+\exp(-u)} \cdot X_S^{T}\partial \sigma_{S-1}\left(\cdots \right)X_{S-1}^{T}\partial \sigma_{S-1}(\cdots)\cdots \partial \sigma_{2}(X_2^{T} \sigma_1(X_1^{T}a_j))X_2^{T} \sigma_1(X_1^{T}a_j) + \lambda \left\|X_1 \right\|^2 \notag \\
& = \frac{\exp(-u)}{1+\exp(-u)} \cdot - u + \lambda \left\| X_1 \right\|^2 
\end{align}
\begin{align}
& \sum_{i=1}^{m_2}\left\langle X_2[i], \frac{\partial F_j(X)}{\partial X_2[i]}\right\rangle \notag \\
& = \frac{-\exp(-u)b_j}{1+\exp(-u)} \cdot   X_S^{T}\partial \sigma_{S-1}(\cdots)\cdots \sum_{i=1}^{m_2}\left(X_3[:][i]\partial \sigma_2(X_2[i]^{T}\sigma_1(X_1^{T}a_j))\right)X_2[i]^{T}\sigma_1(X_1^{T}a_j) + \lambda \sum_{i=1}^{m_2}\left\|  X_2[i] \right\|^2 \notag \\
& = \frac{-\exp(-u)b_j}{1+\exp(-u)} \cdot   X_S^{T}\partial \sigma_{S-1}(\cdots)\cdots \sum_{i=1}^{m_2}\left(X_3[:][i] \sigma_2(X_2[i]^{T}\sigma_1(X_1^{T}a_j))\right) + \lambda \left\|  X_2 \right\|^2 \notag \\
& = \frac{-\exp(-u)b_j}{1+\exp(-u)} \cdot  X_S^{T} \partial\sigma_{S-1}(\cdots)\cdots \partial \sigma_3(X_3^{T}\sigma_2(X_1^{T}a_j))X_3^{T}\sigma_2(X_2^{T}\sigma_1(X_1^{T}a_j)) + \lambda \left\|  X_2 \right\|^2 \notag \\
& = \frac{\exp(-u)}{1+\exp(-u)} \cdot -u + \lambda \left\|  X_2\right\|^2.
\end{align}
$\vdots$
\begin{align}
\left\langle X_S, \frac{\partial F_j}{\partial X_S}\right\rangle & = \frac{\exp(-u)}{1+\exp(-u)} \cdot - b_j  X_S^{T} \sigma_{S-1}\left(X_{S-1}^{T}\sigma_{S-1}\cdots \left(X_1^{T}a_j \right)\right) + \lambda \left\|X_S \right\|^2 \notag \\
& = \frac{\exp(-u)}{1+\exp(-u)} \cdot -u  + \lambda \left\|X_S \right\|^2.
\end{align}
Incorporating the above results, we can achieve that 
\begin{align*}
\left\langle X, \nabla F_j(X)\right\rangle  = \lambda\left\|X \right\|^2 - \frac{S u \exp(-u)}{1+\exp(-u)}.
\end{align*}
For any $u $, we can estimate 
\begin{align}
    \frac{u \exp(-u)}{1+\exp(-u)} = \frac{u }{1+\exp(u)} \leq 1.
\end{align}
Finally, we get that
\begin{align}
\left\langle X, \nabla F_j(X)\right\rangle  \geq \lambda\left\|X \right\|^2 - S.
\end{align}
\end{itemize}

\subsection{Applications in Section \ref{sec:general:sgd}}

\begin{itemize}
    \item Logistic regression with $\ell_1$ regularizer: $f(x) = \sum_{i=1}^{n}\log(1+\exp(-b_i a_i^{x})) + \lambda \left\| x \right\|_1$. The gradient is 
    \begin{align}
    \nabla f(x) = \frac{-b_i a_i \exp(-b_ia_i^{T}x)}{1 + \exp(-b_i a_i^{T}x)}  + \lambda \mbox{sign}(x)
    \end{align}
    then
    \begin{align}
    \left\langle  \nabla f(x), x\right\rangle  =  -b_i a_i^{T}x\frac{\exp(-b_ia_i^{T}x)}{1+\exp(-b_ia_i^{T}x)} + \lambda \left\|x \right\|_1 = \frac{u\exp(u)}{1+\exp(u)} + \lambda \left\| x \right\|_1
    \end{align}
If $u < 0$, then we get $ \left|\frac{u\exp(u)}{1+\exp(u)} \right| \leq \frac{-u}{2\exp(-u)} \leq \frac{1}{2}$; else if $ u>0$, then $\frac{u \exp(u)}{1+\exp(u)} >0$. Therefore, no matter $u > 0$ or not, we both have that
\begin{align}
     \left\langle  \nabla f(x), x\right\rangle  \geq \lambda \left\| x \right\|_1 - \frac{1}{2} \geq \lambda \left\| x \right\| - \frac{1}{2}
\end{align}
\item $S$-Layer Neural networks with a $\ell_1$ regularizer: replace the $\ell_2$ regularizer of Section \ref{dissipative:example} by $\ell_1$, we can derive that
\begin{align}
 \left\langle \nabla f(x), x\right\rangle \geq \lambda \left\|x\right\|_1 - S.
\end{align}
\end{itemize}
% \item \textcolor{red}{Questions: could I find some example in neural networks that for the point $x$ is far from the optimal solution, then this condition is satisfied.}
\end{itemize}
\section{Objective Function is Dissipative on its Domain}\label{supple:Rzero}
\begin{theorem}\label{thm:sgd:1}
Under the condition of Lemma \ref{lem:sgd:iter}, if $f$ is $L$-smooth and $R$-dissipative for all $x \in \R^d$. Consider the SGD algorithm with the total number of iterations $T \geq1$, for any step-size $\eta_k \leq \theta_1/((1+\rho)L^2)$, we have
\begin{align}
  \E[\left\|x_{T+1} - x^{\ast} \right\|^2]  \leq   \Pi_{k=1}^{T}\left( 1 - \eta_k \theta_1 \right)\left\| x_1 -x^{\ast} \right\|^2 + \sum_{k=1}^{T}  \left(\theta_2  + \eta_k \sigma^2\right)\eta_k\cdot\Pi_{s>k}^{T}\left( 1 - \eta_s \theta_1 \right).
\end{align}
\end{theorem}
\begin{proof}
Applying the $L$-smooth assumption, for any $x, y \in \R^d$, we have
\begin{align*}
\left\|\nabla f(y) - \nabla f(x) \right\| & \leq L\left\|y-x \right\|, \notag \\
f(y) + \left\langle \nabla f(x), y - x \right\rangle - \frac{L}{2}\left\| y-x\right\|^2  & \leq     f(y)  \leq f(x) + \left\langle \nabla f(x), y - x \right\rangle + \frac{L}{2}\left\| y-x\right\|^2.
\end{align*} 
Let $x = x^{\ast}$.  By $\nabla f(x^{\ast})= 0$, we get that
\begin{subequations}
\begin{align}
\left\| \nabla f(y) \right\| & \leq L\left\|x-x^{\ast} \right\|, \label{inequ:Lsmooth:grad}\\
    f(y) - f^{\ast}  & \leq \frac{L}{2}\left\|y-x^{\ast} \right\|^2 \label{inequ:Lsmooth:f}.
\end{align}
\end{subequations}
In this case, $R=0$, it means that $\left\langle x_k - x^{\ast}, \nabla f(x_k)\right\rangle \geq \theta_1 \left\| x_k - x^{\ast}\right\|^2 - \theta_2$  for all $x \in \R^d$. Applying this condition and (\ref{inequ:Lsmooth:grad}) into Lemma \ref{lem:sgd:iter}, we have
\begin{align*}
\E[\left\|x_{k+1} - x^{\ast} \right\|^2] & \leq  \left\| x_{k} - x^{\ast}\right\|^2  - 2\eta_k \theta_1 \left\|x_{k} - x^{\ast}  \right\|^2 + \theta_2 \eta_k + \eta_k^2 \sigma^2 + \eta_k^2 (1+\rho) L^2 \left\|x_{k} - x^{\ast}  \right\|^2 \notag \\
& =\left( 1 - 2\eta_k \theta_1 + \eta_k^2(1+\rho)L^2 \right) \left\| x_k -x^{\ast} \right\|^2  + \left(\theta_2  + \eta_k \sigma^2\right)\eta_k
\end{align*} 
Letting $\eta_k \leq \frac{\theta_1}{(\rho+1)L^2}$ and , then
$ 0 <  1 - 2\eta_k \theta_1 + \eta_k^2(1+\rho)L^2 \leq 1 - \eta_k \theta_1$. Consider the above recursion from $k=1$ to $T$, we can achieve that
\begin{align*}
  \E[\left\|x_{T+1} - x^{\ast} \right\|^2]  \leq   \Pi_{k=1}^{T}\left( 1 - \eta_k \theta_1 \right)\left\| x_1 -x^{\ast} \right\|^2 + \sum_{k=1}^{T}  \left(\theta_2  + \eta_k \sigma^2\right)\eta_k\cdot\Pi_{s>k}^{T}\left( 1 - \eta_s \theta_1 \right).
\end{align*}
Therefore, the proof is complete.
\end{proof}

\begin{theorem}(Constant step-size)
Under the same conditions of Theorem \ref{thm:sgd:1}. If the step-size is a constant $\eta_k = \eta \leq \theta_1/((1+\rho)L^2)$, then for any $T \geq 1$, we have
\begin{align*}
     \E[\left\|x_{T+1} - x^{\ast} \right\|^2] \leq (1 - \eta\theta_1)^{T}\left\| x_1 -x^{\ast} \right\|^2 + \frac{\left(\theta_2  + \eta \sigma^2\right)}{\theta_1}.
\end{align*}
\end{theorem}
\begin{proof}
\begin{align*}
  \E[\left\|x_{T+1} - x^{\ast} \right\|^2]  & \leq   \Pi_{k=1}^{T}\left( 1 - \eta_k \theta_1 \right)\left\| x_1 -x^{\ast} \right\|^2 + \sum_{k=1}^{T}  \left(\theta_2  + \eta_k \sigma^2\right)\eta_k\cdot\Pi_{s>k}^{T}\left( 1 - \eta_s \theta_1 \right)  \notag \\
  & = (1 - \eta\theta_1)^{T}\left\| x_1 -x^{\ast} \right\|^2 + \left(\theta_2  + \eta \sigma^2\right)\eta \sum_{k=1}^{T} (1-\eta \theta_1)^{T-k} \notag \\
  & =  (1 - \eta\theta_1)^{T}\left\| x_1 -x^{\ast} \right\|^2 + \left(\theta_2  + \eta \sigma^2\right)\eta \cdot \frac{1- (1-\eta \theta_1)^{T}}{\eta \theta_1} \notag \\
  & = (1 - \eta\theta_1)^{T}\left\| x_1 -x^{\ast} \right\|^2 + \frac{\left(\theta_2  + \eta \sigma^2\right)}{\theta_1}\cdot \left(1- (1-\eta \theta_1)^{T} \right) \notag \\
  & \leq (1 - \eta\theta_1)^{T}\left\| x_1 -x^{\ast} \right\|^2 + \frac{\left(\theta_2  + \eta \sigma^2\right)}{\theta_1}.
\end{align*}
\end{proof}
\begin{theorem}(Time-dependent step-size)
Under the same conditions of Theorem \ref{thm:sgd:1}. If the step-size is time dependent on the iteration where $\eta_k \leq \theta_1/((1+\rho)L^2)$, then for any $T \geq 1$, we have
\begin{align*}
     \E[\left\|x_{T+1} - x^{\ast} \right\|^2]  & \leq \exp\left(-\theta_1 \sum_{k=1}^{T}\eta_k\right)\left\| x_1 -x^{\ast} \right\|^2  + \sum_{k=1}^{T}  \left(\theta_2  + \eta_k \sigma^2\right)\eta_k\cdot \exp\left(-\theta_1 \sum_{s=k+1}^{T}\eta_s\right)
\end{align*}
\begin{itemize}
    \item $\eta_k = \eta_1/k^p$ for $p \in (0,1]$

    \begin{itemize}
        \item If $p \in (0,1)$, we have
        \begin{align*}
        \E[\left\|x_{T+1} - x^{\ast} \right\|^2]  \leq \exp\left(-\frac{\theta_1\eta_1}{1-p}\left( (T+1)^{1-p}-1\right) \right)\left\| x_1 -x^{\ast} \right\|^2 + \frac{\left(\theta_2  + \eta_1 \sigma^2\right)2^p}{\theta_1}.
        \end{align*}
        \item $p=1$, then
        \begin{align*}
             \E[\left\|x_{T+1} - x^{\ast} \right\|^2]  \leq \frac{1}{(T+1)^{\theta_1\eta_1}} \left\| x_1 - x^{\ast} \right\|^2 + \frac{2\left(\theta_2  + \eta_1 \sigma^2\right)}{\theta_1}.
        \end{align*}
    \end{itemize}

    \item Bandwidth-based step-sizes including step-decay band and polynomial band
    \begin{itemize}
        \item $ m/k^{p}\leq \eta_k  \leq M/k^{p}$
        \begin{itemize}
            \item $p \in (0,1)$,
                 \begin{align*}
        \E[\left\|x_{T+1} - x^{\ast} \right\|^2] \leq  \exp\left(-\frac{\theta_1 m }{1-p}\left( (T+1)^{1-p}-1\right) \right)\left\| x_1 -x^{\ast} \right\|^2 + \frac{\left(\theta_2  + M\sigma^2\right)2^p M}{m\theta_1}.
        \end{align*}
        \item $p=1$
            \begin{align*}
             \E[\left\|x_{T+1} - x^{\ast} \right\|^2]  \leq \frac{1}{(T+1)^{\theta_1 m }} \left\| x_1 - x^{\ast} \right\|^2 + \frac{2\left(\theta_2  + M\sigma^2\right)M}{m\theta_1}.
        \end{align*}
        
        \end{itemize}
   
        \item $ m \delta(k) \leq \eta_i^k \leq M\delta(k)$ where $\delta(k) = 1/\alpha^{k-1}$ for $i \in [S]$ and $ 1 \leq k \leq N$
        
    \end{itemize}
\end{itemize}
\end{theorem}
\begin{proof}
\begin{align}\label{inequ:time-dependent}
     \E[\left\|x_{T+1} - x^{\ast} \right\|^2]  & \leq   \Pi_{k=1}^{T}\left( 1 - \eta_k \theta_1 \right)\left\| x_1 -x^{\ast} \right\|^2 + \sum_{k=1}^{T}  \left(\theta_2  + \eta_k \sigma^2\right)\eta_k\cdot\Pi_{s>k}^{T}\left( 1 - \eta_s \theta_1 \right) \notag \\
     & = \exp\left(-\theta_1 \sum_{k=1}^{T}\eta_k\right)\left\| x_1 -x^{\ast} \right\|^2  + \sum_{k=1}^{T}  \left(\theta_2  + \eta_k \sigma^2\right)\eta_k\cdot \exp\left(-\theta_1 \sum_{s=k+1}^{T}\eta_s\right) 
\end{align}
If $\eta_k = \eta_1/k^p$ for $p \in (0,1)$, we have 
\begin{subequations}
\begin{align}
    \sum_{k=1}^{T} \eta_k & \geq \eta_1 \int_{k=1}^{T+1}1/k^p dk = \frac{\eta_1}{1-p}\left((T+1)^{1-p}-1\right)  \\
    \sum_{s=k+1}^{T} \eta_s & \geq \eta_1 \int_{s=k+1}^{T+1}1/s^p ds = \frac{\eta_1}{1-p}\left((T+1)^{1-p}-(k+1)^{1-p}\right)
\end{align}
\end{subequations}
Then we turn to estimate the last term of (\ref{inequ:time-dependent}) as below
\begin{align*}
 \sum_{k=1}^{T}  \left(\theta_2  + \eta_k \sigma^2\right)\eta_k \cdot \exp\left(-\theta_1 \sum_{s=k+1}^{T}\eta_s\right) 
& \leq \sum_{k=1}^{T}  \left(\theta_2  + \frac{\eta_1}{k^p} \sigma^2\right)\frac{\eta_1}{k^p}\cdot \exp\left(-\frac{\theta_1\eta_1}{1-p}\left(T^{1-p} - (k+1)^{1-p} \right) \right) \notag \\
& \leq \frac{\left(\theta_2  + \eta_1 \sigma^2\right)\eta_1}{\exp\left(\frac{\theta_1\eta_1}{1-p}T^{1-p} \right)}\sum_{k=1}^{T}\frac{1}{k^p}\exp\left(\frac{\theta_1\eta_1}{1-p}(k+1)^{1-p}\right) \notag \\
& \leq \frac{\left(\theta_2  + \eta_1 \sigma^2\right)\eta_1}{\exp\left(\frac{\theta_1\eta_1}{1-p}T^{1-p} \right)}\left(1+\frac{1}{k}\right)^{p}\int_{k=1}^{T+1}\frac{1}{(k+1)^p}\exp\left(\frac{\theta_1\eta_1}{1-p}(k+1)^{1-p}\right) dk \notag \\
& = \frac{\left(\theta_2  + \eta_1 \sigma^2\right)\eta_1}{\exp\left(\frac{\theta_1\eta_1}{1-p}T^{1-p} \right)}\frac{2^{p}}{\theta_1\eta_1}\left(\exp\left(\frac{\theta_1\eta_1}{1-p}(T+1)^{1-p}\right) - \exp\left(\frac{\theta_1\eta_1}{1-p}2^{1-p} \right) \right) \notag \\
& \leq \frac{\left(\theta_2  + \eta_1 \sigma^2\right)2^p}{\theta_1}.
\end{align*}

If $p=1$, then
\begin{subequations}
\begin{align}
    \sum_{k=1}^{T} \eta_k & \geq \eta_1 \int_{k=1}^{T+1} \frac{dk}{k} = \eta_1 \log(T+1) \\
    \sum_{s=k+1}^{T} \eta_s & \geq \eta_1 \int_{s=k+1}^{T+1} \frac{ds}{s} = \eta_1 \log\left(\frac{T+1}{k+1}\right).
\end{align}
\end{subequations}
The last term of (\ref{inequ:time-dependent}) can be bounded as 
\begin{align*}
\sum_{k=1}^{T}  \left(\theta_2  + \eta_k \sigma^2\right)\eta_k \cdot \exp\left(-\theta_1 \sum_{s=k+1}^{T}\eta_s\right) & \leq \sum_{k=1}^{T}  \left(\theta_2  + \eta_1 \sigma^2\right)\frac{\eta_1}{k} \cdot \exp\left(-\theta_1\eta_1 \log\left(\frac{T+1}{k+1}\right) \right) \notag \\
& \leq \left(\theta_2  + \eta_1 \sigma^2\right)\eta_1\sum_{k=1}^T \frac{(k+1)^{\theta_1\eta_1}}{k(T+1)^{\theta_1 \eta_1}}  \notag \\
& \leq \left(\theta_2  + \eta_1 \sigma^2\right)\eta_1\left(1 + \frac{1}{k}\right) \int_{k=1}^{T} \frac{(k+1)^{\theta_1\eta_1-1}}{(T+1)^{\theta_1\eta_1}}  \notag \\
& \leq 2\left(\theta_2  + \eta_1 \sigma^2\right)\eta_1 \frac{\left( (T+1)^{\theta_1\eta_1}-2^{\theta_1\eta_1}\right)}{\theta_1\eta_1(T+1)^{\theta_1\eta_1}} \notag \\
& \leq \frac{2\left(\theta_2  + \eta_1 \sigma^2\right)}{\theta_1}.
\end{align*}

If $ m \delta(k) \leq \eta_k \leq M \delta(k) $ and $\delta(k) = 1/\alpha^{k-1}$ where $\alpha > 1$.
Then
\begin{align*}
\sum_{k=1}^{N}\sum_{i=1}^{S} \eta_i^k & \geq m S \sum_{k=1}^{N}\alpha^{-k+1} = \frac{m S(1-\alpha^{-N})}{1-\alpha^{-1}} > mS.
\end{align*}
The second term of (\ref{inequ:time-dependent}) can be estimated as:
\begin{align*}
& \sum_{k=1}^{N}\sum_{i=1}^{S} \left(\theta_2  + \eta_i^k \sigma^2\right)\eta_i^k \cdot \exp\left(-\theta_1 \left(\sum_{s=i+1}^{S} \eta_s^k +\sum_{s=1}^{S}\sum_{l=k+1}^{N}\eta_s^l\right) \right) \notag \\
& \leq \sum_{k=1}^{N} \left(\theta_2  +  M  \sigma^2\right) M \alpha^{-k+1} \exp\left(-m\theta_1 \left(S\sum_{l=k+1}^{N}\alpha^{-l+1}\right) \right)\left( 1 + \sum_{i=1}^{S-1}\exp\left(-m\theta_1 (S-i)\alpha^{-k+1}\right)\right) \notag \\
& \leq \left(\theta_2  +  M  \sigma^2\right) M \sum_{k=1}^{N} \alpha^{-k+1} \exp\left(-m\theta_1 \left(S\sum_{l=k+1}^{N}\alpha^{-l+1}\right) \right)\left( 1 + \frac{\exp(-m\theta_1 \alpha^{-k+1})}{1 - \exp(-m\theta_1 \alpha^{-k+1})}\right)  \notag \\
% & \leq \left(\theta_2  +  M  \sigma^2\right) M \sum_{k=1}^{N} \alpha^{-k+1} \exp\left(-m\theta_1 \left(S\sum_{l=k+1}^{N}\alpha^{-l+1}\right) \right) \frac{1}{m\theta_1 \alpha^{-k+1}} \notag \\
& \mathop{\leq}^{(a)} \frac{\left(\theta_2  +  M  \sigma^2\right) M }{1-\exp(-m\theta_1)}\left(1 + \sum_{k=1}^{N-1} \exp\left(-m\theta_1 \left(S\sum_{l=k+1}^{N}\alpha^{-l+1}\right) \right) \right) \notag \\
& \leq \frac{\left(\theta_2  +  M  \sigma^2\right) M }{1-\exp(-m\theta_1)}\left( 1+ \sum_{k=1}^{N-1} \exp\left(-m\theta_1 S \alpha^{-N+1}\right)^{k} \right) \leq \frac{\left(\theta_2  +  M  \sigma^2\right) M }{1-\exp(-m\theta_1)}\left( 1+  \frac{\exp\left(-m\theta_1 S \alpha^{-N+1}\right)}{1-\exp\left(-m\theta_1 S \alpha^{-N+1}\right)} \right) \notag \\
& \leq \frac{\left(\theta_2  +  M  \sigma^2\right) M }{1-\exp(-m\theta_1)} \frac{1}{1-\exp\left(-m\theta_1 S \alpha^{-N+1}\right)}
\end{align*}
where $(a)$ follows from the fact that for all $k \in [N]$, we have
\begin{align}
 \frac{\alpha^{-k+1}}{1 - \exp(-m\theta_1\alpha^{-k+1})} \leq  \frac{1}{1-\exp(-m\theta_1)}
\end{align}
due to that $h(x) = \frac{x}{1-\exp(-m\theta_1x)}$ is increasing with $x$. By properly choosing $S$ and $N$, for example: $N=\log_{\alpha} T/2$ and $S = 2T/\log_{\alpha}T$, we have $S\alpha^{-N+1} = \frac{2T \alpha}{\sqrt{T}\log_{\alpha} T} = \frac{2\alpha\sqrt{T}}{\log_{\alpha} T}$. For sufficient large $T$, we can see that $ 0 < \exp\left(-m\theta_1 S \alpha^{-N+1}\right) \ll 1$. Thus, $1-\exp\left(-m\theta_1 S \alpha^{-N+1}\right) \approx 1$, we can achieve that
\begin{align*}
 \E[\left\| x_{T+1} - x^{\ast} \right\|^2] \leq \exp\left(-\theta_1 m S\right)\left\| x_1 - x^{\ast} \right\|^2 + \frac{\left(\theta_2  +  M  \sigma^2\right) M }{1-\exp(-m\theta_1)}.
\end{align*}
\end{proof}

%   \item If $\eta_k = \eta_1/k$, then $\tau_k = \eta_k/\eta_{k-1} = (k-1)/k < 1$ and $1-\tau_k= 1/k$.  In this case, we also can get that
%     \begin{align}
%         \eta_k\tau_k - \eta_{k-1}\tau_{k-1} = \eta_1\left(\frac{k-1}{k^2} - \frac{k-2}{(k-1)^2} \right) < 0
%     \end{align}
%     Then 
%     \begin{align}
%     \E[W_{k+1} \mid \mathcal{F}_k] & \leq  W_k -\left(\theta_1 \frac{\eta_1}{k} -\frac{1}{k^2} - (\rho+1)\left((1-\beta)^2 +1\right)L^2\frac{\eta_1^2}{k^2}\right)\left\| x_k - x^{\ast} \right\|^2 \notag \\
%     & \quad + 
%   2\theta_2\frac{\eta_1}{k} + \frac{\eta_1^2}{k^2}\left( (1-\beta)^2+1\right)\sigma^2 
%     \end{align}

% For 
% \begin{align}
% k \geq \frac{2\left(1+(\rho+1)\left((1-\beta)^2 +1\right)L^2\eta_1^2\right)}{\theta_1\eta_1}
% \end{align}
% we have 
% \begin{align}
%     \theta_1 \frac{\eta_1}{k} -\frac{1}{k^2} - (\rho+1)\left((1-\beta)^2 +1\right)L^2\frac{\eta_1^2}{k^2} \geq \frac{\theta_1\eta_1}{2k}
% \end{align}
% Then
%     \begin{align}
%     \E[W_{k+1} \mid \mathcal{F}_k] & \leq  W_k -\frac{\theta_1 \eta_1}{2k}\left\| x_k - x^{\ast} \right\|^2  + 
%   2\theta_2\frac{\eta_1}{k} + \frac{\eta_1^2}{k^2}\left( (1-\beta)^2+1\right)\sigma^2 
%     \end{align}
%     In this case, if 
%     \begin{align}
%      R^2 \geq \frac{4\theta_2}{\theta_1} + \frac{2\eta_1\left(1+(1-\beta)^2 \right)\sigma^2}{\theta_1}
%     \end{align}
%     we can see that $\E[W_k]$ is decreasing. 
\end{document}